\documentclass{article} %
\usepackage{iclr2027_conference,times}
\iclrfinalcopy

\usepackage{upgreek} %

\def\1{\bm{1}}

\def\ry{y}

\def\rvz{{\mathbf{z}}}

\def\vdelta{{\bm{\delta}}}

\def\vxi{{\bm{\xi}}}
\def\va{{\bm{a}}}

\def\ve{{\bm{e}}}

\def\vh{{\bm{h}}}

\def\vs{{\bm{s}}}

\def\vu{{\bm{u}}}
\def\vv{{\bm{v}}}
\def\vw{{\bm{w}}}
\def\vx{{\bm{x}}}
\def\vy{{\bm{y}}}
\def\vz{{\bm{z}}}

\def\mI{{\bm{I}}}

\def\mSigma{{\bm{\Sigma}}}

\DeclareMathAlphabet{\mathsfit}{\encodingdefault}{\sfdefault}{m}{sl}
\SetMathAlphabet{\mathsfit}{bold}{\encodingdefault}{\sfdefault}{bx}{n}

\def\sX{{\mathbb{X}}}

\usepackage{hyperref}
\usepackage{url}

\title{GRIT: Out-of-Distribution Robustness from Finite and Noisy Invariant Pairs}

\author{Ruqi Bai$^1$\thanks{Equal contribution.}\hspace{0.35em}, Easton Currie$^1$\footnotemark[1]\hspace{0.35em}, Yao Ji$^2$, Mingyu Kim$^1$, Zeyu Zhou$^1$, David I. Inouye$^1$\\
$^1$Purdue University \quad $^2$Georgia Institute of Technology}

\usepackage{savesym}
\savesymbol{AND} %

\usepackage{algorithmic}
\usepackage{algorithm}

\newif\ifshowcomments 

\showcommentsfalse

\usepackage{wrapfig}
\usepackage{booktabs}
\usepackage{caption}

\usepackage{graphicx}

\usepackage{algorithm} %
\usepackage{algorithmic} %

\usepackage{mathtools} %
\usepackage{amsthm, amssymb, amscd, amsfonts} %
\usepackage{thmtools}
\usepackage{thm-restate} %
\usepackage{bm} %
\usepackage{bbm} %

\newif\ifappendixproofs
\appendixproofstrue
\ifappendixproofs
  \usepackage[appendix=append,bibliography=common]{apxproof}
\else
  \usepackage[appendix=inline,bibliography=common]{apxproof}
\fi
\makeatletter
\def\@axp@newtheorem#1#2#3#4{%
  \ifblank{#4}{%
    \ifblank{#2}{%
      \newtheorem{#1}{#3}%
    }{%
      \newtheorem{#1}[#2]{#3}%
    }%
  }{%
    \newtheorem{#1}{#3}[#4]%
  }%
}
\makeatother
\usepackage{url}
\usepackage{hyperref} %
\usepackage{xr}
\usepackage{cleveref}

\hypersetup{colorlinks, allcolors=[rgb]{0.02,0.27,0.5}}
\usepackage{natbib} %
\renewcommand{\bibsection}{\subsubsection*{References}}

\usepackage{multirow}
\usepackage{enumitem} %
\usepackage[normalem]{ulem}
\usepackage{lipsum}  
\usepackage{xspace} %
\usepackage{comment}
\usepackage{tikz}
\usetikzlibrary{automata}
\usetikzlibrary{positioning}  %
\usetikzlibrary{arrows}       %
\tikzset{node distance=1.5cm, %
         every state/.style={ %
           semithick,
           fill=gray!5},
         initial text={},     %
         double distance=3pt, %
         every edge/.style={  %
         draw,
           ->,>=stealth',     %
           auto,
           semithick}}

\usepackage{dsfont}
 \usepackage{colortbl}
\newtheoremrep{theorem}{Theorem}
\newtheoremrep{lemma}{Lemma}
\newtheoremrep{corollary}[theorem]{Corollary}
\newtheoremrep{proposition}{Proposition}
\newtheoremrep{definition}{Definition}
\newtheoremrep{intuition}{Intuition}
\newtheoremrep{observation}{Observation}
\newtheoremrep{assumption}{Assumption}
\newtheoremrep{hypothesis}{Hypothesis}
\newtheoremrep{conjecture}{Conjecture}
\newtheoremrep{question}{Question}
\newtheoremrep{example}{Example}
\theoremstyle{definition}
\newtheoremrep{remark}{Remark}

\makeatletter
\newif\ifapxproofsectionwritten
\newcommand{\apxensureproofsection}{%
  \ifapxproofsectionwritten
    \nosectionappendix
  \else
    \global\apxproofsectionwrittentrue
    \ifdefined\axp@proofsfile
      \immediate\write\axp@proofsfile{%
        \noexpand\section*{Proofs}%
        \noexpand\nosectionappendix
      }%
    \fi
    \nosectionappendix
  \fi
}
\@ifpackageloaded{apxproof}{%
  \pretocmd{\theoremrep}{\apxensureproofsection}{}{}%
  \pretocmd{\lemmarep}{\apxensureproofsection}{}{}%
  \pretocmd{\corollaryrep}{\apxensureproofsection}{}{}%
  \pretocmd{\propositionrep}{\apxensureproofsection}{}{}%
  \pretocmd{\definitionrep}{\apxensureproofsection}{}{}%
  \pretocmd{\intuitionrep}{\apxensureproofsection}{}{}%
  \pretocmd{\observationrep}{\apxensureproofsection}{}{}%
  \pretocmd{\assumptionrep}{\apxensureproofsection}{}{}%
  \pretocmd{\hypothesisrep}{\apxensureproofsection}{}{}%
  \pretocmd{\conjecturerep}{\apxensureproofsection}{}{}%
  \pretocmd{\questionrep}{\apxensureproofsection}{}{}%
  \pretocmd{\examplerep}{\apxensureproofsection}{}{}%
  \pretocmd{\remarkrep}{\apxensureproofsection}{}{}%
}{}
\makeatother

\newcommand{\x}{x}

\newtheoremrep{desiderata}{Desiderata}
\makeatletter
\@ifpackageloaded{apxproof}{%
  \pretocmd{\desideratarep}{\apxensureproofsection}{}{}%
}{}
\makeatother

\definecolor{dg}{RGB}{0,120,0}

\usepackage{upgreek} %

\def\1{\bm{1}}

\def\ry{y}

\def\rvz{{\mathbf{z}}}

\def\vdelta{{\bm{\delta}}}

\def\vxi{{\bm{\xi}}}
\def\va{{\bm{a}}}

\def\ve{{\bm{e}}}

\def\vh{{\bm{h}}}

\def\vs{{\bm{s}}}

\def\vu{{\bm{u}}}
\def\vv{{\bm{v}}}
\def\vw{{\bm{w}}}
\def\vx{{\bm{x}}}
\def\vy{{\bm{y}}}
\def\vz{{\bm{z}}}

\def\mI{{\bm{I}}}

\def\mSigma{{\bm{\Sigma}}}

\DeclareMathAlphabet{\mathsfit}{\encodingdefault}{\sfdefault}{m}{sl}
\SetMathAlphabet{\mathsfit}{bold}{\encodingdefault}{\sfdefault}{bx}{n}

\def\sX{{\mathbb{X}}}

\definecolor{newcolor}{RGB}{0, 0, 0}      %
\definecolor{removecolor}{RGB}{200, 150, 150}   %
\definecolor{updatecolor}{RGB}{0, 0, 0} %

\usepackage{etoolbox}  %

\definecolor{darkbrown}{HTML}{654321} %

\definecolor{darkgray}{rgb}{0.4, 0.4, 0.4} %

\hypersetup{pdftitle={GRIT: Out-of-Distribution Robustness from Finite and Noisy Invariant Pairs},
pdfauthor={Ruqi Bai, Easton Currie, Yao Ji, Mingyu Kim, Zeyu Zhou, David I. Inouye}}
\begin{document}
\maketitle
\lhead{Preprint}
\begin{abstract}
Predictors often exploit spurious correlations that fail under distribution shift, undermining out-of-distribution (OOD) robustness. Invariant pairs, such as the same object photographed against different backgrounds, directly indicate which variation should not affect prediction. Yet real pair collections are finite and noisy, so they can miss relevant variation or lead a learner to discard predictive information. To learn from such collections, we propose Geometric Robustness Invariant Training (GRIT), which estimates an unstable subspace from the covariance of pair differences in fixed pretrained features and projects it out before fitting a predictor. Because pair processing is separate from predictor fitting, the same projection can accompany ERM or other robustness objectives. Unlike quadratic prediction consistency, which penalizes variation along every direction with nonzero pair covariance, including variation introduced by pair noise, GRIT enforces exact invariance only to the estimated subspace. We bound GRIT's target risk in terms of pair count, noise relative to the weakest clean signal, and selected rank, assuming that the label distribution given a linear projection of the features is invariant across environments. Empirically, GRIT improves the mean OOD accuracy of each of eight learning objectives on ColoredMNIST-INV and Waterbirds-INV and has higher mean accuracy than a prediction-consistency penalty using the same pairs in 15 of 16 settings. On UrbanCars, combining pairs that vary the background with pairs that vary a co-occurring object raises worst-group accuracy above that of either pair type alone.
\end{abstract}
\suppressfloats[t]
\begin{figure}[t]
\centering
\resizebox{0.782\linewidth}{!}{\begingroup\endlinechar=-1\relax\begingroup
\definecolor{heroInk}{HTML}{233340}
\definecolor{heroTrue}{HTML}{167C65}
\definecolor{heroEst}{HTML}{9254A1}
\definecolor{heroError}{HTML}{B6463C}
\definecolor{heroClassA}{HTML}{397CB5}
\definecolor{heroClassB}{HTML}{CD752E}
\definecolor{heroGray}{HTML}{7D8992}
\input{figs/hero-geometry-data}
\begin{tikzpicture}[x=1bp,y=-1bp,
  every node/.style={font=\sffamily\fontsize{9}{10}\selectfont,text=heroInk,inner sep=0pt},
  heading/.style={font=\sffamily\bfseries\fontsize{10}{11}\selectfont,anchor=west},
  parameter/.style={font=\fontsize{14}{14}\selectfont},
  flow/.style={->,>=stealth,draw=heroGray,line width=.8bp},
  leader/.style={->,>=stealth,draw=heroGray,line width=.6bp},
  true/.style={draw=heroTrue,line width=1.2bp},
  estimated/.style={draw=heroEst,line width=1.4bp,dash pattern=on 4bp off 2bp}]
\path[use as bounding box] (0,0) rectangle (396,181);
\node[heading] at (2,7) {(a) Invariant pairs};
\node[heading,align=left] at (93,12) {(b) Estimate\\\textit{unstable} direction};
\node[heading] at (218,7) {(c) Project, then classify};
\node[anchor=west] at (2,28) {In fixed features};
\foreach \dx/\dy [count=\j from 0] in \heroPairVectors {
  \pgfmathsetmacro{\cx}{20+36*mod(\j,2)}
  \pgfmathsetmacro{\cy}{46+19*floor(\j/2)}
  \ifnum\j<4\def\pairColor{heroClassB}\else\def\pairColor{heroClassA}\fi
  \draw[heroGray,line width=.7bp] (\cx-8*\dx,\cy-8*\dy)--(\cx+8*\dx,\cy+8*\dy);
  \filldraw[fill=\pairColor!22,draw=\pairColor,line width=.7bp] (\cx-8*\dx,\cy-8*\dy) circle (1.8bp);
  \filldraw[fill=\pairColor!22,draw=\pairColor,line width=.7bp] (\cx+8*\dx,\cy+8*\dy-2.1)--(\cx+8*\dx-1.85,\cy+8*\dy+1.2)--(\cx+8*\dx+1.85,\cy+8*\dy+1.2)--cycle;
}
\node[anchor=west,align=left] at (2,126) {Same task content,\\changed conditions};
\draw[flow] (77,77)--(88,77);
\node[anchor=west] at (93,34) {Pair-difference covariance};
\fill[heroEst!5,rotate around={-\heroEstAngle:(147,76)}] (147,76) ellipse (44bp and 10bp);
\foreach \dx/\dy in \heroPairVectors {
  \draw[heroGray!65,line width=.6bp] (147-25*\dx,76-25*\dy)--(147+25*\dx,76+25*\dy);
  \fill[heroGray!70] (147+25*\dx,76+25*\dy) circle (1.2bp);
  \fill[heroGray!70] (147-25*\dx,76-25*\dy) circle (1.2bp);
}
\draw[true] (147-51*\heroTrueCos,76-51*\heroTrueSin)--(147+51*\heroTrueCos,76+51*\heroTrueSin);
\draw[estimated] (147-51*\heroEstCos,76-51*\heroEstSin)--(147+51*\heroEstCos,76+51*\heroEstSin);
\fill[heroInk] (147,76) circle (1.7bp);
\draw[heroError,line width=1.5bp] plot[domain=\heroEstAngle:\heroTrueAngle,samples=21] ({147+42*cos(\x)},{76+42*sin(\x)});
\draw[true] (96,120)--(107,120);
\node[anchor=west,text=heroTrue] at (112,120) {True direction};
\draw[estimated] (96,133)--(107,133);
\node[anchor=west,text=heroEst] at (112,133) {Estimated direction};
\draw[flow] (204,77)--(213,77);
\filldraw[fill=heroGray!25,draw=heroGray,line width=.6bp] (221,23) circle (1.7bp);
\filldraw[fill=heroGray!25,draw=heroGray,line width=.6bp] (228,21)--(226.15,24.2)--(229.85,24.2)--cycle;
\node[anchor=west] at (234,23) {Train};
\draw[heroInk,line width=.9bp] (272,21)--(272,25) (270,23)--(274,23);
\node[anchor=west] at (279,23) {Shifted test};
\begin{scope}
\clip (218,32) rectangle (392,137);
\draw[heroGray!25,line width=.45bp] (219,79)--(391,79);
\draw[heroGray!25,line width=.45bp] (305,136)--(305,33);
\fill[heroEst!9] (305,79)--(405,79+100*\heroEstTan)--(405,79+100*\heroTrueTan)--cycle;
\fill[heroEst!9] (305,79)--(205,79-100*\heroEstTan)--(205,79-100*\heroTrueTan)--cycle;
\foreach \xa/\ya/\xb/\yb/\lab in \heroPairEndpoints {
  \draw[heroGray!65,line width=.6bp] (305+18*\xa,79+18*\ya)--(305+18*\xb,79+18*\yb);
}
\foreach \xx/\yy/\lab in \heroSourceDomainA {
  \ifnum\lab=1\def\classColor{heroClassB}\else\def\classColor{heroClassA}\fi
  \filldraw[fill=\classColor!22,draw=\classColor!65,line width=.45bp] (305+18*\xx,79+18*\yy) circle (1.6bp);
}
\foreach \xx/\yy/\lab in \heroSourceDomainB {
  \ifnum\lab=1\def\classColor{heroClassB}\else\def\classColor{heroClassA}\fi
  \filldraw[fill=\classColor!22,draw=\classColor!65,line width=.45bp] (305+18*\xx,79+18*\yy-1.9)--(305+18*\xx-1.65,79+18*\yy+1.05)--(305+18*\xx+1.65,79+18*\yy+1.05)--cycle;
}
\draw[heroInk,line width=.95bp,densely dotted] (219,79+86*\heroErmTan)--(391,79-86*\heroErmTan);
\draw[true] (219,79-86*\heroTrueTan)--(391,79+86*\heroTrueTan);
\draw[estimated] (219,79-86*\heroEstTan)--(391,79+86*\heroEstTan);
\foreach \xx/\yy/\lab in \heroTargetPoints {
  \ifnum\lab=1\def\classColor{heroClassB}\else\def\classColor{heroClassA}\fi
  \draw[\classColor,line width=.9bp] (305+18*\xx-1.7,79+18*\yy)--(305+18*\xx+1.7,79+18*\yy) (305+18*\xx,79+18*\yy-1.7)--(305+18*\xx,79+18*\yy+1.7);
}
\draw[true,->,>=stealth] (305,79)--(305+34*\heroTrueSin,79-34*\heroTrueCos);
\draw[heroEst,line width=1.4bp,->,>=stealth] (305,79)--(305+34*\heroEstSin,79-34*\heroEstCos);
\draw[heroTrue,line width=.65bp] (305+5*\heroTrueCos,79+5*\heroTrueSin) --(305+5*\heroTrueCos+5*\heroTrueSin,79+5*\heroTrueSin-5*\heroTrueCos) --(305+5*\heroTrueSin,79-5*\heroTrueCos);
\end{scope}
\node[parameter,text=heroEst] at (278,40) {$\boldsymbol{w}_1$};
\draw[leader,draw=heroEst] (287,37) .. controls (302,29) and (315,36) .. (319.4,48.3);
\node[parameter,text=heroTrue] at (344,43) {$\boldsymbol{w}_{\mathcal S}^{*}$};
\draw[leader,draw=heroTrue] (333,46)--(325,51);
\draw[heroInk,line width=.95bp,densely dotted] (218,146)--(229,146);
\node[anchor=west] at (233,146) {ERM};
\draw[true] (263,146)--(274,146);
\node[anchor=west,text=heroTrue] at (278,146) {Ideal invariant};
\draw[estimated] (342,146)--(353,146);
\node[anchor=west,text=heroEst] at (358,146) {GRIT};
\node[align=center] at (43,171) {Pair count and\\signal-to-noise};
\draw[leader] (43,157)--(43,140);
\draw[flow] (85,171)--(106,171);
\node[align=center] at (158,171) {Subspace\\estimation error};
\draw[leader,draw=heroError,line width=.9bp] (181,157) .. controls (211,157) and (219,99) .. (183.5,97.2);
\draw[flow] (211,171)--(229,171);
\node[align=center] at (313,171) {Remaining shift sensitivity\\and target-risk bound};
\draw[leader] (379,159) .. controls (389,155) and (388,127) .. (370,116.6);
\end{tikzpicture}
\endgroup\endgroup}
\captionsetup{font=normalsize,skip=3pt}
\caption[GRIT overview: method and guarantee]{\textbf{GRIT uses noisy invariant pairs to approach the ideal invariant classifier.}
\textbf{(a)} Eight pairs connect two training domains (shapes); colors denote classes.
\textbf{(b)} Their difference covariance estimates an unstable direction; the red arc marks estimation error.
\textbf{(c)} Projecting before fitting improves on ERM here, but GRIT retains some sensitivity from estimation error (shading).
Test points ($+$) shift along the true unstable direction, reversing its correlation with the label.
The ideal and GRIT boundaries are parallel to the true and estimated directions in (b), respectively, making each classifier invariant to shifts along its corresponding direction.
Arrows show classifier normals: $\vw_{\mathcal S}^*$ for the ideal classifier and $\vw_1$ for rank-one GRIT; shading marks disagreement with the ideal boundary.
 }
\label{fig:grit-overview}
\vspace{-0.5em}
\end{figure}
 \section{Introduction}
\label{sec:introduction}
Spurious correlations challenge out-of-distribution (OOD) robustness because signals that are predictive during training may become unreliable after deployment, as when image backgrounds predict object class. We study predictors fitted on fixed pretrained features. Such features make linear probing practical \citep{radford2021learning}, yet the fitted predictors may still exploit spurious correlations \citep{kirichenko2023last}. Learning from paired observations has driven advances in representation learning from augmented views (SimCLR, \citealp{chen2020simple}; BYOL, \citealp{grill2020bootstrap}), image--text pairs (CLIP, \citealp{radford2021learning}), and related views or regions (I-JEPA, \citealp{assran2023self}; LeJEPA, \citealp{balestriero2025lejepa}).

For OOD robustness, consistency regularization trains on pairs while penalizing differences in their predictions, representations, or feature attributions. CoRe penalizes prediction or loss variation within groups sharing an identity and label \citep{heinzedeml2021conditional}. MatchDG learns cross-domain matches and regularizes their representations \citep{mahajan2021domain}. LAM aligns feature contributions to class logits across pairs that share semantic content \citep{gao2024consistency}. IPG uses pairs in corrective gradient updates that align class-dependent feature attributions and separate classes \citep{surner2026invariance}. We call such pairs \textbf{invariant pairs}: two inputs intended to share task-relevant content while differing in factors that should not affect prediction. Examples include the same bird against different backgrounds, cars of the same model photographed from different angles, or similar tissue samples from a patient imaged using different microscopes. Such pairs can arise through augmentations, multiple views, matching, or targeted interventions. \emph{Our motivating hypothesis is that invariant pairs can do for OOD robustness what labeled examples do for supervised learning: provide a common interface that decouples application-specific knowledge from application-agnostic learning.}

Despite these advances, a central gap remains in characterizing which invariances can be reliably enforced from \textbf{finite, noisy invariant pairs}. A finite pair collection may fail to capture desired invariances, while pair noise may induce undesired invariances that discard information that remains predictive across environments. For consistency regularization, penalizing these pairwise differences can also suppress features that remain predictive across environments. Addressing this gap requires connecting pair count, noise, and geometric coverage to both the choice of invariance constraint and its effect on target risk. We ask: \textbf{How can we learn the right invariances from finite, noisy invariant pairs, and when are they sufficient for OOD robustness?}

We propose Geometric Robustness Invariant Training (GRIT)\footnote{Code: \url{https://github.com/inouye-lab/GRIT}.} and characterize when finite, noisy invariant pairs suffice for OOD robustness. GRIT estimates a linear subspace from pair differences and projects it out before predictor fitting (Figure~\ref{fig:grit-overview}). This constraint can accompany ERM or another learning objective. For linear probing, our bounds separate source risk, changes in stable-feature distributions (bridge covariate shift), and sensitivity to unstable variation. GRIT's remaining sensitivity to unstable directions depends on pair count, noise relative to the weakest clean signal, and rank. The paper makes three contributions:

\begin{enumerate}[leftmargin=*]
\item \textbf{A geometric method for finite, noisy invariant pairs.} GRIT estimates a subspace from the leading eigenvectors of the pair covariance and projects features onto its orthogonal complement before predictor fitting. The selected rank controls which directions are removed, separating subspace estimation from downstream learning and allowing the same constraint to accompany different learning objectives.
\item \textbf{Theoretical conditions for OOD robustness from finite, noisy invariant pairs.} Under stated noise and distribution-shift assumptions, we give a high-probability target-risk bound for GRIT (Theorem~\ref{thm:hard-grit}). The GRIT bound depends on pair count, noise relative to the weakest clean signal, and the cost of removing extra directions. Fitted-predictor examples match its leading pair-count and noise dependence at the correct rank. Separately, we analyze quadratic prediction consistency (Theorem~\ref{thm:quadratic-cr}), characterizing how penalty strength and pair noise affect target risk.
\item \textbf{Empirical improvements in OOD robustness with invariant pairs.} Because GRIT projects fixed features before predictor fitting, it can serve as a preprocessing step for ERM or compatible domain generalization methods. On ColoredMNIST-INV and Waterbirds-INV, GRIT improves mean performance with each of eight learning objectives, including ERM and four-group GroupDRO. Its mean exceeds prediction consistency in 15 of 16 settings, with the exception occurring on Waterbirds-INV. On UrbanCars, combining pairs with background or object changes improves robustness to two spurious features.
\end{enumerate}
 \section{Problem Setup: Predictive Stability and Pair Information}
\label{sec:setup}
To connect invariant pairs to prediction, we formalize predictive stability in a fixed representation.
Let $\vz=\phi(\vx)\in\mathbb R^d$, where $\phi:\sX\to\mathbb R^d$ is a fixed feature extractor.
We use $(\rvz,\ry)$ for random features and labels and $(\vz,y)$ for their realized values. Let $p$ denote the source distribution and $q$ denote an unseen target distribution of $(\rvz,\ry)$.
For a subspace $\mathcal T$, let $P_{\mathcal T}$ be its orthogonal projector and $\vz_{\mathcal T}=P_{\mathcal T}\vz$.

\begin{assumption}[Predictive stability]
\label{ass:stability}
There is a subspace $\mathcal S\subseteq\mathbb R^d$ and a conditional distribution $P_{\mathrm{inv}}$ such that, for every source or target environment $e$,
\begin{equation}
P^{(e)}_{\ry\mid\rvz_{\mathcal S}}(y\mid\vz_{\mathcal S})=P_{\mathrm{inv}}(y\mid\vz_{\mathcal S})
\quad P^{(e)}_{\rvz_{\mathcal S}}\text{-almost surely}.
\label{eq:stability}
\end{equation}
\end{assumption}
Here $P^{(e)}_{\ry\mid\rvz_{\mathcal S}}$ is the conditional label distribution and $P^{(e)}_{\rvz_{\mathcal S}}$ is the marginal distribution of the stable features in environment $e$.
We call $\mathcal S$ a \emph{predictively stable subspace} and $\mathcal U=\mathcal S^\perp$ its \emph{unstable complement}.
Assumption~\ref{ass:stability} fixes the label distribution given the stable features, but allows the distribution of those features to vary across environments.
Unstable features can also help predict labels in the source environment even when their relationship to labels changes at deployment.
Removing them may therefore reduce source accuracy; predictive stability alone does not eliminate this tradeoff.
Assumption~\ref{ass:stability} extends feature-subset conditional-invariance models to linear projections \citep{rojascarulla2018invariant,peters2016causal}; full-feature covariate shift is the case $\mathcal S=\mathbb R^d$ \citep{sugiyama2007covariate}.
Appendix~\ref{app:connections} details the distinction from IRM.

\paragraph{Two sources of information.}
The learner receives labeled inputs $\{(\vx_i,y_i)\}_{i=1}^n$ and invariant pairs $\{(\vx_j,\vx_j')\}_{j=1}^k$, with all inputs in $\sX$.
Labels inform prediction, while differences $\boldsymbol\delta_j=\phi(\vx_j)-\phi(\vx_j')$ inform the unstable geometry.
To estimate the unstable subspace, a pair collection needs \emph{coverage} of relevant unstable directions and \emph{precision} to avoid removing stable predictive directions and unnecessarily reducing predictive performance.
These requirements concern the collection as a whole.
For coverage, many invariant pairs may repeat one variation and leave others unseen, whereas pooling complementary pair types can cover several variations that prediction should ignore.
For precision, individual differences need not lie entirely in $\mathcal U$: the aggregate covariance can still distinguish unstable variation from residual variation outside that subspace.

These collection properties apply whether pairs are designed or sampled.
Repeated measurements, multiple views, augmentations, and targeted interventions can supply candidate invariant pairs, provided the chosen relation preserves task-relevant information.
Appendix~\ref{app:connections} gives sufficient constructions and their limits.
We study learning from these collections once they are supplied.
 \section{GRIT: Estimating and Removing Unstable Directions}
\label{sec:method}
If $\mathcal U$ were known, invariance to its directions would require a linear score $f_{\vw}(\vz)=\vw^\top\vz$ to satisfy $P_{\mathcal U}\vw=0$.
We could impose this constraint by fitting on $P_{\mathcal S}\vz$.
GRIT substitutes an estimate of $\mathcal U$ obtained from the pairs into this construction.
Whether this constraint helps prediction depends on how much label information the stable subspace $\mathcal S$ retains.

\paragraph{Estimate, select, project, fit.}
Define the \emph{pair covariance} $\widetilde{\boldsymbol\Sigma}_k$ by giving each observed difference and its negative equal weight. The symmetrization is detailed in Appendix~\ref{app:grit-computation}.
With eigenvalues in descending order and orthonormal eigenvectors $\widehat{\vu}_a$, compute
\begin{equation}
\widetilde{\boldsymbol\Sigma}_k=\frac1k\sum_{j=1}^k\boldsymbol\delta_j\boldsymbol\delta_j^\top,
\qquad P_{k,r}=\sum_{a=1}^r\widehat{\vu}_a\widehat{\vu}_a^\top,
\qquad \vz^{\mathrm{GRIT}}=(\boldsymbol I-P_{k,r})\vz.
\label{eq:grit-projection}
\end{equation}
The rank $r$ counts removed directions; the predictor retains $d-r$ dimensions.
Fit the chosen objective on projected features, using validation data to select $r$ and training hyperparameters.
For empirical risk minimization (ERM), this is equivalent to minimizing the empirical loss subject to $P_{k,r}\vw=0$.
The same preprocessing can accompany objectives that use training environments or groups.
Appendix~\ref{app:grit-computation} gives the computation details.

Projection gives exact invariance to $\widehat{\mathcal U}_r=\operatorname{range}(P_{k,r})$: every head receiving only projected features has the same output at $\vz$ and $\vz+\vu$ for $\vu\in\widehat{\mathcal U}_r$.
Appendix~\ref{app:spectral-minimax} gives a minimax justification under Gaussian pair noise.
Projection removes differences along the estimated subspace, but members of a noisy invariant pair can still differ in the retained directions. Section~\ref{sec:theory} analyzes when the estimated subspace captures the unstable variation and how estimation error affects target risk.

\paragraph{Relationship to consistency.}
Quadratic prediction consistency penalizes differences in the scalar linear scores:
\begin{equation}
\frac{\lambda}{k}\sum_{j=1}^k(\vw^\top\boldsymbol\delta_j)^2
=\lambda\vw^\top\widetilde{\boldsymbol\Sigma}_k\vw
=\lambda\sum_{a=1}^d\widetilde\lambda_a(\widehat{\vu}_a^\top\vw)^2,
\qquad\lambda>0.
\label{eq:quadratic-consistency}
\end{equation}
For logistic prediction these are logit differences, not probability differences.
This follows the same consistency principle as penalties in CoRe, MatchDG, and LAM \citep{heinzedeml2021conditional,mahajan2021domain,gao2024consistency}. Squared differences between view embeddings also appear in LeJEPA's prediction objective \citep{balestriero2025lejepa}.
The key difference from GRIT is that quadratic consistency retains eigenvalue magnitudes as penalty weights, whereas GRIT selects a subspace and enforces invariance to it exactly.

\begin{proposition}[Covariance changes preserving the GRIT constraint]
\label{prop:method-equivalence}
Let $C,C'\succeq0$ be pair covariance matrices with strict gaps between their $r$th and $(r+1)$st eigenvalues, for fixed $1\le r<d$. Let $\mathcal V$ be the leading $r$-dimensional eigenspace of $C$.
The rank-$r$ GRIT constraints from $C$ and $C'$ are identical if and only if $C'$ does not mix $\mathcal V$ with $\mathcal V^\perp$ and its eigenvalues within $\mathcal V$ remain larger than those within $\mathcal V^\perp$; equivalently,
\[
P_{\mathcal V}C'P_{\mathcal V^\perp}=0,
\qquad \lambda_{\min}(C'|_{\mathcal V})>\lambda_{\max}(C'|_{\mathcal V^\perp}).
\]
For fixed $\lambda>0$, the consistency penalties $\lambda\vw^\top C\vw$ and $\lambda\vw^\top C'\vw$ are identical for every $\vw$ if and only if $C=C'$.
\end{proposition}
Thus GRIT tolerates changes to eigenvalues and rotations of eigenvectors within either the removed or retained subspace, provided the two subspaces do not mix and their eigenvalues remain separated.
For example, $\operatorname{diag}(4,1)$ and $\operatorname{diag}(9,1)$ both give the rank-one GRIT constraint $w_1=0$, but assign different consistency penalties to that coordinate.
Appendix~\ref{app:equivalence} proves the proposition and further characterizes the allowed perturbations. Section~\ref{sec:theory} connects subspace recovery to target risk.
 \section{Theoretical Analysis}
\label{sec:theory}
We analyze binary logistic risk $R_p(\vw)=\mathbb E_p\log(1+e^{-\ry\vw^\top\rvz})$, with $\ry\in\{-1,1\}$; mean squared error analogues appear in Appendix~\ref{app:mse}.

\subsection{What invariance controls under distribution shift}
\label{sec:shift-decomposition}
Assumption~\ref{ass:stability} fixes the conditional label distribution given $\vz_{\mathcal S}$, but allows the test distribution to change in two ways: the stable marginal can change from $p(\vz_{\mathcal S})$ to $q(\vz_{\mathcal S})$, and the unstable conditional can change from $p(\vz_{\mathcal U}\mid\vz_{\mathcal S},y)$ to $q(\vz_{\mathcal U}\mid\vz_{\mathcal S},y)$. Our focus is the second change, which invariance learned from pairs can control. We introduce a bridge distribution to isolate it while accounting separately for the change in the stable marginal.

Define $\breve q(\vz_{\mathcal S},y,\vz_{\mathcal U})=p(\vz_{\mathcal S}){p(y\mid\vz_{\mathcal S})}{q(\vz_{\mathcal U}\mid\vz_{\mathcal S},y)}$.
This bridge keeps the source stable features and labels jointly distributed as before, but replaces the unstable conditional with that of the target. Thus $p\to\breve q$ isolates the unstable change, while $\breve q\to q$ changes only the stable marginal under Assumption~\ref{ass:stability}. We account for the latter through the \emph{bridge covariate shift term} $C_{\rm br}(\vw)=|R_q(\vw)-R_{\breve q}(\vw)|$.
We fix a measurable version of the target conditional on the source stable-label support and assume that $p,q,\breve q$ have finite second feature moments. Let $\Pi_{\rm cond}$ contain couplings of $p$ and $\breve q$ that share $(\rvz_{\mathcal S},\ry)$, and let $\Delta=\rvz'-\rvz\in\mathcal U$ under such a coupling.

\begin{theorem}[Risk decomposition]
\label{thm:risk-decomposition}
For every $\vw\in\mathbb R^d$,
\begin{equation}
R_q(\vw)\le R_p(\vw)+C_{\rm br}(\vw)
 +\inf_{\pi\in\Pi_{\rm cond}}\mathbb E_\pi|\vw^\top\Delta|.
\label{eq:risk-decomposition}
\end{equation}
If a conditional coupling satisfies $\mathbb E_\pi[\Delta\Delta^\top]\preceq\rho_\Delta^2P_{\mathcal U}$ for $\rho_\Delta\ge0$, the last term is at most $\rho_\Delta\|P_{\mathcal U}\vw\|_2$.
\end{theorem}

The three terms are source prediction risk, bridge covariate shift, and direct unstable shift. Stable predictors eliminate the last term, while changes to the stable marginal can still change their risk. Appendix~\ref{app:common-bayes} explains when a common stable Bayes rule remains optimal, and Appendix~\ref{app:bridge-proof} bounds the exact bridge term by transport. Appendix~\ref{app:logistic-transport-refined} sharpens the direct-shift bound for learned logistic predictors.

The direct-shift bound separates deployment-shift size $\rho_\Delta$ from predictor sensitivity $\|P_{\mathcal U}\vw\|_2$. By estimating which directions a predictor should ignore, pairs can help reduce this sensitivity without reproducing the target distribution. Here $\mathcal U$ is the unstable complement from Assumption~\ref{ass:stability}; the bound on shift size additionally uses the conditional-displacement condition in Theorem~\ref{thm:risk-decomposition}.

\subsection{When does a collection reveal the unstable subspace?}
\label{sec:pair-information}
The coverage and precision requirements from Section~\ref{sec:setup} can be expressed through the pair covariance. The following result characterizes when its leading eigenspace identifies $\mathcal U$, even if individual differences also contain components in $\mathcal S$.

\begin{proposition}[Exact spectral recovery]
\label{prop:spectral-recovery}
Let $1\le r_*=\dim(\mathcal U)<d$ and $\boldsymbol M\succeq0$. Its unique leading $r_*$-dimensional eigenspace is $\mathcal U$ if and only if
$P_{\mathcal U}\boldsymbol M P_{\mathcal S}=0$ and
$\lambda_{\min}(\boldsymbol M|_{\mathcal U})>\lambda_{\max}(\boldsymbol M|_{\mathcal S})$.
\end{proposition}

The first condition prevents the covariance from mixing stable and unstable directions. The second requires its smallest eigenvalue within $\mathcal U$ to exceed its largest within $\mathcal S$, so all unstable directions are selected first. Equality at this boundary would make the selected subspace nonunique. These conditions characterize leading-eigenspace selection, rather than every possible estimator.

For the remaining pair guarantees, fix $1\le r_*:=\dim(\mathcal U)<d$.
\begin{assumption}[Conditional pair noise]
\label{ass:pairs}
Conditional on clean signals $\vu_1,\ldots,\vu_k$ spanning $\mathcal U$, the observed differences satisfy $\vdelta_j=\vu_j+\vxi_j$, where $\vu_j\in\mathcal U$. The noises are independent, mean zero, have covariance $\sigma_\xi^2\mI$, and satisfy $\sup_{\|\vv\|_2=1}\|\vv^\top\vxi_j\|_{\psi_2}\le K\sigma_\xi$ for $\sigma_\xi\ge0$ and fixed $K\ge1$.
\end{assumption}

Define $\mSigma_k=k^{-1}\sum_j\vu_j\vu_j^\top$, with positive eigenvalues $\lambda_{1,k}\ge\cdots\ge\lambda_{r_*,k}>0$. Its noisy conditional mean is $\mSigma_k+\sigma_\xi^2\mI$, which recovers $\mathcal U$ by Proposition~\ref{prop:spectral-recovery}. Sampling fluctuation can rotate this eigenspace. The next theorem controls its prediction cost; Appendix~\ref{app:concentration} also states a consistency criterion along growing signal collections.

\subsection{Finite-pair guarantee for GRIT}
\label{sec:grit-guarantee}
We now connect recovery to prediction. Define $g_k=\lambda_{r_*,k}$ for the weakest clean signal, and let $L_z=\|\mathbb E_p[\rvz\rvz^\top]\|_{\rm op}^{1/2}$. Assume $\mathcal W_{\mathcal S}^*=\arg\min_{\vw\in\mathcal S}R_p(\vw)$ is nonempty, with risk $R_{p,\mathcal S}^*$. All rank projectors use one ordered eigenbasis and a fixed tie rule.

\begin{theorem}[GRIT with sufficient rank]
\label{thm:hard-grit}
Assume Assumptions~\ref{ass:stability}--\ref{ass:pairs} and the moment and conditional-shift conditions of Theorem~\ref{thm:risk-decomposition}. For every $\eta\in(0,1)$, conditional on the signals, with probability at least $1-\eta$, simultaneously for $r_*\le r\le d$, every existing minimizer $\vw_r\in\arg\min_{P_{k,r}\vw=0}R_p(\vw)$ and every $\vw_{\mathcal S}^*\in\mathcal W_{\mathcal S}^*$ satisfy
\begin{equation}
\begin{aligned}
R_q(\vw_r)\le{}&\underbrace{R_{p,\mathcal S}^*}_{(1)}
 +\underbrace{C_{\rm br}(\vw_r)}_{(2)}
 +\underbrace{L_z\|(P_{k,r}-P_{k,r_*})\vw_{\mathcal S}^*\|_2}_{(3)}\\
 &+\underbrace{C_{r,*}\left[
 \left(\sqrt{\frac{\sigma_\xi^2}{g_k}}+\frac{\sigma_\xi^2}{g_k}\right)
 \sqrt{\frac{d+\log(1/\eta)}{k}}
 +\frac{\sigma_\xi^2}{g_k}\frac{d+\log(1/\eta)}{k}
 \right]}_{(4)},
\end{aligned}
\label{eq:grit-main-bound}
\end{equation}
where $C_{r,*}=C_K(L_z\|\vw_{\mathcal S}^*\|_2+\rho_\Delta\|\vw_r\|_2)$ and $C_K$ depends only on $K$.
If $k\ge d+\log(1/\eta)$, the bracket in term (4) is at most
$2(\sqrt{\sigma_\xi^2/g_k}+\sigma_\xi^2/g_k)\sqrt{(d+\log(1/\eta))/k}$.
\end{theorem}

Terms (1) and (2) retain stable prediction risk and bridge covariate shift. Term (3) bounds the cost of extra-direction removal: it vanishes at $r=r_*$ or when the extra directions are orthogonal to the stable oracle. Term (4) covers source prediction cost and residual unstable sensitivity from subspace estimation, with task and shift dependence in $C_{r,*}$. Because the bound holds simultaneously for every stable source optimum, one may use the optimum giving the smallest sum of terms (3) and (4). The theorem studies fitting with population source risk; its probability accounts for the $k$ pair noises, not for sampling the labeled training data.

The noise-to-signal ratio $\sigma_\xi^2/g_k$ governs recovery. Balanced coverage helps by raising the weakest signal $g_k$ at fixed total signal energy. The $1/k$ term keeps the bound valid when $k<d+\log(1/\eta)$. With the displayed ratio and predictor factors controlled, the larger-$k$ regime gives a $1/\sqrt{k}$ rate. Appendix~\ref{app:spectral-refinements} proves the sharper subspace estimate, and Appendix~\ref{app:grit-tightness} gives source-dependent and quadratic extra-removal refinements.
The next proposition gives a tightness result for the finite-pair term: some problem instances attain its dependence on pair count and noise, up to constant factors, even after optimal source fitting under GRIT's estimated constraint.
\begin{proposition}[Tightness of the finite-pair rate]
\label{prop:grit-target-tightness}
There exist constants $c,\delta_0>0$ and fixed source and target distributions satisfying Assumption~\ref{ass:stability} and the moment and conditional-shift conditions, with $d=2$, $r=r_*=1$, $\rho_\Delta=1$, and zero bridge shift, such that the following holds. For every $g,\sigma>0$ and integer $k\ge1$ with
$\delta_k:=\sqrt{(\sigma^2/g+\sigma^4/g^2)/k}\le\delta_0$,
there is a Gaussian design satisfying Assumption~\ref{ass:pairs} with weakest clean eigenvalue $g$ and noise variance $\sigma^2$ for which the GRIT predictor $\vw_k$, fitted to minimize population source risk under its estimated constraint, satisfies with probability at least $1/64$,
\begin{equation}
R_q(\vw_k)-R_{p,\mathcal S}^*
\ge c\,\delta_k\bigl(L_z\|\vw_{\mathcal S}^*\|_2+\rho_\Delta\|\vw_k\|_2\bigr).
\label{eq:grit-target-lower}
\end{equation}
The fitted optimum exists for every pair realization, with uniformly bounded source and target feature norms and fitted coefficient norms.
\end{proposition}

Here $R_{p,\mathcal S}^*$ is the risk of the best stable source predictor. With probability at least $1/64$, GRIT's target risk exceeds this reference by an amount of order $\delta_k$ times the displayed prediction and shift factors, even though its source objective is minimized exactly. Both the loss of source predictive information and the remaining sensitivity to unstable shift contribute to this excess. For fixed dimension and confidence and sufficiently small $\delta_k$, this attains the pair-count and noise dependence of term (4) in Theorem~\ref{thm:hard-grit}, up to constants. Appendix~\ref{app:grit-target-tightness} proves the result, and Appendix~\ref{app:grit-bridge-tightness} gives a further construction with positive bridge covariate shift that attains a refined target-risk comparison asymptotically.

\begin{proposition}[Underselection]
\label{prop:underselection}
If $r<r_*$, then $\dim(\mathcal U\cap\widehat{\mathcal U}_r^\perp)\ge r_*-r$. For every fixed feasible $\vw$ with $P_{\mathcal U}\vw\ne0$ and every source $p$ satisfying Assumption~\ref{ass:stability} with finite second feature moments, there are feature-space targets $q_t$ sharing $p(\vz_{\mathcal S},y)$ such that
$\mathbb P_{q_t}(\ry\vw^\top\rvz<0)\to1$ and $R_{q_t}(\vw)\to\infty$ as $t\to\infty$.
\end{proposition}

Underselection leaves unstable directions available; failure requires the fitted predictor to use them. This construction has zero bridge shift and growing direct-shift radius.

\subsection{Guarantee for quadratic prediction consistency}
\label{sec:cr-guarantee}
Quadratic prediction consistency is a natural comparison because it uses the same invariant pairs to encourage similar predictions. Its penalty acts on all directions of observed pair variation, including noise, whereas GRIT enforces invariance only to the estimated subspace. Analyzing this penalty separately reveals how its strength trades off preservation of stable predictive information against sensitivity to unstable directions.

 Building on consistency-based robust prediction \citep{heinzedeml2021conditional,yang2023sample}, we use the fitted optimum's first-order condition to obtain a guarantee. For $\eta\in(0,1)$, let $h=d+\log(1/\eta)$ and define the covariance-error radius
\begin{equation}
e_k=C_K\left[
 \sigma_\xi(\sqrt{\lambda_{1,k}}+\sigma_\xi)\sqrt{\frac hk}
 +\sigma_\xi^2\frac hk\right],
\label{eq:pair-error-scale}
\end{equation}
valid for every $k$ under Assumption~\ref{ass:pairs} (Lemma~\ref{lem:pair-concentration}). Conditional on the clean signals, $e_k$ bounds the operator-norm deviation of the observed pair covariance from its mean $\mSigma_k+\sigma_\xi^2\mI$ with probability at least $1-\eta$.

\begin{theorem}[Quadratic prediction consistency]
\label{thm:quadratic-cr}
Under Assumptions~\ref{ass:stability}--\ref{ass:pairs} and the moment, shift, and stable-optimizer existence conditions above, with conditional probability at least $1-\eta$, every existing minimizer $\vw_\lambda\in\arg\min_{\vw}\{R_p(\vw)+\lambda\vw^\top\widetilde{\mSigma}_k\vw\}$, $\lambda>0$, satisfies, for every $\vw_{\mathcal S}^*\in\mathcal W_{\mathcal S}^*$,
\begin{equation}
R_q(\vw_\lambda)\le\underbrace{R_{p,\mathcal S}^*}_{(1)}
+\underbrace{C_{\rm br}(\vw_\lambda)}_{(2)}
+\underbrace{\frac{\lambda\sigma_\xi^2}{2}\|\vw_{\mathcal S}^*\|_2^2
+\frac{\rho_\Delta^2}{8\lambda(g_k+\sigma_\xi^2)}}_{(3)}
+\underbrace{2\lambda e_k\|\vw_\lambda\|_2\|\vw_{\mathcal S}^*-\vw_\lambda\|_2}_{(4)}.
\label{eq:cr-main-bound}
\end{equation}
\end{theorem}

Group (3) balances preservation of stable information against residual unstable sensitivity. Increasing $\lambda$ reduces the second contribution but increases the first: isotropic pair noise adds a stable-direction penalty in the mean covariance. Using optimizer stationarity halves both coefficients relative to an objective-value comparison, without adding assumptions. Group (4) captures finite-pair covariance fluctuation, whose concentration depends on the largest clean eigenvalue.

Our analysis of quadratic prediction consistency applies independently of GRIT and may be of interest in its own right. Complementing existing consistency analyses, it uses the fitted optimizer's first-order condition to connect regularization strength and pair covariance to target risk. This perspective clarifies how noise in invariant pairs affects the balance between preserving stable predictive information and reducing unstable sensitivity. Appendix~\ref{app:cr-optimizer-refined} develops this analysis further through a sharper comparison retaining signed terms and additional curvature conditions under which the stable-information cost becomes quadratic in small regularization strength.
 \section{Related Work}
\label{sec:related}

\paragraph{Robust learning from paired observations.}
CoRe penalizes within-identity prediction or loss variation, while MatchDG learns matches and regularizes their representations \citep{heinzedeml2021conditional,mahajan2021domain}.
LAM matches feature contributions to logits, and IPG combines attribution alignment with class separation to guide corrective updates \citep{gao2024consistency,surner2026invariance}; these differ from the scalar prediction penalty analyzed here.
Label-changing counterfactual examples provide a different signal for learning and gradient supervision \citep{kaushik2021explaining,teney2020learning}.
Consistency learning already has statistical and robustness guarantees, including intervention models and misspecified augmentations \citep{heinzedeml2021conditional,yang2023sample}.
Our analysis treats GRIT and quadratic prediction consistency in a common finite, noisy-pair setting.

\paragraph{Removing nuisance directions from representations.}
Nuisance attribute projection removes channel-related directions in speaker recognition \citep{solomonoff2005advances}, while word-embedding debiasing estimates a subspace from definitional pairs and removes it from selected representations \citep{bolukbasi2016man}.
In vision-language models, biased prompts calibrate a linear transformation \citep{chuang2023debiasing}, and RoboShot uses language-derived directions to remove harmful components and strengthen useful ones for zero-shot prediction \citep{adila2024zero}.
LEACE instead uses concept labels for minimum-distortion linear erasure, generally through an oblique map \citep{belrose2023leace}.
GRIT studies a reusable orthogonal projection estimated from supplied finite, noisy invariant pairs and its consequences for downstream target risk.

\paragraph{Predictive invariance and invariant representations.}
Invariant causal prediction and invariant transfer model persistent predictive relationships \citep{peters2016causal,rojascarulla2018invariant}, whereas IRM asks for a common optimal predictor \citep{arjovsky2019invariant}; the connection depends on loss, predictor class, and realizability \citep{kamath2021does}.
Paired-view identifiability results recover shared content up to an invertible map under latent content/style assumptions \citep{vonkugelgen2021self}, a different goal from estimating an unstable subspace in fixed features.
Distribution alignment constrains feature distributions through domain-adversarial adaptation (DANN), conditional adversarial alignment for DG (C-DANN), or correlation alignment (CORAL) \citep{ganin2016domain,li2018deep,sun2016deep}.
GRIT constrains directions without requiring feature distributions to agree across environments.

\paragraph{Complementary approaches to OOD robustness.}
Group-balanced last-layer retraining \citep{kirichenko2023last} and invariant-subspace recovery from class-conditional environment moments \citep{wang2022provable} provide fixed-feature and geometric precedents with different information requirements.
Anchor regression uses anchor variables to obtain squared-loss robustness against structured shifts in linear models \citep{rothenhausler2021anchor}.
Other methods control worst-group risk, cross-environment optimality, risk or gradient variation \citep{sagawa2020distributionally,arjovsky2019invariant,krueger2021out,rame2022fishr}, or use selective interpolation \citep{yao2022improving}.
Spectral Decoupling (SD) penalizes squared logits to counter gradient starvation \citep{pezeshki2021gradient}, while Risk Distribution Matching (RDM) aligns the worst environment's distribution of individual losses with their pooled distribution using maximum mean discrepancy \citep{nguyen2024domain}.
Projection can precede compatible objectives; we next evaluate this compatibility and compare GRIT with prediction consistency.
 \section{Experimental Evaluation}
\label{sec:experiments}
Our experimental code is publicly available.\footnote{\url{https://github.com/inouye-lab/GRIT}.}
We augment ColoredMNIST \citep{arjovsky2019invariant} and Waterbirds \citep{sagawa2020distributionally} with invariant pairs by construction, calling the resulting variants \emph{ColoredMNIST-INV} and \emph{Waterbirds-INV}. The suffix denotes the availability of invariant-pair information.
ColoredMNIST-INV provides 256 pairs showing the same training digit in both colors. For Waterbirds-INV, we follow the original dataset's foreground--background compositing methodology to construct 240 pairs showing the same bird against land and water backgrounds. Appendix~\ref{app:experimental-data} gives the exact constructions and how the endpoints enter training.
For the objective comparison, we use frozen, unnormalized, 512-dimensional CLIP ViT-B/32 features \citep{radford2021learning} on both variants.
These constructions preserve task content without establishing exact subspace geometry or feature-noise assumptions. The frozen-feature probes differ from the original benchmarks in representation and training protocol, so absolute accuracies are not directly comparable.

Each objective uses the same labeled dataset for vanilla, prediction
consistency (PC), and GRIT. PC penalizes paired logit differences, representing
output-consistency components of CoRe, MatchDG, and SelfReg
\citep{heinzedeml2021conditional,mahajan2021domain,kim2021selfreg}.\footnote{MatchDG's reported second-phase
architecture takes the predictor after its representation to be the identity,
so its matching penalty acts on logits \citep[App.~C.1]{mahajan2021domain}.}
GRIT projects before fitting. Both pair-aware variants use the same pair bank
and frozen CLIP features, comparing a representative consistency penalty with
GRIT's geometric constraint.

On ColoredMNIST-INV, Waterbirds-INV, and UrbanCars, all methods select hyperparameters and checkpoints using validation data, with group labels for Waterbirds-INV and UrbanCars and held-out color environments for ColoredMNIST-INV. Appendices~\ref{app:experiments} and~\ref{app:urbancars} detail information access, validation, and seed aggregation, and report sample standard deviations.

\subsection{Does GRIT improve robustness across objectives?}
\label{sec:objective-results}
Table~\ref{tab:pair-matrix} shows higher GRIT performance than vanilla for all eight
objectives on both datasets. ERM gains are 50.87 and 17.40 points, respectively.
Four-group GroupDRO already reaches 79.83 on Waterbirds-INV; GRIT raises this
to 81.91, a 2.08-point gain.

\begin{table}[htbp]
\caption{Mean OOD accuracy (\%, ten seeds): ColoredMNIST-INV test accuracy and
Waterbirds-INV worst-group accuracy. Within each objective and dataset,
\textbf{bold}/\textit{italic} mark the best/second-best mean.
Solid/dotted boxes mark the best/second-best mean among all entries for each
dataset. PC: prediction consistency; +PC and +GRIT add the corresponding
pair-based method to the row objective.
Appendix~\ref{app:result-verification} reports sample standard deviations.}
\label{tab:pair-matrix}
\vspace{-1em}
\centering
\small
\newcommand{\coresolid}[1]{\tikz[baseline=(v.base)]{\node[inner sep=0pt,outer sep=0pt] (v) {#1};\draw[overlay,line width=0.5pt] ([xshift=-2pt,yshift=-1.5pt]v.south west) rectangle ([xshift=2pt,yshift=1.5pt]v.north east);}}
\newcommand{\coredotted}[1]{\tikz[baseline=(v.base)]{\node[inner sep=0pt,outer sep=0pt] (v) {#1};\draw[overlay,line width=0.5pt,densely dotted] ([xshift=-2pt,yshift=-1.5pt]v.south west) rectangle ([xshift=2pt,yshift=1.5pt]v.north east);}}
\setlength{\tabcolsep}{3pt}
\begin{tabular}{lrrrrrr}
\toprule
& \multicolumn{3}{c}{ColoredMNIST-INV (CLIP feat.)} & \multicolumn{3}{c}{Waterbirds-INV (CLIP feat.)}\\
\cmidrule(lr){2-4}\cmidrule(lr){5-7}
Objective & Vanilla & +PC & +GRIT & Vanilla & +PC & +GRIT\\
\midrule
ERM & 17.02 & \textit{57.16} & \textbf{67.89} & 52.69 & \textit{66.84} & \textbf{70.09}\\
SD & 17.02 & \textit{56.68} & \textbf{67.77} & 52.06 & \textit{66.82} & \textbf{70.09}\\
RDM & 39.95 & \textit{59.55} & \textbf{67.60} & 52.80 & \textbf{78.59} & \textit{76.29}\\
IRMv1 & 59.89 & \textit{61.28} & \coresolid{\textbf{70.27}} & 47.78 & \textit{77.86} & \textbf{78.78}\\
V-REx & 54.83 & \textit{59.98} & \textbf{68.20} & 48.27 & \textit{78.09} & \textbf{80.99}\\
Fishr & 65.28 & \textit{65.33} & \coredotted{\textbf{68.91}} & 47.16 & \textit{76.72} & \textbf{80.84}\\
GroupDRO & 30.90 & \textit{61.51} & \textbf{68.01} & 79.83 & \textit{81.66} & \coredotted{\textbf{81.91}}\\
LISA & 36.68 & \textit{61.62} & \textbf{67.80} & 79.65 & \textit{80.77} & \coresolid{\textbf{82.30}}\\
\bottomrule
\end{tabular}
\vspace{-0.5em}
\end{table}
 
GRIT exceeds prediction consistency in 15/16 settings; prediction consistency
wins with RDM on Waterbirds-INV.
Bold marks the highest mean within each objective and dataset; for GroupDRO on Waterbirds-INV, the paired 95\% interval for GRIT minus prediction consistency includes zero
(Appendix~\ref{app:result-verification}).

\subsection{What controls recovery from noisy invariant pairs?}
\label{sec:cg-sweeps}
\label{sec:geometry-results}
We test the pair-information dependencies in Theorem~\ref{thm:hard-grit} on
\emph{CueMNIST}: binary MNIST \citep{lecun1998gradient} with two independently controlled color cues.
Its engineered 96-dimensional features give a known two-dimensional unstable
subspace. Each cue matches the source label with probability $0.95$;
the target reverses both cues. All methods use the same 3,000 labeled examples
and ridge-logistic fitting. GRIT removes rank two; prediction
consistency (PC), here a trace-normalized penalty on the scalar score, selects its strength by
worst-group validation accuracy.
We vary the smallest clean eigenvalue, pair noise, and pair count separately.
Geometry is shown at three secondary settings; accuracy retains the original
defaults. The smallest-eigenvalue sweep fixes the largest eigenvalue at one;
its total clean energy varies. Independent Gaussian noise is added to the constructed feature-space pair differences.

\begin{figure}[!t]
\centering
\includegraphics[width=\linewidth]{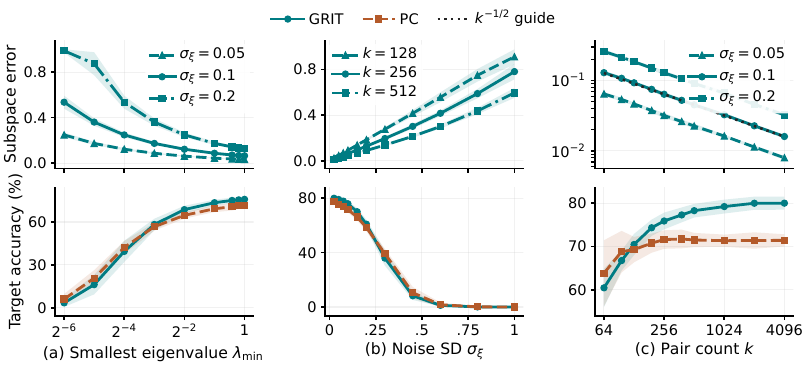}
\vspace{-1em}
\caption[Weakest signal, noise and pair count control CueMNIST recovery]{\textbf{Stronger signal in the weakest direction and less noise improve recovery and target accuracy; more pairs yield near-$k^{-1/2}$ subspace error and a GRIT advantage over PC at larger budgets.} On CueMNIST, the top row shows GRIT's subspace error; the bottom compares target accuracy. Defaults are $k=256$, $\sigma_\xi=0.1$, and clean spectrum $(1,1)$. Top-panel legends vary a second setting; bottom panels use the defaults. Bands show sample SD across 10 replicates; the dotted line is a shape guide. Y-axis ranges differ by panel.
 }
\label{fig:cg-sweeps}
\vspace{-0.5em}
\end{figure}

Figure~\ref{fig:cg-sweeps} supports the predicted dependencies: stronger signal
in the weakest direction, less noise, and more pairs improve recovery, and the
corresponding accuracy trends give evidence that these factors affect robustness.
Pair count alone therefore does not characterize collection quality.
PC is competitive at small budgets; at 4,096 pairs, GRIT reaches 80.0\% target
accuracy versus 71.4\% for PC under the tested selection protocol.
Appendix~\ref{app:cg-sweeps} gives the complete construction, fitted-instability
diagnostics, Gaussian replication, and comparison limitations.
 \subsection{Can two pair types encode complementary invariances?}
\label{sec:urbancars}
Section~\ref{sec:setup} motivates combining pair types to cover different
unstable variations. UrbanCars \citep{li2023whac} makes this need concrete:
both the background and a co-occurring object correlate with the car label
during training. One type of invariant pair changes only the background; the other
changes only the object. We vary their allocation within a fixed budget of
256 pairs on frozen CLIP features. Hyperparameters and checkpoints are selected
by validation worst-group accuracy.

\begin{table}[!ht]
\centering
\caption{Combining two pair types improves UrbanCars robustness.
$n_B,n_O$ count pairs that vary the background and object, respectively. B/O shift changes the indicated
factor's label alignment, holding the other aligned; worst group is the minimum
across eight label--background--object groups. Accuracy (\%); bold marks each
row's best method per metric. Solid/dotted boxes mark the highest/second-highest
distinct mean across all allocations and methods within each metric.
+PC/+GRIT add prediction consistency/GRIT to ERM.}
\label{tab:urbancars}
\vspace{-1em}
\small
\newcommand{\urbansolid}[1]{\tikz[baseline=(v.base)]{\node[inner sep=0pt,outer sep=0pt] (v) {#1};\draw[overlay,line width=0.5pt] ([xshift=-2pt,yshift=-1.5pt]v.south west) rectangle ([xshift=2pt,yshift=1.5pt]v.north east);}}
\newcommand{\urbandotted}[1]{\tikz[baseline=(v.base)]{\node[inner sep=0pt,outer sep=0pt] (v) {#1};\draw[overlay,line width=0.5pt,densely dotted] ([xshift=-2pt,yshift=-1.5pt]v.south west) rectangle ([xshift=2pt,yshift=1.5pt]v.north east);}}
\setlength{\tabcolsep}{3.5pt}
\begin{tabular}{rr ccc ccc ccc}
\toprule
\multicolumn{11}{c}{UrbanCars (CLIP feat.)} \\
\multicolumn{2}{c}{Pair counts} & \multicolumn{3}{c}{B shift} & \multicolumn{3}{c}{O shift} & \multicolumn{3}{c}{Worst group} \\
\cmidrule(lr){1-2}\cmidrule(lr){3-5}\cmidrule(lr){6-8}\cmidrule(lr){9-11}
$n_B$ & $n_O$ & ERM & +PC & +GRIT & ERM & +PC & +GRIT & ERM & +PC & +GRIT \\
\midrule
256 & 0   & 45.4 & \urbandotted{85.7} & \urbansolid{\textbf{88.5}} & \textbf{76.0} & 65.1 & 68.4 & 20.4 & 52.9 & \textbf{56.7} \\
128 & 128 & 45.4 & 76.0 & \textbf{79.1} & 76.0 & 84.4 & \urbandotted{\textbf{85.1}} & 20.4 & \urbandotted{68.3} & \urbansolid{\textbf{71.9}} \\
0   & 256 & 45.4 & \textbf{54.7} & 53.9 & 76.0 & 80.2 & \urbansolid{\textbf{94.1}} & 20.4 & 42.6 & \textbf{43.8} \\
\bottomrule
\end{tabular}
\vspace{-0.5em}
\end{table}

Table~\ref{tab:urbancars} shows the tradeoff: each single-type collection favors
its corresponding shift, while combining both raises GRIT's worst-group
accuracy from 56.7\% (pairs that vary only the background) and 43.8\% (pairs that vary only the object) to 71.9\%.
This is consistent with the coverage requirement in Section~\ref{sec:setup}: a single pair type leaves much of the other factor's variation unrepresented.
GRIT exceeds ERM and prediction consistency in worst-group mean accuracy for
every allocation shown. Results average ten seeds and, for pair methods, three
pair banks. The benefit of mixing pair types is smaller when validation is
weighted toward the training distribution (Appendix~\ref{app:urbancars}).
  \section{Discussion and Conclusion}
\label{sec:discussion}
GRIT learns which variations to ignore from finite, noisy invariant pairs and enforces the resulting constraint before predictor fitting. Our theory connects pair count, noise, coverage, and selected rank to target risk, and provides an independent analysis of quadratic prediction consistency. Experiments show improvements across the tested objectives and the value of combining complementary pair types. These results provide a foundation for extending both the predictors and the geometry learned from pairs. Nonlinear heads can use the same projection, while relaxing Assumption~\ref{ass:stability}'s linear stable subspace presents a deeper challenge (Appendix~\ref{app:discussion-geometry}). Progress also depends on acquiring pairs that preserve task information and cover relevant variation; broader studies are needed to assess these requirements across applications (Appendix~\ref{app:discussion-pairs}). By separating this application-specific knowledge from predictor fitting, GRIT offers a starting point for learning procedures that can turn diverse sources of invariant pairs into more robust predictions.
 \clearpage
\section*{Acknowledgments}
The authors acknowledge support from NSF (IIS-2212097) and ARL (W911NF-2020-221).

\bibliography{references}
\bibliographystyle{iclr2027_conference}
\clearpage
\appendix
\section{Proofs and Technical Details}
\label{app:theory}
The results below concern population prediction risk and randomness in the pairs. All shift radii $\rho_\Delta,\rho_{\rm br},\rho_z,\rho_y$ are nonnegative. We do not assume that the clean pair signals are sampled from a common distribution. We state optimizer existence explicitly because an unregularized logistic objective can have an unattained infimum. An additive population objective gap $\epsilon_{\rm opt}$ can be added to the hard-GRIT and objective-comparison CR bounds below. The stationarity-based CR and learned-logistic refinements use exact minimizers. Additional generalization arguments would be needed for empirical labeled-sample minimizers.

\subsection{Bridge distributions, risk decomposition, and common stable prediction}
\label{app:bridge-proof}
All distributions are on finite-dimensional Euclidean feature spaces and the stated label space, so regular conditional distributions exist. Let $\mu_p$ denote the source distribution of $(\rvz_{\mathcal S},\ry)$. Fix a measurable kernel $K_q(d\vz_{\mathcal U}\mid\vz_{\mathcal S},y)$ that agrees with the target conditional on its almost-sure domain, and specify its values on any additional set charged by $\mu_p$. The bridge is $\breve q=\mu_pK_q$. Its value, its moments, and the resulting bound may depend on this extension when the stable-label supports differ. A conditional product coupling establishes that $\Pi_{\rm cond}$ is nonempty; a smaller displacement can be obtained from other conditional couplings. Finite second feature moments of $p,q,\breve q$ suffice for every fixed finite linear score used here.

\paragraph{Proof of Theorem~\ref{thm:risk-decomposition}.}
\begin{proof}
The proof compares losses along a coupling that holds the stable features and label fixed, then separates the exact bridge risk difference. Put $\varphi(a)=\log(1+e^{-a})$. Since $\varphi'(a)=-1/(1+e^a)$, its absolute derivative is at most one. For every $\pi\in\Pi_{\rm cond}$,
\begin{align}
R_q(\vw)
&=R_p(\vw)+[R_{\breve q}(\vw)-R_p(\vw)]
 +[R_q(\vw)-R_{\breve q}(\vw)],\label{eq:bridge-proof-1}\\
R_{\breve q}(\vw)-R_p(\vw)
&=\mathbb E_\pi[\varphi(\ry\vw^\top\rvz')-
                  \varphi(\ry\vw^\top\rvz)],\label{eq:bridge-proof-2}\\
&\le\mathbb E_\pi|\vw^\top(\rvz'-\rvz)|,\label{eq:bridge-proof-3}\\
&\le\sqrt{\vw^\top\mathbb E_\pi[\Delta\Delta^\top]\vw}.
\label{eq:bridge-proof-4}
\end{align}
Equation~\eqref{eq:bridge-proof-2} uses the two coupling marginals and shared label;~\eqref{eq:bridge-proof-3} uses the derivative bound and $|\ry|=1$;~\eqref{eq:bridge-proof-4} is Cauchy--Schwarz. Bounding the last bracket in~\eqref{eq:bridge-proof-1} by its absolute value and taking the infimum in~\eqref{eq:bridge-proof-3} proves~\eqref{eq:risk-decomposition}. Under the stated displacement bound,~\eqref{eq:bridge-proof-4} is at most $\rho_\Delta\|P_{\mathcal U}\vw\|_2$.
\end{proof}

\paragraph{Bridge transport.}
If $\gamma$ couples $q$ with $\breve q$, possibly with different labels, define
$\boldsymbol\eta_\gamma=\ry^q\rvz^q-\ry^{\rm br}\rvz^{\rm br}$.
The same derivative bound, applied to signed margins, gives
\begin{equation}
C_{\rm br}(\vw)\le\inf_\gamma\mathbb E_\gamma
 |\vw^\top\boldsymbol\eta_\gamma|
\le\inf_\gamma\sqrt{\vw^\top\mathbb E_\gamma[
 \boldsymbol\eta_\gamma\boldsymbol\eta_\gamma^\top]\vw}.
\label{eq:bridge-transport}
\end{equation}
In particular, the existence of a coupling with signed second moment at most $\rho_{\rm br}^2\mI$ gives $C_{\rm br}(\vw)\le\rho_{\rm br}\|\vw\|_2$. This is an optional upper bound on the exact bridge term.

\paragraph{A sharp direct-shift coefficient.}
\label{app:sharpness}
The coefficient one in~\eqref{eq:bridge-proof-3} cannot be uniformly reduced, even when the bridge term is zero and the direct displacement is fixed. Set $d=2$, $\mathcal S=\operatorname{span}(\ve_1)$, $\mathcal U=\operatorname{span}(\ve_2)$, and $\vw=\ve_2$. For fixed $a>0$ and each $b>0$, take $\ry=1$ and deterministic source and target features $(0,-b)$ and $(0,-b-a)$. Assumption~\ref{ass:stability} holds, $\breve q=q$, and $\mathbb E|\vw^\top\Delta|=a$. Direct calculation gives
\begin{equation}
\frac{R_q(\vw)-R_p(\vw)}{\mathbb E|\vw^\top\Delta|}
=\frac{\log(1+e^{b+a})-\log(1+e^b)}{a}\longrightarrow1.
\label{eq:logistic-sharpness}
\end{equation}
Indeed, $\log(1+e^c)=c+\log(1+e^{-c})$, and the two remainders vanish as $b\to\infty$. Every member has finite moments and direct radius $a$. This establishes sharpness of the transport coefficient for arbitrary fixed predictors. For fitted logistic predictors, Appendix~\ref{app:logistic-transport-refined} sharpens the transport bound; Appendix~\ref{app:grit-target-tightness} separately establishes matching target-risk dependence for fitted GRIT.

\subsubsection{A common Bayes rule can have different risks}
\label{app:common-bayes}
\begin{proposition}[Common stable logistic predictions]
\label{prop:common-bayes}
Suppose Assumption~\ref{ass:stability} gives a measurable conditional probability $\theta(\vz_{\mathcal S})=\mathbb P(\ry=1\mid\rvz_{\mathcal S}=\vz_{\mathcal S})\in(0,1)$ almost surely under both stable marginals. Suppose there is a finite $\vw_{\mathcal S}^*\in\mathcal S$ with
$\vw_{\mathcal S}^{*\top}\vz_{\mathcal S}=\log(\theta/(1-\theta))$ almost surely under both. Then this score minimizes logistic risk among all measurable stable scores in both distributions. If $q_{\mathcal S}\ll p_{\mathcal S}$, every stable source minimizer with finite risk has the same score $q_{\mathcal S}$-almost surely and is also target optimal among stable scores.
\end{proposition}
\begin{proof}
The shared conditional risk is minimized pointwise; strictness identifies source predictions, and absolute continuity transfers that identification. At a fixed stable feature value with conditional probability $\theta$, define
\begin{equation}
L_\theta(a)=\theta\log(1+e^{-a})+(1-\theta)\log(1+e^a).
\end{equation}
Its derivatives are $L_\theta'(a)=(1+e^{-a})^{-1}-\theta$ and
$L_\theta''(a)=e^{-a}/(1+e^{-a})^2>0$. Thus its unique minimum is at $a^*=\log(\theta/(1-\theta))$. The conditional excess $L_\theta(a)-L_\theta(a^*)$ is nonnegative and equals zero precisely at $a=a^*$. Integrating proves optimality under each stable marginal. A source minimizer has zero expected conditional excess, hence agrees with $a^*$ $p_{\mathcal S}$-almost surely. Absolute continuity then yields agreement $q_{\mathcal S}$-almost surely. This proves equality of predictions; it does not require uniqueness of coefficient vectors.
\end{proof}
The shared optimum has conditional risk
$-\theta\log\theta-(1-\theta)\log(1-\theta)$, whose expectation can differ across stable marginals. Assumption~\ref{ass:stability} therefore does not imply equal optimal risks. The realizability assumption is also material: a misspecified linear class can have different optimal approximations under different marginals. Absolute continuity is a sufficient condition for transferring source-almost-sure identification; topological support containment alone would not exclude a target atom on a source null set.

\subsection{Spectral equivalence and exact recovery}
\label{app:equivalence}
\paragraph{Proof of Proposition~\ref{prop:method-equivalence}.}
\begin{proof}
The GRIT constraint fixes the retained subspace, while a quadratic penalty specifies the full covariance. Let $P,P'$ project onto the leading rank-$r$ eigenspaces of $C,C'$. Equality of the feasible sets means $\ker P=\ker P'$, or equivalently that both selected subspaces equal $\mathcal V$. We now characterize when $C'$ selects $\mathcal V$.
If it does, symmetry makes $\mathcal V^\perp$ invariant as well, so $P_{\mathcal V}C'P_{\mathcal V^\perp}=0$. The strict rank-$r$ gap gives $\lambda_{\min}(C'|_{\mathcal V})>\lambda_{\max}(C'|_{\mathcal V^\perp})$.
Conversely, the zero cross block decomposes $C'$ into its restrictions to the two subspaces. The strict ordering places all $r$ eigenvalues from $\mathcal V$ first, so its leading subspace is exactly $\mathcal V$ and the GRIT constraints coincide.
For fixed $\lambda>0$, equality of quadratic penalties for every $\vw$ implies
$\vw^\top\boldsymbol D\vw=0$ for $\boldsymbol D=\boldsymbol C-\boldsymbol C'$.
Taking $\vw=\ve_i$ gives $D_{ii}=0$, and taking $\vw=\ve_i+\ve_j$ gives $2D_{ij}=0$ for $i\ne j$, since $\boldsymbol D$ is symmetric. Therefore $\boldsymbol C=\boldsymbol C'$, which also suffices. The quadratic-form equivalence concerns the whole coefficient space.
\end{proof}

For any function $h$ using only projected features and every $\vv\in\operatorname{range}(P)$,
$h((\mI-P)(\vz+\vv))=h((\mI-P)\vz)$ because $(\mI-P)\vv=0$.
For a linear score, the exact sensitivity identity is
\begin{equation}
\sup_{\vv\in\operatorname{range}(P),\ \|\vv\|_2\le\rho}
 |\vw^\top\vv|=\rho\|P\vw\|_2.
\label{eq:projector-sensitivity}
\end{equation}
Cauchy--Schwarz proves the upper bound. If $P\vw\ne0$, the vector
$\rho P\vw/\|P\vw\|_2$ attains it; otherwise both sides vanish. The equalities concern the selected subspace, whether or not it equals $\mathcal U$.

\paragraph{Proof of Proposition~\ref{prop:spectral-recovery}.}
\begin{proof}
The spectral theorem reduces the claim to two invariant blocks and their ordering. If the unique leading $r_*$-dimensional eigenspace is $\mathcal U$, symmetry implies that $\mathcal S$ is invariant too, so $P_{\mathcal U}\boldsymbol M P_{\mathcal S}=0$. Its $r_*$ selected eigenvalues must all exceed the remaining eigenvalues: a boundary tie would allow a different choice in the tied eigenspace. This proves the two conditions. Conversely, the zero cross block makes $\boldsymbol M$ the orthogonal direct sum of its restrictions to $\mathcal U$ and $\mathcal S$. Strict spectral ordering places exactly the $r_*$ eigenvalues from $\mathcal U$ above those from $\mathcal S$, giving the unique leading eigenspace.
\end{proof}

\begin{proposition}[Perturbations preserving a selected projector]
\label{prop:structured-perturbations}
Let $\boldsymbol M\succeq0$ have a strict gap $g=\lambda_r(\boldsymbol M)-\lambda_{r+1}(\boldsymbol M)>0$, $1\le r<d$, with leading projector $P$ and range $\mathcal V$. Let $\boldsymbol E$ be symmetric, $\boldsymbol M'=\boldsymbol M+\boldsymbol E\succeq0$, and suppose $\boldsymbol M'$ also has a strict rank-$r$ gap. Then $P_r(\boldsymbol M')=P$ if and only if
\begin{equation}
P\boldsymbol E(\mI-P)=0,\qquad
\lambda_{\min}(\boldsymbol M'|_{\mathcal V})>
\lambda_{\max}(\boldsymbol M'|_{\mathcal V^\perp}).
\label{eq:structured-region}
\end{equation}
Within the cross-block-preserving class, $\|\boldsymbol E\|_{\rm op}<g/2$ suffices; the strict radius $g/2$ cannot be increased uniformly. Without the cross-block restriction, no positive radius guarantees exact equality of the projector.
\end{proposition}
\begin{proof}
We first apply the exact recovery criterion to the selected range, then use boundary eigenvectors to locate the sharp structured threshold. Since $\boldsymbol M$ already preserves the two blocks, $\boldsymbol M'$ preserves them exactly when $P\boldsymbol E(\mI-P)=0$. Proposition~\ref{prop:spectral-recovery}, with $\mathcal V$ in place of $\mathcal U$, then proves~\eqref{eq:structured-region}. Rayleigh quotients give
\begin{align}
\lambda_{\min}(\boldsymbol M'|_{\mathcal V})
&\ge\lambda_r(\boldsymbol M)-\|\boldsymbol E\|_{\rm op},\\
\lambda_{\max}(\boldsymbol M'|_{\mathcal V^\perp})
&\le\lambda_{r+1}(\boldsymbol M)+\|\boldsymbol E\|_{\rm op}.
\end{align}
Their difference is positive when $2\|\boldsymbol E\|_{\rm op}<g$.
Let unit eigenvectors $\vu\in\mathcal V$ and $\vv\in\mathcal V^\perp$ have eigenvalues $\lambda_r$ and $\lambda_{r+1}$. The perturbation
$\boldsymbol E_t=-t\vu\vu^\top+t\vv\vv^\top$ has norm $t$ and closes their gap at $t=g/2$. The altered eigenvalues then equal $(\lambda_r+\lambda_{r+1})/2>0$, so the matrix remains positive semidefinite but loses uniqueness. For $t$ slightly above $g/2$, positive semidefiniteness persists and these eigenvalues exchange order, changing the unique selected subspace. Finally, conjugating $\boldsymbol M$ by arbitrarily small nonzero rotations in the $(\vu,\vv)$ plane preserves all eigenvalues and positive semidefiniteness, while changing the leading projector. The resulting perturbations converge to zero in norm, proving the last claim.
\end{proof}

\subsection{Conditional concentration and subspace estimation}
\label{app:concentration}
We use $\|X\|_{\psi_2}=\inf\{a>0:\mathbb E e^{X^2/a^2}\le2\}$ and
$\|X\|_{\psi_1}=\inf\{a>0:\mathbb E e^{|X|/a}\le2\}$.
All probabilities in this subsection condition on the complete signal design. Only the $k$ original noises are independent observations; symmetrizing pair orientations does not double that count.

\begin{lemma}[Pair covariance concentration]
\label{lem:pair-concentration}
Under the conditional noise requirements of Assumption~\ref{ass:pairs}, for arbitrary designed signals $\vu_j\in\mathbb R^d$ and $\eta\in(0,1)$, set $h=d\log9+\log(4/\eta)$ and $\lambda_{1,k}=\|\mSigma_k\|_{\rm op}$ for $\mSigma_k=k^{-1}\sum_j\vu_j\vu_j^\top$. There is a constant $C>0$, independent of the design, dimension, and sample size, such that with probability at least $1-\eta$,
\begin{equation}
\left\|\widetilde{\mSigma}_k-(\mSigma_k+\sigma_\xi^2\mI)\right\|_{\rm op}
\le C\left[K\sigma_\xi\sqrt{\frac{\lambda_{1,k}h}{k}}
 +K^2\sigma_\xi^2\left(\sqrt{\frac hk}+\frac hk\right)\right].
\label{eq:concentration-sharp}
\end{equation}
Consequently, the constant $C_K$ in~\eqref{eq:pair-error-scale} can be chosen to make $e_k$ a valid radius for every $k\ge1$. No bound on the largest individual $\|\vu_j\|_2$ is required.
\end{lemma}
\begin{proof}
The proof controls the signal--noise cross term by scalar sub-Gaussian concentration and the noise-covariance term by scalar Bernstein, uniformly over a finite net. This avoids imposing a maximum leverage condition on the deterministic design.

\paragraph{Step 1: scalar concentration.}
Expand the error into symmetric matrices
\begin{equation}
\boldsymbol E=
\underbrace{\frac1k\sum_j(\vu_j\vxi_j^\top+\vxi_j\vu_j^\top)}_{\boldsymbol A}
+\underbrace{\frac1k\sum_j(\vxi_j\vxi_j^\top-\sigma_\xi^2\mI)}_{\boldsymbol B}.
\label{eq:covariance-expansion}
\end{equation}
For a fixed unit vector $\vv$, the summands
$X_j=2(\vv^\top\vu_j)(\vv^\top\vxi_j)$ are independent and centered. Their sub-Gaussian norm squares satisfy
\begin{equation}
\sum_j\|X_j\|_{\psi_2}^2
\le4K^2\sigma_\xi^2\sum_j(\vv^\top\vu_j)^2
=4K^2\sigma_\xi^2k\,\vv^\top\mSigma_k\vv
\le4K^2\sigma_\xi^2k\lambda_{1,k}.
\label{eq:cross-term-variance}
\end{equation}
General Hoeffding concentration therefore gives, for any $s>0$ and a universal $C_1$,
\begin{equation}
\mathbb P\left(|\vv^\top\boldsymbol A\vv|>
 C_1K\sigma_\xi\sqrt{\lambda_{1,k}s/k}\right)\le2e^{-s}.
\label{eq:cross-scalar-tail}
\end{equation}
This is Theorem~2.6.2 of \citet[p.~29]{vershynin2018high}, applied using~\eqref{eq:cross-term-variance}. Zero-variance cases make the cross term zero and need no division by its scale.

For the second matrix, $Y_j=(\vv^\top\vxi_j)^2-\sigma_\xi^2$ are independent and centered by isotropy. Squaring a sub-Gaussian variable gives a sub-exponential variable, with
$\|X^2\|_{\psi_1}=\|X\|_{\psi_2}^2$ \citep[Lemma~2.7.6, p.~33]{vershynin2018high}. Centering changes this norm by at most a universal factor: the triangle inequality and Jensen's inequality give
$\|\mathbb EX^2\|_{\psi_1}=\mathbb EX^2/\log2\le\|X^2\|_{\psi_1}$.
Thus $\|Y_j\|_{\psi_1}\le2K^2\sigma_\xi^2$. Scalar Bernstein, Theorem~2.8.1 of \citet[p.~36]{vershynin2018high}, implies
\begin{equation}
\mathbb P\left(|\vv^\top\boldsymbol B\vv|>
 C_2K^2\sigma_\xi^2[\sqrt{s/k}+s/k]\right)\le2e^{-s}.
\label{eq:noise-scalar-tail}
\end{equation}
Independence between the coordinates of an individual noise vector was not used.

\paragraph{Step 2: uniform matrix control.}
A $1/4$-net $\mathcal N$ of the unit sphere with $|\mathcal N|\le9^d$ exists by the volumetric bound \citep[Corollary~4.2.13, p.~84]{vershynin2018high}. For every symmetric matrix $\boldsymbol H$,
\begin{equation}
\|\boldsymbol H\|_{\rm op}\le2\max_{\vv\in\mathcal N}|\vv^\top\boldsymbol H\vv|.
\label{eq:net-operator}
\end{equation}
For completeness, take a maximizing unit vector $\vx$ and $\vv\in\mathcal N$ with $\|\vx-\vv\|_2\le1/4$. Expanding the difference of quadratic forms bounds it by $2\|\boldsymbol H\|_{\rm op}\|\vx-\vv\|_2\le\|\boldsymbol H\|_{\rm op}/2$, proving~\eqref{eq:net-operator}; see also Exercise~4.4.3(b), pp.~89--90 of \citet{vershynin2018high}. Applying~\eqref{eq:cross-scalar-tail} and~\eqref{eq:noise-scalar-tail} at $s=h$ to each net point fails with probability at most $4\cdot9^de^{-h}=\eta$. On the remaining event,~\eqref{eq:net-operator} and the triangle inequality yield~\eqref{eq:concentration-sharp}.

\paragraph{Step 3: the main-text radius.}
For $d\ge1$ and $\eta\in(0,1)$, $d\log9+\log(4/\eta)\le4(d+\log(1/\eta))$. Absorbing these fixed factors and $K,K^2$ into $C_K$ gives~\eqref{eq:pair-error-scale} for every $k$. When $\sigma_\xi=0$, covariance zero and mean zero imply $\vxi_j=0$ almost surely, so both radii are zero.
\end{proof}

\begin{lemma}[Projector perturbation under the geometry of Assumption~\ref{ass:pairs}]
\label{lem:projector-perturbation}
Assume the signal geometry in Assumption~\ref{ass:pairs} $\operatorname{range}(\mSigma_k)=\mathcal U$, with $1\le r_*<d$. Let $\boldsymbol M=\mSigma_k+\sigma_\xi^2\mI$, $g=\lambda_{r_*,k}>0$, and let $\widetilde{\mSigma}_k=\boldsymbol M+\boldsymbol E$ be symmetric with $\|\boldsymbol E\|_{\rm op}\le e<g/2$. Then the leading rank-$r_*$ projector is unique and
\begin{equation}
\|P_{k,r_*}-P_{\mathcal U}\|_{\rm op}
\le\frac{e}{g-e}\le\frac{2e}{g}.
\label{eq:projector-bound}
\end{equation}
For every $r\ge r_*$ using the same ordered eigenbasis,
$\|P_{\mathcal U}(\mI-P_{k,r})\|_{\rm op}\le2e/g$.
\end{lemma}
\begin{proof}
The isotropic noise makes the stable block a scalar matrix, so projecting the perturbed eigenvalue equation onto that block directly controls leakage. Rayleigh quotients, or the min--max characterization of eigenvalues, give
\begin{equation}
\widetilde\lambda_{r_*}\ge\sigma_\xi^2+g-e,
\qquad
\widetilde\lambda_{r_*+1}\le\sigma_\xi^2+e.
\label{eq:weyl-special}
\end{equation}
The gap is at least $g-2e>0$. Let $\widehat{\boldsymbol Q}$ contain the leading $r_*$ orthonormal eigenvectors and $\widehat{\boldsymbol\Lambda}$ their eigenvalues. Since $P_{\mathcal S}\mSigma_k=0$, their eigenvalue equation implies
\begin{align}
P_{\mathcal S}(\boldsymbol M+\boldsymbol E)\widehat{\boldsymbol Q}
 &=P_{\mathcal S}\widehat{\boldsymbol Q}\widehat{\boldsymbol\Lambda},\\
P_{\mathcal S}\widehat{\boldsymbol Q}
 (\widehat{\boldsymbol\Lambda}-\sigma_\xi^2\mI)
 &=P_{\mathcal S}\boldsymbol E\widehat{\boldsymbol Q},\\
\|P_{\mathcal S}\widehat{\boldsymbol Q}\|_{\rm op}
 &\le e\| (\widehat{\boldsymbol\Lambda}-\sigma_\xi^2\mI)^{-1}\|_{\rm op}
 \le\frac{e}{g-e}.
\label{eq:eigenvector-leakage}
\end{align}
The inverse exists by~\eqref{eq:weyl-special}.

To relate leakage to projector error without an additional perturbation theorem, let $P,Q$ be equal-rank orthogonal projectors with orthonormal bases $V,W$. Both $\|(\mI-P)Q\|_{\rm op}^2$ and $\|P(\mI-Q)\|_{\rm op}^2$ equal $1-s_{\min}(V^\top W)^2$, by multiplying the compressed projectors and using the singular values of the square matrix $V^\top W$. Furthermore, $(P-Q)^2$ has diagonal blocks $P(\mI-Q)P$ and $(\mI-P)Q(\mI-P)$ relative to $P$. Its largest eigenvalue is the maximum of the two equal leakage squares. Consequently
$\|P-Q\|_{\rm op}=\|(\mI-P)Q\|_{\rm op}=\|P(\mI-Q)\|_{\rm op}$.
Applying this identity to~\eqref{eq:eigenvector-leakage} proves~\eqref{eq:projector-bound}. Finally,
$(\mI-P_{k,r_*})(\mI-P_{k,r})=\mI-P_{k,r}$ by nesting, so enlarging the removed subspace cannot increase missed-unstable leakage.
\end{proof}

\begin{proposition}[Consistency along designed collections]
\label{prop:projector-consistency}
Consider a sequence of designs satisfying Assumption~\ref{ass:pairs}, with $1\le r_*<d$, $g_k=\lambda_{r_*,k}>0$, and $\boldsymbol E_k=\widetilde{\mSigma}_k-(\mSigma_k+\sigma_{\xi,k}^2\mI)$. Then $\|P_{k,r_*}-P_{\mathcal U}\|_{\rm op}\to0$ in probability, under any fixed measurable tie rule, if either of the following holds:
\begin{enumerate}[label=(\roman*),leftmargin=*]
\item $\|\boldsymbol E_k\|_{\rm op}/g_k\to0$ in probability;
\item $K$ is uniformly bounded and, for some $\eta_k\downarrow0$, the weakest-signal factor $T_k(\eta_k)$ in~\eqref{eq:relative-pair-scale}, with $g=g_k$ and $\sigma=\sigma_{\xi,k}$, tends to zero.
\end{enumerate}
A sufficient condition for (i) is that the radius in~\eqref{eq:concentration-sharp}, evaluated at $\eta_k\downarrow0$ and divided by $g_k$, tends to zero, with $K$ uniformly bounded.
\end{proposition}
\begin{proof}
Both conditions give a vanishing projector bound on an event whose probability tends to one. Under (i), with probability tending to one, $\|\boldsymbol E_k\|_{\rm op}<g_k/2$; Lemma~\ref{lem:projector-perturbation} bounds projector error by twice this vanishing ratio. Lemma~\ref{lem:pair-concentration} makes the stated covariance-radius condition sufficient. Under (ii), Theorem~\ref{thm:pair-subspace-rate} bounds projector error by $C_KT_k(\eta_k)\to0$ with probability at least $1-\eta_k\to1$. Either argument proves convergence in probability. The matrices need not converge to one common covariance, and the designs need not be nested.
\end{proof}
Condition (ii) avoids dependence on the largest clean eigenvalue and does not require controlling the full covariance error relative to the gap. Isotropy of the mean alone does not control finite-pair rotation. More generally, mean-zero anisotropic noise has expected covariance $\mSigma_k+k^{-1}\sum_j\operatorname{Cov}(\vxi_j)$. Proposition~\ref{prop:spectral-recovery} can still apply when this matrix preserves the blocks and their ordering. Arbitrary anisotropy or bias need not do so, even with infinitely many pairs.

\subsection{Shared proxy bound and the GRIT guarantee}
\label{app:grit-proof}
The next result shows how source prediction, coverage, and the cost of a constraint combine before choosing a particular pair estimator. Define $\mSigma_{z,p}=\mathbb E_p[\rvz\rvz^\top]$ and $\|\vv\|_{\boldsymbol H}=\sqrt{\vv^\top\boldsymbol H\vv}$ for a positive semidefinite matrix $\boldsymbol H$.

\begin{lemma}[A deterministic rank--budget bound]
\label{lem:proxy-bound}
Let $\boldsymbol C=\boldsymbol Q\boldsymbol\Lambda\boldsymbol Q^\top$ be a rank-$r$ positive semidefinite matrix, $1\le r\le d$, where $\boldsymbol Q$ has orthonormal columns and $\boldsymbol\Lambda\succ0$. Set $P=\boldsymbol Q\boldsymbol Q^\top$, $P_0=\mI-P$,
$\alpha=\|P_{\mathcal U}P_0\|_{\rm op}$, and
$\kappa=\|P_{\mathcal U}\boldsymbol Q\boldsymbol\Lambda^{-1/2}\|_{\rm op}$.
For $\tau\ge0$, every feasible $\vw$ satisfies
\begin{equation}
\|P_{\mathcal U}\vw\|_2\le\alpha\|P_0\vw\|_2+\kappa\tau,
\qquad \|\vw\|_{\boldsymbol C}\le\tau.
\label{eq:proxy-sensitivity}
\end{equation}
If $\vw_\tau$ minimizes $R_p$ over this feasible set and $\vw_{\mathcal S}^*$ is a stable source minimizer, put $b=\|\vw_{\mathcal S}^*\|_{\boldsymbol C}$ and $s=(b-\tau)_+/b$, defining $s=0$ when $b=0$. Then, under Theorem~\ref{thm:risk-decomposition}'s assumptions,
\begin{equation}
R_q(\vw_\tau)\le R_{p,\mathcal S}^*+C_{\rm br}(\vw_\tau)
+s\|P\vw_{\mathcal S}^*\|_{\mSigma_{z,p}}
+\rho_\Delta(\alpha\|P_0\vw_\tau\|_2+\kappa\tau).
\label{eq:proxy-risk-bound}
\end{equation}
\end{lemma}
\begin{proof}
Decompose the learned coefficient into constrained and unconstrained components. Then construct a feasible stable-oracle approximation by shrinking only the constrained component. The first step is
\begin{align}
P_{\mathcal U}\vw
 &=P_{\mathcal U}P_0\vw+
 P_{\mathcal U}\boldsymbol Q\boldsymbol\Lambda^{-1/2}
 (\boldsymbol\Lambda^{1/2}\boldsymbol Q^\top\vw),\\
\|P_{\mathcal U}\vw\|_2
 &\le\alpha\|P_0\vw\|_2+
 \kappa\|\boldsymbol\Lambda^{1/2}\boldsymbol Q^\top\vw\|_2
 \le\alpha\|P_0\vw\|_2+\kappa\tau.
\end{align}
For the source comparison, let $\vv=\vw_{\mathcal S}^*-sP\vw_{\mathcal S}^*$. Its proxy norm is $(1-s)b\le\tau$, including $b=0$. The unit logistic derivative bound gives
\begin{align}
R_p(\vw_\tau)&\le R_p(\vv),\label{eq:proxy-comparison-1}\\
R_p(\vv)-R_p(\vw_{\mathcal S}^*)
 &\le\mathbb E_p|s\,\rvz^\top P\vw_{\mathcal S}^*|,\\
 &\le s\|P\vw_{\mathcal S}^*\|_{\mSigma_{z,p}}.
\label{eq:proxy-comparison-3}
\end{align}
Equation~\eqref{eq:proxy-comparison-1} is source optimality;~\eqref{eq:proxy-comparison-3} is Cauchy--Schwarz. Combine these inequalities with~\eqref{eq:proxy-sensitivity} and Theorem~\ref{thm:risk-decomposition}. The result is deterministic once the proxy is fixed.
\end{proof}
A quadratic-penalty minimizer is also a source-risk minimizer in the budget set whose radius is its own proxy norm: any point in that set with smaller risk would have smaller penalized objective.

\paragraph{Coverage over an unbounded feasible class.}
For $\rho>0$, $\sup_{\|\vw\|_{\boldsymbol C}\le\tau}\rho\|P_{\mathcal U}\vw\|_2$ is finite if and only if $\mathcal U\subseteq\operatorname{range}(\boldsymbol C)$. If coverage holds, $P_{\mathcal U}P_0=0$ and~\eqref{eq:proxy-sensitivity} gives finiteness. If it fails, some $\vv\in\ker\boldsymbol C$ has $P_{\mathcal U}\vv\ne0$; every multiple of $\vv$ is feasible and has unbounded sensitivity. This concerns certification over the whole unbounded feasible class. A particular stable predictor has zero direct sensitivity even without coverage; a norm-bounded feasible class has finite sensitivity regardless of coverage.

\paragraph{Proof of Theorem~\ref{thm:hard-grit}.}
\begin{proof}
Project a stable oracle into the GRIT-feasible space to control source risk, then use estimated-subspace leakage to control direct unstable shift. Work on the event of Theorem~\ref{thm:pair-subspace-rate}, common to all ranks and coefficients. Put $\widehat P=P_{k,r_*}$, $P=P_{k,r}$, and $\alpha=\|\widehat P-P_{\mathcal U}\|_{\rm op}$. Because $r\ge r_*$, $P-\widehat P$ is an orthogonal projector. The feasible comparator $(\mI-P)\vw_{\mathcal S}^*$ gives
\begin{align}
R_p(\vw_r)&\le R_p((\mI-P)\vw_{\mathcal S}^*),\label{eq:hard-proof-1}\\
&\le R_{p,\mathcal S}^*+\|P\vw_{\mathcal S}^*\|_{\mSigma_{z,p}},\label{eq:hard-proof-2}\\
&\le R_{p,\mathcal S}^*+L_z\|(P-\widehat P)\vw_{\mathcal S}^*\|_2
 +L_z\|\widehat P\vw_{\mathcal S}^*\|_2,\label{eq:hard-proof-3}\\
&\le R_{p,\mathcal S}^*+L_z\|(P-\widehat P)\vw_{\mathcal S}^*\|_2
 +\alpha L_z\|\vw_{\mathcal S}^*\|_2.
\label{eq:hard-proof-4}
\end{align}
Equation~\eqref{eq:hard-proof-2} follows from the same Lipschitz comparison used in Lemma~\ref{lem:proxy-bound};~\eqref{eq:hard-proof-3} uses the source second moment and the triangle inequality. For~\eqref{eq:hard-proof-4}, use $P_{\mathcal U}\vw_{\mathcal S}^*=0$ and the definition of $\alpha$. Feasibility and $\operatorname{range}(\widehat P)\subseteq\operatorname{range}(P)$ also give
\begin{equation}
\|P_{\mathcal U}\vw_r\|_2
=\|P_{\mathcal U}(\mI-P)\vw_r\|_2
\le\|P_{\mathcal U}(\mI-\widehat P)\|_{\rm op}\|\vw_r\|_2\le\alpha\|\vw_r\|_2.
\label{eq:hard-proof-sensitivity}
\end{equation}
Substitute~\eqref{eq:hard-proof-4} and~\eqref{eq:hard-proof-sensitivity} into Theorem~\ref{thm:risk-decomposition}, and bound $\alpha$ by Theorem~\ref{thm:pair-subspace-rate}. Every comparison was pointwise on the same matrix event, proving the simultaneous rank and oracle quantifiers. At $r=d$ the only feasible coefficient is zero and the argument remains valid.
\end{proof}

The same proof permits the clipped radius $\min\{1,C_K T_k\}$ from Theorem~\ref{thm:pair-subspace-rate} in place of $C_KT_k$, or the sharper block radii given there. Stable and constrained minimizers need not be unique; taking the infimum of the combined oracle-dependent terms is valid. An approximate constrained population optimum adds its objective gap to the source comparison. Appendix~\ref{app:grit-tightness} refines that comparison using source stationarity and logistic curvature.

\paragraph{Fixed-energy interpretation.}
At fixed $T=\operatorname{tr}(\mSigma_k)>0$ and $r_*$, the weakest signal obeys $g_k\le T/r_*$. Equality holds exactly when all positive clean eigenvalues equal $T/r_*$. Every noise-to-signal term in~\eqref{eq:grit-main-bound} decreases as $g_k$ increases, so this spectrum minimizes the displayed pair factor at fixed energy.

\subsection{Underselection and actual target-risk failure}
\label{app:underselection}
\paragraph{Proof of Proposition~\ref{prop:underselection}.}
\begin{proof}
Rank--nullity supplies an available unstable direction. A label-dependent translation along the predictor's unstable component then makes its signed margin tend to minus infinity while preserving the stable-label joint distribution.
For the restriction $P_{k,r}|_{\mathcal U}:\mathcal U\to\widehat{\mathcal U}_r$, its rank is at most $r$, so
\begin{equation}
\dim\ker(P_{k,r}|_{\mathcal U})
=r_*-\operatorname{rank}(P_{k,r}|_{\mathcal U})\ge r_*-r.
\end{equation}
Its kernel is $\mathcal U\cap\widehat{\mathcal U}_r^\perp$, proving the first part. For a fixed feasible $\vw$ with $c=\|P_{\mathcal U}\vw\|_2>0$, let $\vu=P_{\mathcal U}\vw/c$ and define $q_t$ as the pushforward of $p$ through
\begin{equation}
(\vz,y)\longmapsto(\vz-t y\vu,y),\qquad t\ge0.
\label{eq:adverse-translation}
\end{equation}
The stable projection and label are unchanged because $\vu\in\mathcal U$, hence Assumption~\ref{ass:stability} is preserved and $\breve q_t=q_t$ for the corresponding conditional version. The signed margin satisfies
\begin{equation}
\ry\vw^\top(\rvz-t\ry\vu)=\ry\vw^\top\rvz-tc.
\label{eq:underselection-margin}
\end{equation}
The initial margin is finite almost surely. Therefore its probability of being strictly negative tends to one by bounded convergence of the corresponding indicators. The logistic loss in~\eqref{eq:underselection-margin} increases pointwise to infinity; monotone convergence gives $R_{q_t}(\vw)\to\infty$. Each finite-$t$ target has finite second feature moments. The explicit coupling has displacement $-t\ry\vu$ and second moment $t^2\vu\vu^\top\preceq t^2P_{\mathcal U}$, so a radius $\rho_\Delta=t$ suffices. The bridge term is zero for every $t$.
\end{proof}
The proposition constructs feature distributions. Realizing these translations through a restricted pretrained feature map is a separate condition; identity features suffice. Moreover, underselection alone does not force the source optimum to use an unstable component. The following construction shows that it can do so even when pair geometry is exact and the direct radius stays fixed.

\begin{proposition}[A source optimizer can exploit an omitted direction]
\label{prop:underselection-optimizer}
For every $0\le r<m$ and $\epsilon\in(0,1/4)$, there exist distributions with $\dim\mathcal U=m$ satisfying Assumption~\ref{ass:stability} and arbitrarily large clean designed pair collections satisfying Assumption~\ref{ass:pairs} such that rank-$r$ GRIT has a unique population source logistic minimizer with source error $\epsilon$ and target error $1-\epsilon$. The stable optimum has error $1/4$ on both distributions. The bridge term is zero and direct radius two suffices. As $\epsilon\downarrow0$, the GRIT source logistic risk tends to zero while its target risk tends to infinity.
\end{proposition}
\begin{proof}
We place one source-predictive unstable coordinate below the selected spectral cutoff, and reverse its association with the label at target. Take features
$\rvz=(X,Z_1,\ldots,Z_{m-1},T)\in\mathbb R^{m+1}$, stable space $\operatorname{span}(\ve_X)$, and unstable complement. Let $X$ be a fair sign and
$\mathbb P(\ry=y\mid X=x)=(1+e^{-ybx})^{-1}$ with $b=\log3$.
Let $V\in\{-1,1\}$ be an independent sign with $\mathbb P(V=1)=1-\epsilon$, and let all $Z_i$ be independent fair signs, independent of $(X,\ry,V)$. Source features use $T=\ry V$; target features use $T=-\ry V$. The stable-label distribution is identical.

Put $a=\log((1-\epsilon)/\epsilon)>b$. Bayes' rule gives the source conditional log odds $bX+aT$, so the coefficient
$\vw_p=b\ve_X+a\ve_T$ realizes the Bayes logit. Source feature support is the complete sign cube, with positive probability at each point and label probabilities strictly between zero and one. For every finite coefficient and every nonzero direction $\vv$, the logistic Hessian quadratic form is
\begin{equation}
\mathbb E_p\left[
 \frac{e^{\vw^\top\rvz}}{(1+e^{\vw^\top\rvz})^2}
 (\vv^\top\rvz)^2\right]>0.
\end{equation}
The strict inequality follows because the cube spans the feature space. Hence $\vw_p$ is the unique source minimizer.

Choose distinct positive numbers $\mu_1>\cdots>\mu_m$, with the $T$ coordinate assigned $\mu_m$. At $k=Nm$, repeat each clean signal $\sqrt{m\mu_i}\ve_i\in\mathcal U$ exactly $N$ times and set all noise to zero. The pair covariance is diagonal with these unstable eigenvalues and a zero stable eigenvalue, for every $N$. Selecting its leading $r$ eigenvectors removes $r$ of the $Z_i$ and leaves $T$. Thus $\vw_p$ remains feasible and is the unique constrained source minimizer. Identity features and pairs with identical stable coordinates realize these designed differences.

Because $a>b$, the sign of $bX+aT$ equals $T$ everywhere. Source and target errors are consequently $\epsilon$ and $1-\epsilon$. The stable minimizer is $b\ve_X$, with error $1/4$ in both domains. To verify the surrogate-risk claim directly, evaluating the source risk at $a\ve_T$ gives
\begin{equation}
R_p(\vw_p)\le(1-\epsilon)\log(1+e^{-a})+
\epsilon\log(1+e^a)
=-\epsilon\log\epsilon-(1-\epsilon)\log(1-\epsilon)\longrightarrow0.
\end{equation}
At target, the event $T=-\ry$ has probability $1-\epsilon$ and margin at most $b-a$, giving
\begin{equation}
R_q(\vw_p)\ge(1-\epsilon)\log(1+e^{a-b})
\ge(1-\epsilon)(a-b)\longrightarrow\infty.
\end{equation}
Coupling the domains with the same $(X,\ry,V,Z_i)$ flips only $T$; the displacement second moment is $4\ve_T\ve_T^\top\preceq4P_{\mathcal U}$. The bridge equals the target. At each fixed $\epsilon$, moments and coefficients are finite, but $\|\vw_p\|_2$ grows as $\epsilon\downarrow0$; no fixed-norm claim is made.
\end{proof}

\begin{samepage}
\subsection{Quadratic consistency: main proof and an objective comparison}
\label{app:cr-proof}
\paragraph{Proof of Theorem~\ref{thm:quadratic-cr}.}
\begin{proof}
The proof combines covariance concentration with a deterministic bound that uses the penalized optimum's first-order condition. Lemma~\ref{lem:pair-concentration} gives
$\widetilde{\mSigma}_k=\mSigma_k+\sigma_\xi^2\mI+\boldsymbol E$ with $\|\boldsymbol E\|_{\rm op}\le e_k$ on one event of conditional probability at least $1-\eta$. Proposition~\ref{prop:cr-stationarity-bound} applies on that event with $\vv=\vw_{\mathcal S}^*$ and $g=g_k$, proving~\eqref{eq:cr-main-bound}. The event is common to every regularization strength, optimizer, and stable oracle.
\end{proof}
\end{samepage}

\paragraph{An alternative bound from objective values.}
For comparison, an objective-value argument gives a second valid bound. It also accommodates an additive optimization gap and supplies the MSE argument below. On the covariance event, abbreviate $\vw=\vw_\lambda$ and $g=g_k$.
Source optimality yields
\begin{align}
R_p(\vw)
&\le R_{p,\mathcal S}^*+
 \lambda\vw_{\mathcal S}^{*\top}\widetilde{\mSigma}_k\vw_{\mathcal S}^*
 -\lambda\vw^\top\widetilde{\mSigma}_k\vw,\label{eq:cr-proof-1}\\
&\le R_{p,\mathcal S}^*+\lambda\sigma_\xi^2\|\vw_{\mathcal S}^*\|_2^2
 -\lambda\vw^\top\mSigma_k\vw-\lambda\sigma_\xi^2\|\vw\|_2^2
 +\lambda e_k(\|\vw_{\mathcal S}^*\|_2^2+\|\vw\|_2^2),\label{eq:cr-proof-2}\\
&\le R_{p,\mathcal S}^*+\lambda\sigma_\xi^2\|\vw_{\mathcal S}^*\|_2^2
 -\lambda(g+\sigma_\xi^2)\|P_{\mathcal U}\vw\|_2^2
 -\lambda\sigma_\xi^2\|P_{\mathcal S}\vw\|_2^2
 +\lambda e_k(\|\vw_{\mathcal S}^*\|_2^2+\|\vw\|_2^2).
\label{eq:cr-proof-3}
\end{align}
Equation~\eqref{eq:cr-proof-2} uses $\mSigma_k\vw_{\mathcal S}^*=0$ and
$|\vv^\top\boldsymbol E\vv|\le e_k\|\vv\|_2^2$.
Equation~\eqref{eq:cr-proof-3} uses $\mSigma_k\succeq gP_{\mathcal U}$ and orthogonality of $\mathcal U$ and $\mathcal S$.
Add the bridge term and $\rho_\Delta t$ from Theorem~\ref{thm:risk-decomposition}, where $t=\|P_{\mathcal U}\vw\|_2$. The elementary identity
\begin{equation}
\rho_\Delta t-\lambda(g+\sigma_\xi^2)t^2
=\frac{\rho_\Delta^2}{4\lambda(g+\sigma_\xi^2)}
-\lambda(g+\sigma_\xi^2)
\left(t-\frac{\rho_\Delta}{2\lambda(g+\sigma_\xi^2)}\right)^2
\label{eq:cr-square}
\end{equation}
bounds their sum by the first term on the right. Dropping the nonpositive stable penalty gives the alternative bound
\begin{equation}
\begin{aligned}
R_q(\vw)\le{}&R_{p,\mathcal S}^*+C_{\rm br}(\vw)
 +\lambda\sigma_\xi^2\|\vw_{\mathcal S}^*\|_2^2
 +\frac{\rho_\Delta^2}{4\lambda(g+\sigma_\xi^2)}\\
 &+\lambda e_k(\|\vw_{\mathcal S}^*\|_2^2+\|\vw\|_2^2).
\end{aligned}
\label{eq:cr-objective-bound}
\end{equation}
Both~\eqref{eq:cr-main-bound} and~\eqref{eq:cr-objective-bound} are valid on this event, so their minimum is valid too. The stationarity bound halves the two population-matrix coefficients, but its covariance remainder need not be smaller for every oracle and fitted coefficient. An approximate penalized minimizer adds its objective gap to~\eqref{eq:cr-objective-bound}; objective closeness alone does not provide the exact stationarity used by the stronger bound.

The nonnegative terms in either bound can exceed the actual risk increase. Appendix~\ref{app:cr-optimizer-refined} retains the signed identity, gives a quadratic refinement under segment curvature, and distinguishes fixed-source small-regularization behavior from worst-case behavior over a changing source class.

\subsection{Mean squared error regression}
\label{app:mse}
For real labels with finite second moments under $p,q,\breve q$, define
$R_p^{\rm MSE}(\vw)=\mathbb E_p(\vw^\top\rvz-\ry)^2$ and
$R_p^{\rm RMSE}(\vw)=\sqrt{R_p^{\rm MSE}(\vw)}$.
Use the same common conditional label distribution in Assumption~\ref{ass:stability} and the same conditional couplings, now for real labels. The exact bridge quantity is
$C_{\rm br}^{\rm RMSE}(\vw)=|R_q^{\rm RMSE}(\vw)-R_{\breve q}^{\rm RMSE}(\vw)|$.
The source moment constant remains $L_z=\|\mathbb E_p[\rvz\rvz^\top]\|_{\rm op}^{1/2}$; no global Lipschitz bound on squared loss is assumed.

\begin{theorem}[Regression transport and GRIT]
\label{thm:mse-grit}
Suppose a conditional coupling satisfies
$\mathbb E[\Delta\Delta^\top]\preceq\rho_\Delta^2P_{\mathcal U}$ and the stated second moments are finite. Then
\begin{equation}
R_q^{\rm RMSE}(\vw)\le R_p^{\rm RMSE}(\vw)
+C_{\rm br}^{\rm RMSE}(\vw)+\rho_\Delta\|P_{\mathcal U}\vw\|_2.
\label{eq:mse-decomposition}
\end{equation}
Let $\vw_{\mathcal S}^*$ minimize source MSE over $\mathcal S$. Under Assumption~\ref{ass:pairs}, on the event of Theorem~\ref{thm:pair-subspace-rate}, with its radius $\gamma_k$, every source-MSE minimizer $\vw_r$ subject to $P_{k,r}\vw=0$, $r\ge r_*$, satisfies
\begin{equation}
\begin{aligned}
R_q^{\rm RMSE}(\vw_r)\le{}&R_p^{\rm RMSE}(\vw_{\mathcal S}^*)
 +C_{\rm br}^{\rm RMSE}(\vw_r)
 +L_z\|(P_{k,r}-P_{k,r_*})\vw_{\mathcal S}^*\|_2\\
 &+\gamma_k
 (L_z\|\vw_{\mathcal S}^*\|_2+\rho_\Delta\|\vw_r\|_2).
\end{aligned}
\label{eq:mse-hard-bound}
\end{equation}
Squaring the nonnegative right side gives an MSE bound.
\end{theorem}
\begin{proof}
Apply the triangle inequality in $L^2$ to coupled prediction residuals, then use the same feasible-oracle geometry as for logistic risk. Under a conditional coupling, labels are shared, so
\begin{align}
R_{\breve q}^{\rm RMSE}(\vw)
 &=\|\vw^\top\rvz-\ry+\vw^\top\Delta\|_{L^2},\\
 &\le\|\vw^\top\rvz-\ry\|_{L^2}+\|\vw^\top\Delta\|_{L^2},\\
 &\le R_p^{\rm RMSE}(\vw)+\rho_\Delta\|P_{\mathcal U}\vw\|_2.
\end{align}
Adding the exact RMSE bridge difference proves~\eqref{eq:mse-decomposition}. The source minimizer of MSE also minimizes RMSE, because square root is increasing. With $P=P_{k,r}$, source optimality and the $L^2$ triangle inequality give
\begin{equation}
R_p^{\rm RMSE}(\vw_r)\le
R_p^{\rm RMSE}((\mI-P)\vw_{\mathcal S}^*)
\le R_p^{\rm RMSE}(\vw_{\mathcal S}^*)+
\|P\vw_{\mathcal S}^*\|_{\mSigma_{z,p}}.
\end{equation}
The decomposition of the final norm into extra-direction removal and rank-$r_*$ leakage is exactly~\eqref{eq:hard-proof-3}--\eqref{eq:hard-proof-4}, and~\eqref{eq:hard-proof-sensitivity} controls unstable sensitivity. Substitution proves~\eqref{eq:mse-hard-bound}. All terms on its right are nonnegative, justifying squaring.
\end{proof}

Under a bridge coupling $\gamma$ for real labels, the reverse $L^2$ triangle inequality also gives
\begin{equation}
C_{\rm br}^{\rm RMSE}(\vw)
\le\left\|\vw^\top(\rvz^q-\rvz^{\rm br})-(\ry^q-\ry^{\rm br})\right\|_{L^2(\gamma)}.
\label{eq:mse-bridge-transport}
\end{equation}
If feature displacement second moment is at most $\rho_z^2\mI$ and label displacement second moment at most $\rho_y^2$, its right side is at most $\rho_z\|\vw\|_2+\rho_y$. The deterministic rank--budget result also extends to RMSE with the same $s\|P\vw_{\mathcal S}^*\|_{\mSigma_{z,p}}$ term, by the source comparison just proved.

\begin{theorem}[Quadratic consistency for MSE]
\label{thm:mse-cr}
Under Assumption~\ref{ass:pairs} and the regression moment and conditional-shift conditions above, let
$\vw_\lambda\in\arg\min_{\vw}\{R_p^{\rm MSE}(\vw)+\lambda\vw^\top\widetilde{\mSigma}_k\vw\}$, $\lambda>0$, and let $\vw_{\mathcal S}^*$ be a stable source-MSE minimizer. On the covariance event $\|\widetilde{\mSigma}_k-(\mSigma_k+\sigma_\xi^2\mI)\|_{\rm op}\le e$, put
\begin{equation}
A=R_p^{\rm MSE}(\vw_{\mathcal S}^*)+
\lambda\sigma_\xi^2\|\vw_{\mathcal S}^*\|_2^2+
\lambda e(\|\vw_{\mathcal S}^*\|_2^2+\|\vw_\lambda\|_2^2).
\end{equation}
Then
\begin{equation}
R_q^{\rm MSE}(\vw_\lambda)
\le\left[
C_{\rm br}^{\rm RMSE}(\vw_\lambda)+
\sqrt{\left(1+\frac{\rho_\Delta^2}{\lambda(\lambda_{r_*,k}+\sigma_\xi^2)}\right)A}
\right]^2.
\label{eq:mse-cr}
\end{equation}
In particular, $e=e_k$ from Lemma~\ref{lem:pair-concentration} gives the finite-pair guarantee without a small-gap requirement.
\end{theorem}
\begin{proof}
Source optimality provides a quadratic budget for the source residual and unstable coefficient. A two-dimensional Cauchy--Schwarz inequality then combines that budget with RMSE transport. Let $a=\lambda(\lambda_{r_*,k}+\sigma_\xi^2)>0$ and $t=\|P_{\mathcal U}\vw_\lambda\|_2$. Repeating~\eqref{eq:cr-proof-1}--\eqref{eq:cr-proof-3} for the MSE objective and dropping the nonpositive stable penalty gives
\begin{equation}
R_p^{\rm MSE}(\vw_\lambda)+a t^2\le A.
\label{eq:mse-cr-budget}
\end{equation}
Therefore, by~\eqref{eq:mse-decomposition},
\begin{align}
R_q^{\rm RMSE}(\vw_\lambda)
 &\le C_{\rm br}^{\rm RMSE}(\vw_\lambda)
   +\sqrt{R_p^{\rm MSE}(\vw_\lambda)}+\rho_\Delta t,\\
 &\le C_{\rm br}^{\rm RMSE}(\vw_\lambda)
   +\sqrt{1+\rho_\Delta^2/a}\,
     \sqrt{R_p^{\rm MSE}(\vw_\lambda)+a t^2},\\
 &\le C_{\rm br}^{\rm RMSE}(\vw_\lambda)+\sqrt{(1+\rho_\Delta^2/a)A}.
\end{align}
The middle line is Cauchy--Schwarz for vectors
$(\sqrt{R_p^{\rm MSE}(\vw_\lambda)},\sqrt a\,t)$ and
$(1,\rho_\Delta/\sqrt a)$; the last uses~\eqref{eq:mse-cr-budget}.
Squaring proves~\eqref{eq:mse-cr}.
\end{proof}
For completeness, these finite-dimensional population MSE minimizers exist under the stated second moments. Each objective is a nonnegative quadratic form minus a linear term plus a constant. Its linear coefficient is orthogonal to every null vector of the quadratic form: a null direction has zero source score almost surely, hence zero covariance with the square-integrable label. Restricting to the range gives a positive definite quadratic with an attained minimum; null directions leave its value unchanged. This also applies within the linear subspaces defining stable and hard-GRIT predictors.

The underselection obstruction persists for regression. For a fixed coefficient with $c=\|P_{\mathcal U}\vw\|_2>0$, translate source features by $tP_{\mathcal U}\vw/c$ while holding the label fixed. The target residual becomes $\vw^\top\rvz-\ry+tc$, so
$R_{q_t}^{\rm MSE}(\vw)=R_p^{\rm MSE}(\vw)+2tc\,\mathbb E_p(\vw^\top\rvz-\ry)+t^2c^2\to\infty$.
The stable-label joint distribution stays fixed and each finite-$t$ target has finite moments. As in classification, a restriction on the feature-map image or deployment radius can restrict these constructed targets.
 \section{Subspace Estimation and Statistical Optimality}
\label{app:spectral-refinements}
\subsection{Subspace estimation from relative covariance blocks}
The pair covariance contains fluctuations within each subspace and across the two subspaces.
Controlling these blocks relative to the clean signal gives a subspace bound that depends on the weakest clean eigenvalue.
Throughout this subsection, let $m=r_*$, $s=d-m$, $\sigma=\sigma_\xi$, and $g=\lambda_{r_*,k}>0$.
Restrictions and inverse square roots of $\mSigma_k$ act on $\mathcal U$ only.
All probabilities condition on the complete clean signal collection.

\begin{lemma}[Rotation from relative blocks]
\label{lem:relative-block-rotation}
Let $\boldsymbol H$ be symmetric, with blocks
$\left(\begin{smallmatrix}\boldsymbol A&\boldsymbol B\\\boldsymbol B^\top&\boldsymbol D\end{smallmatrix}\right)$
relative to $\mathcal U\oplus\mathcal S$, where $\dim\mathcal U=m$.
If $\delta=\lambda_{\min}(\boldsymbol A)-\lambda_{\max}(\boldsymbol D)>0$, its leading rank-$m$ projector $Q$ is unique and
\begin{equation}
\|Q-P_{\mathcal U}\|_{\rm op}
\le \sin\!\left[\frac12\arctan\!\left(\frac{2\|\boldsymbol B\|_{\rm op}}{\delta}\right)\right].
\label{eq:ordered-block-angle}
\end{equation}
In particular, suppose $\boldsymbol H=\mSigma_k+\sigma^2\mI+\boldsymbol E$ and define
\begin{equation}
\begin{aligned}
\varepsilon&=\|\mSigma_{k,\mathcal U}^{-1/2}\boldsymbol E_{UU}
                        \mSigma_{k,\mathcal U}^{-1/2}\|_{\rm op},&
d_S&=\|\boldsymbol E_{SS}\|_{\rm op},\\
b&=\|\mSigma_{k,\mathcal U}^{-1/2}\boldsymbol E_{US}\|_{\rm op},&
\zeta&=1-\varepsilon-d_S/g,
\end{aligned}
\label{eq:relative-block-quantities}
\end{equation}
where $\mSigma_{k,\mathcal U}=\mSigma_k|_{\mathcal U}$.
If $\zeta>0$, then
\begin{equation}
\|Q-P_{\mathcal U}\|_{\rm op}
\le\sin\!\left[\frac12\arctan\!\left(\frac{2b}{\sqrt g\,\zeta}\right)\right].
\label{eq:relative-block-angle}
\end{equation}
\end{lemma}
\begin{proof}
The proof compresses the eigenvector equation to a plane attaining the largest principal angle.
Diagonal block ordering keeps that angle below $\pi/4$; retaining its directional signal energy gives the relative bound.

Min--max on $\mathcal U$ and $\mathcal S$ gives
$\lambda_m(\boldsymbol H)\ge\lambda_{\min}(\boldsymbol A)$ and
$\lambda_{m+1}(\boldsymbol H)\le\lambda_{\max}(\boldsymbol D)$.
The rank-$m$ gap therefore remains positive along the continuous path replacing $\boldsymbol B$ by $t\boldsymbol B$, $0\le t\le1$.
Its leading projector and largest principal angle $\theta$ vary continuously from $P_{\mathcal U}$ and zero.
For $0<\theta<\pi/2$, take unit $\vu\in\mathcal U$ and $\vv\in\mathcal S$ with
\begin{equation}
\vx=(\cos\theta)\vu+(\sin\theta)\vv\in\operatorname{range}(Q),\qquad
\vy=-(\sin\theta)\vu+(\cos\theta)\vv\in\ker Q.
\label{eq:principal-plane-pair}
\end{equation}
Such a pair follows by maximizing $\|P_{\mathcal S}\vx\|^2$ over unit $\vx\in\operatorname{range}(Q)$.
The compression equation $QP_{\mathcal S}\vx=(\sin^2\theta)\vx$ then gives $Q\vu=(\cos\theta)\vx$ and $Q\vv=(\sin\theta)\vx$.
Invariance of the two spectral spaces yields
\begin{equation}
(\vu^\top\boldsymbol A\vu-\vv^\top\boldsymbol D\vv)
 \sin\theta\cos\theta
=\cos(2\theta)\,\vu^\top\boldsymbol B\vv.
\label{eq:principal-plane-identity}
\end{equation}
The left side is positive at $\theta=\pi/4$, while the right side is zero.
Continuity along the path thus gives $\theta<\pi/4$.
Dividing~\eqref{eq:principal-plane-identity} proves
$\tan(2\theta)\le2\|\boldsymbol B\|_{\rm op}/\delta$.
Let $V,W$ be orthonormal bases for $\mathcal U$ and $\operatorname{range}(Q)$.
The equal-rank projector identity gives
\[
\begin{aligned}
\|Q-P_{\mathcal U}\|_{\rm op}
&=\|P_{\mathcal U}(\mI-Q)\|_{\rm op}
=\|(\mI-P_{\mathcal U})Q\|_{\rm op}\\
&=\sqrt{1-s_{\min}(V^\top W)^2}=\sin\theta.
\end{aligned}
\]
Each leakage norm squared equals $1-s_{\min}(V^\top W)^2$. Relative to $\mathcal U\oplus\mathcal S$, $(Q-P_{\mathcal U})^2$ is block diagonal with blocks $P_{\mathcal U}(\mI-Q)P_{\mathcal U}$ and $(\mI-P_{\mathcal U})Q(\mI-P_{\mathcal U})$, whose norms are those leakage squares.
This proves~\eqref{eq:ordered-block-angle}, including $\theta=0$.

For the relative statement, put $t=\vu^\top\mSigma_k\vu\ge g$.
The definitions in~\eqref{eq:relative-block-quantities} imply
\begin{align}
\vu^\top\boldsymbol A\vu-\vv^\top\boldsymbol D\vv
 &\ge(1-\varepsilon)t-d_S>0,\label{eq:relative-direction-gap}\\
|\vu^\top\boldsymbol B\vv|&\le b\sqrt t,\label{eq:relative-direction-cross}\\
\tan(2\theta)&\le\frac{2b\sqrt t}{(1-\varepsilon)t-d_S}
 \le\frac{2b}{\sqrt g\,\zeta}.\label{eq:relative-tangent}
\end{align}
The last inequality uses that $\sqrt t/[(1-\varepsilon)t-d_S]$ decreases on $t\ge g$ when $\varepsilon<1$ and $d_S\ge0$.
Taking the half-arctangent and sine proves~\eqref{eq:relative-block-angle}.
\end{proof}

The ordered-block bound is exact for a two-dimensional matrix with diagonal gap $\delta$ and off-diagonal entry $\|\boldsymbol B\|_{\rm op}$.
Adding a scalar multiple of the identity makes this matrix positive semidefinite without changing its eigenspaces.
In particular, a zero cross block gives exact recovery whenever the diagonal blocks remain ordered.

\begin{lemma}[Relative block concentration]
\label{lem:relative-pair-concentration}
Under Assumption~\ref{ass:pairs}, define
\begin{equation}
T_k=\left(\frac{\sigma}{\sqrt g}+\frac{\sigma^2}{g}\right)
       \sqrt{\frac{d+\log(1/\eta)}{k}}
       +\frac{\sigma^2}{g}\frac{d+\log(1/\eta)}{k},
\qquad \eta\in(0,1).
\label{eq:relative-pair-scale}
\end{equation}
Let $\boldsymbol E=\widetilde{\mSigma}_k-(\mSigma_k+\sigma^2\mI)$ and define its block quantities by~\eqref{eq:relative-block-quantities}.
There is a constant $A_K$, depending only on $K$, such that with conditional probability at least $1-\eta$,
\begin{equation}
\varepsilon+d_S/g\le A_KT_k,\qquad b/\sqrt g\le A_KT_k.
\label{eq:relative-concentration-event}
\end{equation}
\end{lemma}
\begin{proof}
The proof rescales the clean unstable block to identity covariance, controls five scalar sums, and then takes finite nets.
This retains the common signal energy in the cross block and its spectral denominator.

In orthonormal coordinates of $\mathcal U$ and $\mathcal S$, put
$\va_j=\mSigma_{k,\mathcal U}^{-1/2}\vu_j$,
$\boldsymbol\zeta_j=\mSigma_{k,\mathcal U}^{-1/2}P_{\mathcal U}\vxi_j$,
and $\boldsymbol\nu_j=P_{\mathcal S}\vxi_j$.
Then
\begin{equation}
\frac1k\sum_j\va_j\va_j^\top=\mI_m,\qquad
\mathbb E\boldsymbol\zeta_j\boldsymbol\zeta_j^\top
 =\sigma^2\mSigma_{k,\mathcal U}^{-1},\qquad
\mathbb E\boldsymbol\nu_j\boldsymbol\nu_j^\top=\sigma^2\mI_s,\qquad
\mathbb E\boldsymbol\zeta_j\boldsymbol\nu_j^\top=0.
\label{eq:relative-whitened-moments}
\end{equation}
The directional $\psi_2$ bounds are $K\sigma/\sqrt g$ for $\boldsymbol\zeta_j$ and $K\sigma$ for $\boldsymbol\nu_j$.
Expanding the required blocks gives
\begin{align}
\mSigma_{k,\mathcal U}^{-1/2}\boldsymbol E_{UU}\mSigma_{k,\mathcal U}^{-1/2}
 &=\frac1k\sum_j\left(\va_j\boldsymbol\zeta_j^\top+
 \boldsymbol\zeta_j\va_j^\top+
 \boldsymbol\zeta_j\boldsymbol\zeta_j^\top-\sigma^2\mSigma_{k,\mathcal U}^{-1}\right),\label{eq:relative-uu-expansion}\\
\boldsymbol E_{SS}
 &=\frac1k\sum_j(\boldsymbol\nu_j\boldsymbol\nu_j^\top-\sigma^2\mI_s),\label{eq:relative-ss-expansion}\\
\mSigma_{k,\mathcal U}^{-1/2}\boldsymbol E_{US}
 &=\frac1k\sum_j(\va_j\boldsymbol\nu_j^\top+
 \boldsymbol\zeta_j\boldsymbol\nu_j^\top).\label{eq:relative-us-expansion}
\end{align}
Isotropy centers the last noise product even when the coordinates of a noise vector are dependent.

For fixed unit $\vx\in\mathbb R^m$ and $\vy\in\mathbb R^s$, the squared sub-Gaussian scales of the two signal--noise sums satisfy
\begin{align}
\sum_j\|2(\vx^\top\va_j)(\vx^\top\boldsymbol\zeta_j)\|_{\psi_2}^2
 &\le4K^2\sigma^2k/g,\label{eq:relative-uu-scalar}\\
\sum_j\|(\vx^\top\va_j)(\vy^\top\boldsymbol\nu_j)\|_{\psi_2}^2
 &\le K^2\sigma^2k.\label{eq:relative-us-scalar}
\end{align}
Both inequalities use $\sum_j(\vx^\top\va_j)^2=k$.
The centered noise squares in~\eqref{eq:relative-uu-expansion} and~\eqref{eq:relative-ss-expansion} have $\psi_1$ scales at most $CK^2\sigma^2/g$ and $CK^2\sigma^2$, respectively.
The scalar noise product in~\eqref{eq:relative-us-expansion} has scale at most $CK^2\sigma^2/\sqrt g$.
For the product bound, use $|XY|/(ab)\le X^2/(2a^2)+Y^2/(2b^2)$ and Cauchy--Schwarz on the exponential when the factors have $\psi_2$ bounds $a,b$; independence of the factors is unnecessary.
Centering changes the $\psi_1$ scale by at most a factor two.

Scalar sub-Gaussian concentration bounds a centered sum with squared scale total $V$ by $C\sqrt{Vh}$.
Scalar Bernstein bounds the average of $k$ centered variables of common $\psi_1$ scale $L$ by $CL(\sqrt{h/k}+h/k)$.
Each bound has failure probability at most $2e^{-h}$; these are the independent-sum inequalities in \citet[Theorem~2.6.2, p.~29, and Theorem~2.8.1, p.~36]{vershynin2018high}.
Apply them to the five scalar sums above and take $1/4$-nets of the two unit spheres, of sizes at most $9^m$ and $9^s$.
The operator norm is at most twice the maximal net quadratic form for a symmetric matrix, or twice the maximal net bilinear form for a rectangular matrix: approximating maximizing unit vectors changes the form by at most half the operator norm.
With $h=d\log9+\log(12/\eta)$, the union bound fails with probability at most
$(4\cdot9^m+2\cdot9^s+4\cdot9^{m+s})e^{-h}\le5\eta/6$.
On the remaining event,
\begin{equation}
\begin{aligned}
\varepsilon&\le C\left[K\frac{\sigma}{\sqrt g}\sqrt{\frac hk}
 +K^2\frac{\sigma^2}{g}\left(\sqrt{\frac hk}+\frac hk\right)\right],\\
d_S&\le CK^2\sigma^2\left(\sqrt{\frac hk}+\frac hk\right),\\
b&\le C\left[K\sigma\sqrt{\frac hk}
 +K^2\frac{\sigma^2}{\sqrt g}\left(\sqrt{\frac hk}+\frac hk\right)\right].
\end{aligned}
\label{eq:relative-block-radii}
\end{equation}
Since $d\ge2$, $h\le C'[d+\log(1/\eta)]$.
Absorbing the universal constants and $K,K^2$ proves~\eqref{eq:relative-concentration-event}.
If $\sigma=0$, all noises vanish almost surely and the same conclusion holds with zero radii.
\end{proof}

\begin{theorem}[Finite-pair subspace rate]
\label{thm:pair-subspace-rate}
Under Assumption~\ref{ass:pairs}, let $T_k$ be given by~\eqref{eq:relative-pair-scale}.
There is a constant $C_K$, depending only on $K$, such that with conditional probability at least $1-\eta$,
\begin{equation}
\begin{aligned}
\|P_{k,r_*}-P_{\mathcal U}\|_{\rm op}&\le\gamma_k,
&\gamma_k&=\min\{1,C_KT_k\},\\
\|P_{\mathcal U}(\mI-P_{k,r})\|_{\rm op}&\le\gamma_k
&&\text{for every }r_*\le r\le d.
\end{aligned}
\label{eq:global-pair-subspace-rate}
\end{equation}
All projectors use one ordered eigenbasis and a fixed measurable tie rule.
\end{theorem}
\begin{proof}
Relative block concentration gives accurate rotation when its scale is small; the universal projector bound covers the remaining regime.
On~\eqref{eq:relative-concentration-event}, if $A_KT_k<1/2$, then $\zeta>1/2$ and Lemma~\ref{lem:relative-block-rotation} gives
\begin{equation}
\|P_{k,r_*}-P_{\mathcal U}\|_{\rm op}
\le\frac{b}{\sqrt g\,\zeta}\le2A_KT_k,
\label{eq:small-relative-error}
\end{equation}
using $\sin[\tfrac12\arctan(2x)]\le x$ for $x\ge0$.
If $A_KT_k\ge1/2$, the norm of the difference of two orthogonal projectors is at most one.
Taking $C_K=2A_K$ proves the first inequality in~\eqref{eq:global-pair-subspace-rate} in both regimes.
Nesting gives
$(\mI-P_{k,r_*})(\mI-P_{k,r})=\mI-P_{k,r}$, so
\begin{equation}
\|P_{\mathcal U}(\mI-P_{k,r})\|_{\rm op}
\le\|P_{\mathcal U}(\mI-P_{k,r_*})\|_{\rm op}
=\|P_{k,r_*}-P_{\mathcal U}\|_{\rm op}.
\end{equation}
This proves the simultaneous conclusion.
\end{proof}

When $k\ge d+\log(1/\eta)$, the linear term in $T_k$ is absorbed by its square-root term.
The full expression applies to every design satisfying Assumption~\ref{ass:pairs}, including $k<d$.
The rescaling in the proof is an analysis of the unknown clean covariance; it does not alter the GRIT estimator.

\paragraph{Sharp control from a full perturbation norm.}
\begin{proposition}[Sharp projector envelope]
\label{prop:sharp-projector-envelope}
Let $\boldsymbol M=\mSigma_k+\sigma^2\mI$ and let $\boldsymbol E$ be symmetric with $\|\boldsymbol E\|_{\rm op}\le e<g/2$.
The leading rank-$r_*$ projector $Q$ of $\boldsymbol M+\boldsymbol E$ is unique and
\begin{equation}
\|Q-P_{\mathcal U}\|_{\rm op}
\le\sin\!\left[\frac12\arcsin(2e/g)\right].
\label{eq:sharp-projector-envelope}
\end{equation}
For every $e\in[0,g/2)$, a positive-semidefinite two-dimensional example attains equality.
\end{proposition}
\begin{proof}
The proof bounds the off-diagonal entry in a principal-angle plane and constructs the perturbation that exactly cancels it.
Min--max gives a gap at least $g-2e>0$ along $\boldsymbol M+t\boldsymbol E$, $0\le t\le1$.
For the principal pair~\eqref{eq:principal-plane-pair}, spectral invariance yields
\begin{equation}
0=\vx^\top(\boldsymbol M+\boldsymbol E)\vy
 =-(\sin\theta\cos\theta)\vu^\top\mSigma_k\vu+\vx^\top\boldsymbol E\vy,
\qquad g\sin\theta\cos\theta\le e.
\end{equation}
The continuous angle path cannot reach $\pi/4$, because that would require $g/2\le e$.
Thus $\sin(2\theta)\le2e/g$ on the increasing branch, proving~\eqref{eq:sharp-projector-envelope}.

For attainment, choose $0\le\theta<\pi/4$, $c=\cos\theta$, $h=\sin\theta$, $\vx=(c,h)^\top$, and $\vy=(-h,c)^\top$.
Take
\begin{equation}
\boldsymbol M=\sigma^2\mI+\operatorname{diag}(g,0),\qquad
\boldsymbol M+\boldsymbol E=\sigma^2\mI+gc^2\vx\vx^\top+gh^2\vy\vy^\top.
\end{equation}
In the $(\vx,\vy)$ basis the perturbation is
$\left(\begin{smallmatrix}0&gch\\gch&0\end{smallmatrix}\right)$, so $\|\boldsymbol E\|_{\rm op}=gch$.
The perturbed matrix is positive semidefinite, its gap is $g\cos(2\theta)>0$, and its projector error is $\sin\theta$.
The map $\theta\mapsto g\sin\theta\cos\theta$ covers $[0,g/2)$, proving the claim.
\end{proof}

\subsection{Minimax estimation of Gaussian pair geometry}
\label{app:spectral-minimax}
Under independent isotropic Gaussian pair noise and known rank, GRIT's spectral estimator attains the minimax rate for expected squared subspace error over the clean-signal class specified below.
The following corollary specializes the singular-space estimation results of \citet{cai2018rate} to the pair matrix.

\begin{corollary}[Gaussian pair-geometry minimax rate]
\label{cor:spectral-minimax}
Observe $\boldsymbol D=\boldsymbol L+\sigma\boldsymbol G\in\mathbb R^{d\times k}$, where $\boldsymbol G$ has independent standard normal entries and $\sigma,g>0$.
Suppose the rank $r_*$ is known, with $1\le r_*\le\min(d,k)/16$, and define
\begin{equation}
\mathcal L_{k,g}=\{\boldsymbol L\in\mathbb R^{d\times k}:
 \operatorname{rank}(\boldsymbol L)=r_*,\ s_{r_*}(\boldsymbol L)^2\ge kg\},
\qquad
\mathfrak r_{k,g,\sigma}=
\min\!\left\{1,\frac dk\left(\frac{\sigma^2}{g}+\frac{\sigma^4}{g^2}\right)\right\},
\label{eq:gaussian-clean-class}
\end{equation}
where $s_{r_*}$ denotes the $r_*$th singular value.
Let $P_{\boldsymbol L}$ project onto $\operatorname{col}(\boldsymbol L)$, and let $\widehat P_{\rm GRIT}$ project onto the leading $r_*$ eigenvectors of $\boldsymbol D\boldsymbol D^\top/k$.
There are universal constants $c_0,C_0>0$ such that
\begin{equation}
 c_0\mathfrak r_{k,g,\sigma}
 \le\inf_{\widehat P}\sup_{\boldsymbol L\in\mathcal L_{k,g}}
       \mathbb E\|\widehat P-P_{\boldsymbol L}\|_{\rm op}^2
 \le\sup_{\boldsymbol L\in\mathcal L_{k,g}}
       \mathbb E\|\widehat P_{\rm GRIT}-P_{\boldsymbol L}\|_{\rm op}^2
 \le C_0\mathfrak r_{k,g,\sigma}.
\label{eq:gaussian-grit-minimax}
\end{equation}
The infimum ranges over measurable rank-$r_*$ orthogonal-projector estimators based on $\boldsymbol D$ and the supplied rank.
\end{corollary}
\begin{proof}
Rescale the pair matrix to unit noise and apply the upper and lower singular-space bounds of \citet[Theorems~3--4, pp.~70--71]{cai2018rate}.
Their left-space squared spectral-sine rate for a $d\times k$ rank-$r_*$ matrix with least singular value at least $t$ is
$\min\{1,d(t^2+k)/t^4\}$.
Set $t=\sqrt{kg}/\sigma$ for $\boldsymbol D/\sigma$ to obtain
\begin{equation}
\frac{d(t^2+k)}{t^4}
=\frac dk\left(\frac{\sigma^2}{g}+\frac{\sigma^4}{g^2}\right).
\end{equation}
For the lower bound, apply their right-singular-space result to $\boldsymbol D^\top/\sigma$, with $(p_1,p_2)=(k,d)$; our rank restriction implies both $r_*\le k/16$ and $r_*\le d/2$.
The spectral-sine loss equals the operator norm of the difference of the corresponding rank-$r_*$ projectors.
Finally, the leading left singular space of $\boldsymbol D$ is exactly the leading eigenspace of $\boldsymbol D\boldsymbol D^\top/k$.
These identifications give all three inequalities in~\eqref{eq:gaussian-grit-minimax}.
\end{proof}
 \section{Logistic Optimizer and Transport Refinements}
\label{app:logistic-refinements}
The main bounds retain a simple direct-shift term. This appendix gives optional
refinements specific to population logistic fitting, and derives the improved
quadratic-consistency coefficients without additional curvature assumptions.
Throughout, $\varphi(t)=\log(1+e^{-t})$ and source feature second moments are
finite. These moments justify differentiating under the expectation and give
\begin{equation}
\nabla^2R_p(\vw)
=\mathbb E_p\frac{\rvz\rvz^\top}{4\cosh^2(\vw^\top\rvz/2)},
\qquad
0\preceq\nabla^2R_p(\vw)\preceq\tfrac14\mathbb E_p[\rvz\rvz^\top].
\label{eq:logistic-hessian-refined}
\end{equation}

\subsection{A source-dependent direct-shift coefficient}
\label{app:logistic-transport-refined}
For a finite coefficient $\vw$, define
$\vu=P_{\mathcal U}\vw$, $m=\ry\vw^\top\rvz$, and
\begin{equation}
\beta(m)=(1+e^m)^{-1},\qquad
B_p(\vw)=\{\mathbb E_p\beta(m)^2\}^{1/2}.
\end{equation}
Let $\mathcal Q_\rho(p)$ contain the feature-space targets admitting a coupling
with $p$ that preserves $(P_{\mathcal S}\rvz,\ry)$ and satisfies
$\mathbb E\Delta\Delta^\top\preceq\rho^2P_{\mathcal U}$.
This is a class of distributions on features; it does not impose image
constraints from a prescribed nonlinear feature extractor.

\begin{proposition}[Locally sharp logistic transport]
\label{prop:logistic-local-transport}
For $\rho\ge0$, put $a=\rho\|\vu\|_2$. Then
\begin{equation}
aB_p(\vw)
\le\sup_{q\in\mathcal Q_\rho(p)}\{R_q(\vw)-R_p(\vw)\}
\le\min\{a,\ aB_p(\vw)+a^2/8\}.
\label{eq:logistic-local-transport}
\end{equation}
In particular, the right derivative of the supremum at $\rho=0$ is
$\|\vu\|_2B_p(\vw)$.
\end{proposition}
\begin{proof}
The proof bounds the second-order loss remainder and constructs a displacement
aligned with the first-order loss derivative. For an admissible coupling put
$h=\ry\vw^\top\Delta$. Its preserved stable projection implies
$\Delta\in\mathcal U$. Since $0\le\varphi''\le1/4$,
\begin{align}
R_q(\vw)-R_p(\vw)
&\le-\mathbb E[\beta(m)h]+\tfrac18\mathbb Eh^2,
\label{eq:logistic-transport-taylor}\\
\mathbb Eh^2&\le\rho^2\|\vu\|_2^2=a^2,\\
-\mathbb E[\beta(m)h]&\le B_p(\vw)(\mathbb Eh^2)^{1/2}
\le aB_p(\vw).
\end{align}
The first line is Taylor's integral bound, the second is the displacement moment
bound, and the third is Cauchy--Schwarz. The unit-Lipschitz bound supplies the
alternative upper bound $a$.

If $a>0$, finite margins give $B_p(\vw)>0$. Define the conditional displacement
\begin{equation}
\Delta=-\frac{\ry\,a\beta(m)}{B_p(\vw)\|\vu\|_2^2}\vu.
\label{eq:logistic-transport-attaining-map}
\end{equation}
It preserves the stable features and label, has finite second moments, and
satisfies
$\mathbb E\Delta\Delta^\top
=\rho^2\vu\vu^\top/\|\vu\|_2^2\preceq\rho^2P_{\mathcal U}$.
The margin becomes $m-a\beta(m)/B_p(\vw)$. Convexity gives a loss increase at
least $a\beta(m)^2/B_p(\vw)$, whose expectation proves the lower bound.
The bounds squeeze the right derivative. If $a=0$, every admissible displacement
preserves the score and both bounds are zero.
\end{proof}

Apply the upper bound to the source-to-bridge coupling and add the exact
$C_{\rm br}(\vw)$ for a refinement of Theorem~\ref{thm:risk-decomposition}.
The lower bound allows its target to depend on the fixed coefficient; it is not
a lower bound for every specified target. No pair assumption is used here.

\begin{proposition}[Radial stationarity reduces the local coefficient]
\label{prop:logistic-radial}
Suppose $\boldsymbol C\succeq0$, $\lambda\ge0$, and
\begin{equation}
\vw^\top\nabla R_p(\vw)+2\lambda\vw^\top\boldsymbol C\vw=0.
\label{eq:logistic-radial-condition}
\end{equation}
Then
\begin{equation}
B_p(\vw)^2\le\tfrac14-\lambda\vw^\top\boldsymbol C\vw\le\tfrac14.
\label{eq:logistic-radial-coefficient}
\end{equation}
Every exact population hard-GRIT optimizer satisfies this condition with
$\lambda=0$, and every exact unconstrained quadratic-CR optimizer satisfies it
with its penalty matrix and strength. Their direct-shift contribution is at most
\begin{equation}
\min\{a,\ a/2+a^2/8\},\qquad
a=\rho_\Delta\|P_{\mathcal U}\vw\|_2.
\label{eq:logistic-fitted-transport}
\end{equation}
\end{proposition}
\begin{proof}
A scalar inequality converts the radial first-order condition into a moment
bound for the logistic derivative. For $t\in(0,1)$ set
$F(t)=t\log((1-t)/t)-1/2+2t^2$. Direct differentiation gives
\begin{equation}
F(1/2)=F'(1/2)=0,\qquad
F''(t)=4-\frac1{t(1-t)^2}<0,
\end{equation}
where $t(1-t)^2\le4/27$ proves strict negativity. Consequently $F(t)\le0$.
Substitute $t=\beta(m)$, so $\log((1-t)/t)=m$, to obtain
\begin{equation}
\beta(m)^2\le\tfrac14-\tfrac12m\beta(m),\qquad
\mathbb E_p[m\beta(m)]=-\vw^\top\nabla R_p(\vw)
=2\lambda\vw^\top\boldsymbol C\vw.
\end{equation}
Taking expectations proves~\eqref{eq:logistic-radial-coefficient}.
For a hard optimizer, its feasible set is a linear subspace containing the
entire line through $\vw$, so differentiating risk along that line proves the
radial condition. Unconstrained CR stationarity proves the other case.
Proposition~\ref{prop:logistic-local-transport} gives
\eqref{eq:logistic-fitted-transport}; for CR its coefficient $1/2$ can be
replaced by $\sqrt{1/4-\lambda\vw^\top\boldsymbol C\vw}$.
\end{proof}

\paragraph{Sharpness and optimization scope.}
The coefficient $1/2$ cannot be decreased uniformly for learned CR predictors.
Let source features be $(0,X)$ with $X$ a fair sign, take
$\boldsymbol C=\operatorname{diag}(0,c)$ with $c>0$, and fix $\lambda>0$.
For small $b>0$, generate $\ry=X$ with probability
$s(b)+2\lambda cb$, where $s(b)=(1+e^{-b})^{-1}$, and otherwise set $\ry=-X$.
The fitted second coefficient is $b$ because its stationarity equation is
$s(w_2)-\mathbb P(\ry=X)+2\lambda cw_2=0$; choose first coefficient zero.
The conditional map $\Delta=-\ry\rho\ve_2$ preserves stable features and labels
and has radius $\rho$. Since all source margins tend to zero,
\begin{equation}
\frac{R_q(b\ve_2)-R_p(b\ve_2)}{\rho b}\longrightarrow\tfrac12
\qquad(b\downarrow0,\ \rho>0\text{ fixed}).
\end{equation}
The penalty matrix is realized exactly by noiseless Assumption~\ref{ass:pairs} differences
$\sqrt c\,\ve_2$. This local example does not establish joint sharpness of
the full CR target-risk bound. The quadratic correction in
\eqref{eq:logistic-fitted-transport} cannot simply be omitted at large radius:
$\varphi(-a)-\varphi(0)$ eventually grows with slope one. An empirical optimizer
only supplies an empirical radial identity. If the population radial residual
is $r=\vw^\top\nabla(R_p+\lambda\vw^\top\boldsymbol C\vw)$, the valid
replacement in~\eqref{eq:logistic-radial-coefficient} adds $r/2$ on the right.

\subsection{Quadratic consistency from exact optimizer stationarity}
\label{app:cr-optimizer-refined}
For this subsection, abbreviate $R=R_p$, $\rho=\rho_\Delta$,
$\sigma=\sigma_\xi$, and $\vv=\vw_{\mathcal S}^*$. On a covariance event, let
\begin{equation}
\boldsymbol C=\widetilde{\mSigma}_k
=\boldsymbol M+\boldsymbol E,\qquad
\boldsymbol M=\mSigma_k+\sigma^2\mI,\qquad
\|\boldsymbol E\|_{\rm op}\le e,
\label{eq:cr-refined-setup}
\end{equation}
where $\mSigma_kP_{\mathcal S}=0$ and
$\mSigma_k\succeq gP_{\mathcal U}$, $g>0$.
Let $\vw$ be an existing exact unconstrained minimizer of
$R(\vw)+\lambda\vw^\top\boldsymbol C\vw$, $\lambda>0$.
The pair covariance $\boldsymbol C$ is positive semidefinite. The mean matrix,
clean gap, covariance error, and stable oracle are analysis quantities.

\begin{proposition}[Exact optimizer comparison]
\label{prop:cr-exact-identity}
For any differentiable convex risk $R$, positive semidefinite
$\boldsymbol C$, and existing exact penalized optimizer $\vw$, define
$D_R(\vv,\vw)=R(\vv)-R(\vw)-\nabla R(\vw)^\top(\vv-\vw)$.
For every comparator $\vv$,
\begin{align}
R(\vw)-R(\vv)
&=2\lambda\vw^\top\boldsymbol C(\vv-\vw)-D_R(\vv,\vw)
\label{eq:cr-exact-identity}\\
&=\lambda\{\vv^\top\boldsymbol C\vv-\vw^\top\boldsymbol C\vw
 -(\vv-\vw)^\top\boldsymbol C(\vv-\vw)\}-D_R(\vv,\vw).
\nonumber
\end{align}
Both $D_R(\vv,\vw)$ and the omitted quadratic in ordinary objective comparison
are nonnegative.
\end{proposition}
\begin{proof}
Substitute the unconstrained first-order condition
$\nabla R(\vw)=-2\lambda\boldsymbol C\vw$ into the definition of $D_R$.
Expanding the quadratic difference gives the second line. Convexity of $R$ and
positive semidefiniteness of $\boldsymbol C$ give the stated signs.
\end{proof}

\begin{proposition}[Positive CR bound from stationarity]
\label{prop:cr-stationarity-bound}
Under~\eqref{eq:cr-refined-setup}, with a stable source optimum $\vv$, and the
moment and conditional-shift assumptions of Theorem~\ref{thm:risk-decomposition},
\begin{equation}
\begin{aligned}
R_q(\vw)\le{}&R_{p,\mathcal S}^*+C_{\rm br}(\vw)
 +\frac{\lambda\sigma^2}{2}\|\vv\|_2^2
 +\frac{\rho^2}{8\lambda(g+\sigma^2)}\\
&+2\lambda e\|\vw\|_2\|\vv-\vw\|_2.
\end{aligned}
\label{eq:cr-stationarity-bound}
\end{equation}
The statement holds simultaneously for every strength, existing optimizer, and
stable optimum on the same covariance event.
\end{proposition}
\begin{proof}
Retain the first-order comparison in~\eqref{eq:cr-exact-identity}, split it into
stable and unstable blocks, and complete two squares. Put
$\vs=P_{\mathcal S}\vw$, $\vu=P_{\mathcal U}\vw$, and $t=\|\vu\|_2$.
Since $\vv\in\mathcal S$, block orthogonality gives
\begin{align}
R(\vw)-R(\vv)
&\le2\lambda\sigma^2\vs^\top(\vv-\vs)
 -2\lambda\vu^\top\boldsymbol M\vu
 +2\lambda\vw^\top\boldsymbol E(\vv-\vw)
\label{eq:cr-signed-blocks}\\
&\le2\lambda\sigma^2\vs^\top(\vv-\vs)
 -2\lambda(g+\sigma^2)t^2
 +2\lambda e\|\vw\|_2\|\vv-\vw\|_2.
\nonumber
\end{align}
The first line discards only the nonnegative remainder $D_R$; the covariance
correction there may have either sign. The second uses the clean eigenvalue
lower bound and the operator-norm covariance event. Complete the squares:
\begin{align}
2\lambda\sigma^2\vs^\top(\vv-\vs)
&=\tfrac12\lambda\sigma^2\|\vv\|_2^2
 -2\lambda\sigma^2\|\vs-\vv/2\|_2^2,\\
\rho t-2\lambda(g+\sigma^2)t^2
&=\frac{\rho^2}{8\lambda(g+\sigma^2)}
 -2\lambda(g+\sigma^2)
 \left(t-\frac{\rho}{4\lambda(g+\sigma^2)}\right)^2.
\end{align}
Add the bridge and direct-shift comparison, discard the displayed nonpositive
squares, and use $R(\vv)=R_{p,\mathcal S}^*$. The same covariance event gives the
bound simultaneously for every strength and optimizer satisfying the stated
conditions. The cases $\sigma=0$ and $\rho=0$
require no division by either quantity and are included.
\end{proof}

The covariance remainder vanishes at $\vw=\vv$; bounding the two quadratic forms
separately would lose this cancellation. Its coefficient is not always smaller
than that of ordinary objective comparison, so two valid \emph{complete}
bounds can be combined by taking their minimum. Halving the two displayed
population tradeoff coefficients does not prove their joint logistic sharpness.
If an approximate coefficient has residual
$\boldsymbol r=\nabla R(\vw)+2\lambda\boldsymbol C\vw$, identity
\eqref{eq:cr-exact-identity} has the additional term
$\boldsymbol r^\top(\vw-\vv)$. A source-objective gap alone cannot be inserted
into this stationarity proof as the same additive error.

\subsection{Optional curvature and linear versus quadratic stable cost}
\label{app:cr-curvature}
\begin{proposition}[One-segment curvature refinement]
\label{prop:cr-segment-curvature}
In addition to the preceding CR conditions, suppose
\begin{equation}
D_R(\vv,\vw)\ge\tfrac12\mu_{\mathcal S}
 \|P_{\mathcal S}(\vw-\vv)\|_2^2
 +\tfrac12\mu_{\mathcal U}\|P_{\mathcal U}\vw\|_2^2,
\quad \mu_{\mathcal S},\mu_{\mathcal U}\ge0.
\label{eq:cr-segment-curvature}
\end{equation}
Then
\begin{equation}
\begin{aligned}
R_q(\vw)\le{}&R_{p,\mathcal S}^*+C_{\rm br}(\vw)
 +\frac{2\lambda^2\sigma^4\|\vv\|_2^2}
        {\mu_{\mathcal S}+4\lambda\sigma^2}\\
&+\frac{\rho^2}{2\{\mu_{\mathcal U}+4\lambda(g+\sigma^2)\}}
 +2\lambda e\|\vw\|_2\|\vv-\vw\|_2.
\end{aligned}
\label{eq:cr-curvature-bound}
\end{equation}
When $\mu_{\mathcal S}=\sigma=0$, define the stable contribution to be zero.
\end{proposition}
\begin{proof}
Retain the assumed remainder in the exact comparison, thereby strengthening
each quadratic before completing its square. With $\boldsymbol a=\vs-\vv$,
\begin{align}
2\lambda\sigma^2\vs^\top(\vv-\vs)
 -\tfrac12\mu_{\mathcal S}\|\boldsymbol a\|_2^2
&=-2\lambda\sigma^2\vv^\top\boldsymbol a
 -\tfrac12(\mu_{\mathcal S}+4\lambda\sigma^2)\|\boldsymbol a\|_2^2\\
&\le\frac{2\lambda^2\sigma^4\|\vv\|_2^2}
             {\mu_{\mathcal S}+4\lambda\sigma^2},\\
\rho t-2\lambda(g+\sigma^2)t^2-\tfrac12\mu_{\mathcal U}t^2
&\le\frac{\rho^2}{2\{\mu_{\mathcal U}+4\lambda(g+\sigma^2)\}}.
\end{align}
Both inequalities maximize a concave quadratic. If the stable denominator is
zero, its undivided expression is identically zero. Apply the same covariance
and transport steps as in Proposition~\ref{prop:cr-stationarity-bound}.
\end{proof}

\begin{lemma}[A sufficient logistic curvature condition]
\label{lem:logistic-bounded-score-curvature}
If $\mathbb E_p[\rvz\rvz^\top]\succeq\kappa\mI$, $\kappa>0$, and
$|\vv^\top\rvz|,|\vw^\top\rvz|\le B$ source-almost surely, then
\eqref{eq:cr-segment-curvature} holds with
\begin{equation}
\mu_{\mathcal S}=\mu_{\mathcal U}
=\frac{\kappa}{4\cosh^2(B/2)}.
\label{eq:cr-bounded-score-curvature}
\end{equation}
\end{lemma}
\begin{proof}
Bound every score on the segment between the two coefficients, and integrate
the resulting Hessian lower bound. With $\boldsymbol d=\vw-\vv$,
Taylor's integral identity in the required argument order is
\begin{equation}
D_R(\vv,\vw)
=\int_0^1 t\,\boldsymbol d^\top
 \nabla^2R(\vv+t\boldsymbol d)\boldsymbol d\,dt.
\label{eq:cr-reverse-remainder-integral}
\end{equation}
Every segment score has magnitude at most $B$. Equation~\eqref{eq:logistic-hessian-refined}
therefore lower-bounds its Hessian by $\kappa\mI/(4\cosh^2(B/2))$.
Integrating $t$ over $[0,1]$ and decomposing $\boldsymbol d$ orthogonally proves
the assertion. More generally, a matrix lower bound
$\mu_{\mathcal S}P_{\mathcal S}+\mu_{\mathcal U}P_{\mathcal U}$ along this
segment suffices. Bounds on only the two diagonal Hessian blocks do not suffice,
because the cross terms also affect the quadratic form.
\end{proof}

Bounded features and bounded endpoint coefficient norms imply the score bound.
The moment lower bound excludes directions nearly absent from the source.
These are optional sufficient conditions, not consequences of pair coverage.
For finite endpoints and a positive-definite source second moment, the weaker
source-dependent matrix
\begin{equation}
\mathbb E_p\frac{\rvz\rvz^\top}
 {4\cosh^2(\max\{|\vv^\top\rvz|,|\vw^\top\rvz|\}/2)}
\label{eq:cr-endpoint-curvature}
\end{equation}
is positive definite and lower-bounds the Hessian on the segment. Its positive
weight gives a positive quadratic form in every nonzero direction. This does
not supply a useful uniform eigenvalue lower bound: weak source directions and
large, saturated logistic scores can make it arbitrarily small. Optimizing the
penalized objective does not by itself prevent either phenomenon. Full-source
stationarity of the stable oracle is unnecessary for
Proposition~\ref{prop:cr-segment-curvature}.

Let $\tau=\lambda\sigma^2$. The stable contribution is
$2\tau^2\|\vv\|_2^2/(\mu_{\mathcal S}+4\tau)$. A positive curvature lower
bound uniform along the considered strengths makes it quadratic for small
$\tau$; setting $\mu_{\mathcal S}=0$ recovers $\tau\|\vv\|_2^2/2$.
Only this stable contribution becomes quadratic; finite-pair and direct-shift
contributions need not do so.

\begin{proposition}[Fixed-source local penalty bias]
\label{prop:cr-fixed-source-bias}
Suppose a fixed differentiable convex risk is twice continuously differentiable
near $\vv$, with $\nabla R(\vv)=0$ and
$\boldsymbol H=\nabla^2R(\vv)\succ0$. For a fixed positive semidefinite
$\boldsymbol C$, the penalized optimum is unique for all sufficiently small
$\lambda>0$ and
\begin{align}
\vw_\lambda&=\vv-2\lambda\boldsymbol H^{-1}\boldsymbol C\vv+o(\lambda),\\
R(\vw_\lambda)-R(\vv)
&=2\lambda^2(\boldsymbol C\vv)^\top
 \boldsymbol H^{-1}(\boldsymbol C\vv)+o(\lambda^2).
\label{eq:cr-fixed-source-bias}
\end{align}
\end{proposition}
\begin{proof}
Differentiate the optimizer equation and expand the risk at its stationary
source optimum. The map
$F(\vx,\lambda)=\nabla R(\vx)+2\lambda\boldsymbol C\vx$ vanishes at
$(\vv,0)$ and has invertible derivative $\boldsymbol H$ in $\vx$ there.
The implicit function theorem gives a local stationary branch with derivative
$-2\boldsymbol H^{-1}\boldsymbol C\vv$. Convexity makes it a global optimum.
The Hessian remains positive definite locally; a second global optimum would
contradict strict convexity on the initial segment joining the two optima.
Finally
$R(\vv+\boldsymbol h)-R(\vv)=\boldsymbol h^\top\boldsymbol H\boldsymbol h/2
+o(\|\boldsymbol h\|_2^2)$ gives the displayed coefficient. It is strictly
positive if $\boldsymbol C\vv\ne0$.
\end{proof}

\paragraph{A fixed logistic source with quadratic actual excess.}
Take independent fair signs $X,T$, features $(X,T)$, and
$\mathbb P(\ry=1\mid X,T)=s(bX)$ for fixed $b>0$.
Let $q=p$, $\mathcal S=\operatorname{span}(\ve_1)$, and use the population
pair matrix $\boldsymbol C=\operatorname{diag}(\sigma^2,g+\sigma^2)$ with
$g,\sigma>0$. This is a conditional mean under Assumption~\ref{ass:pairs}, rather than a claim of zero
finite-sample fluctuation. The stable and full-source optimum is $(b,0)$.
Let $h=s(b)(1-s(b))>0$. The source Hessian there is $h\mI$, and
Proposition~\ref{prop:cr-fixed-source-bias} gives
\begin{equation}
R_p(\vw_\lambda)-R_{p,\mathcal S}^*
=\frac{2\sigma^4b^2}{h}\lambda^2+o(\lambda^2).
\label{eq:cr-quadratic-example}
\end{equation}
In contrast, the simple population bound
\eqref{eq:cr-stationarity-bound} adds $\lambda\sigma^2b^2/2$, since both
shift terms and the covariance error vanish. Its excess-to-actual-excess ratio
therefore diverges. The full-source stationarity used here is stronger than
stable optimality or Assumption~\ref{ass:stability}. The fixed-source quantifier also matters for the
following complementary example.

\begin{proposition}[Linear stable cost over the allowed source class]
\label{prop:cr-linear-class}
There is a family of bounded-feature Assumption~\ref{ass:stability} source problems, with fixed nonzero
stable oracle norm, zero shift, and a population pair matrix under Assumption~\ref{ass:pairs}, whose actual
CR excess is of order $\lambda$. Consequently a source-independent
$O(\lambda^2)$ replacement is not valid over this class.
\end{proposition}
\begin{proof}
Let the source information weaken at the penalty scale. Fix $b>0$ and, for
$0<\lambda\le1/8$, take independent fair signs $V,T$, set
$X=\sqrt{8\lambda}\,V$, and use features $(X,T)$ with
$\mathbb P(\ry=1\mid X,T)=s(bX)$, $s(x)=(1+e^{-x})^{-1}$.
Set $q=p$, $\mathcal S=\operatorname{span}(\ve_1)$, and use
$\boldsymbol C=\operatorname{diag}(1,2)$.
This is the conditional mean in Assumption~\ref{ass:pairs} with noise variance one and clean covariance
$\operatorname{diag}(0,1)$; using the mean matrix does not claim that a random
finite-pair covariance has zero fluctuation. The stable optimum is $(b,0)$,
and $\|\rvz\|_2\le\sqrt2$.

Symmetry in $T$, convexity, and the positive-definite penalty give a fitted
coefficient $(c_\lambda,0)$. Put $a=\sqrt{8\lambda}$ and $\pi=s(ab)$.
The scalar source risk is $R(c)=\log(1+e^{ac})-\pi ac$, and stationarity gives
\begin{equation}
a\{s(ac_\lambda)-s(ab)\}+2\lambda c_\lambda=0.
\end{equation}
Its left side is strictly increasing, negative at zero and positive at $b$,
so $0<c_\lambda<b$. The sigmoid expansion is uniform on that interval and gives
\begin{equation}
2(c_\lambda-b)+2c_\lambda+O(\lambda)=0,
\qquad c_\lambda=b/2+O(\lambda).
\end{equation}
Uniform expansion of the scalar risk now gives
\begin{equation}
R(c)-R(b)=\tfrac18a^2(c-b)^2+O(a^4),\qquad
R(c_\lambda)-R(b)=\tfrac14b^2\lambda+O(\lambda^2).
\label{eq:cr-linear-class-excess}
\end{equation}
This contradicts a $K\lambda^2$ bound with $K$ independent of the source.
The source varies with $\lambda$, so it does not contradict
Proposition~\ref{prop:cr-fixed-source-bias}.
\end{proof}

\paragraph{Assumption scope.}
The optimizer comparison applies to differentiable convex risks with exact
unconstrained population fitting. The CR target-risk bounds combine it with the
stated pair-covariance and conditional-transport assumptions. The curvature
refinement uses the additional one-segment condition in~\eqref{eq:cr-segment-curvature}.
 \section{Finite-pair tightness and sharper risk comparisons}
\label{app:grit-tightness}
We connect subspace-estimation error to the risk of the predictor actually fitted by GRIT. The constructions use a fixed feature space and population source fitting. They distinguish an attained finite-pair rate, an asymptotically attained risk comparison, and the separate cost of removing additional directions.

\subsection{A risk comparison retaining source geometry}
\label{app:grit-local-comparison}
Let $\varphi(t)=\log(1+e^{-t})$. For a source distribution with finite second feature moments, define
\begin{equation}
\boldsymbol C_p=\mathbb E_p[\rvz\rvz^\top],\qquad
L_z^2=\|\boldsymbol C_p\|_{\rm op},\qquad
B_p(\vw)=\left(\mathbb E_p(1+e^{\ry\vw^\top\rvz})^{-2}\right)^{1/2}.
\label{eq:tightness-source-quantities}
\end{equation}
These are distribution-dependent analysis quantities.

\begin{lemma}[Source and direct-shift refinement]
\label{lem:grit-local-certificate}
Let $\vv\in\mathcal W_{\mathcal S}^*$ be finite, let $Q$ be an orthogonal projector, and let $\vw$ be an existing source-risk minimizer on $\ker Q$. Suppose the bridge and conditional-displacement assumptions of Theorem~\ref{thm:risk-decomposition} hold with radius $\rho_\Delta$. Set
$\vh=Q\vv$, $\boldsymbol g_* =\nabla R_p(\vv)$, and
$b=\rho_\Delta\|P_{\mathcal U}\vw\|_2$. Then
\begin{equation}
\begin{split}
R_q(\vw)\le R_p(\vv)+C_{\rm br}(\vw)
&+\min\left\{\mathbb E_p|\rvz^\top\vh|,
 -\boldsymbol g_*^\top\vh+\tfrac18\vh^\top\boldsymbol C_p\vh\right\}\\
&+\min\left\{b,\ bB_p(\vw)+\tfrac18 b^2\right\}.
\end{split}
\label{eq:grit-refined-complete-certificate}
\end{equation}
\end{lemma}
\begin{proof}
Project the stable oracle into the fitted feasible space. Compare its source risk using either logistic Lipschitzness or its Hessian upper bound, then retain the source-dependent first-order term in the direct-shift comparison.

Finite second moments justify differentiation under the expectation. Stable optimality and logistic curvature give
\begin{equation}
P_{\mathcal S}\boldsymbol g_*=0,\qquad
0\preceq\nabla^2R_p(\vx)\preceq\boldsymbol C_p/4
\quad\text{for every finite }\vx.
\label{eq:tightness-source-stationarity}
\end{equation}
The vector $\vv-\vh=(I-Q)\vv$ is feasible. Thus
\begin{align}
R_p(\vw)-R_p(\vv)
&\le R_p(\vv-\vh)-R_p(\vv),\label{eq:tightness-source-feasibility}\\
R_p(\vv-\vh)-R_p(\vv)
&\le\mathbb E_p|\rvz^\top\vh|,\label{eq:tightness-source-lipschitz}\\
R_p(\vv-\vh)-R_p(\vv)
&=-\boldsymbol g_*^\top\vh+
 \int_0^1(1-t)\vh^\top\nabla^2R_p(\vv-t\vh)\vh\,dt\nonumber\\
&\le-\boldsymbol g_*^\top\vh+\tfrac18\vh^\top\boldsymbol C_p\vh.
\label{eq:tightness-source-taylor}
\end{align}
Equation~\eqref{eq:tightness-source-lipschitz} uses $|\varphi'|\le1$, and the last line uses~\eqref{eq:tightness-source-stationarity}. The Taylor expression can be negative; a tilted feasible space can exploit source-predictive unstable information.

Under the stipulated source-to-bridge coupling, let
$M=\ry\vw^\top\rvz$, $T=\ry\vw^\top\Delta$, and
$\beta(M)=(1+e^M)^{-1}$. Since $0\le\varphi''\le1/4$,
\begin{align}
R_{\breve q}(\vw)-R_p(\vw)
&\le-\mathbb E[\beta(M)T]+\tfrac18\mathbb ET^2,\label{eq:tightness-direct-taylor}\\
\mathbb ET^2&\le b^2,\qquad
-\mathbb E[\beta(M)T]\le B_p(\vw)\sqrt{\mathbb ET^2}\le bB_p(\vw).
\label{eq:tightness-direct-cauchy}
\end{align}
The first inequality in~\eqref{eq:tightness-direct-cauchy} uses the displacement second moment, and the second is Cauchy--Schwarz. The same risk difference is at most $b$ by logistic Lipschitzness. Add the exact bridge bound and the smaller source and direct-shift bounds to obtain~\eqref{eq:grit-refined-complete-certificate}.
\end{proof}

\subsection{Proof of finite-pair tightness}
\label{app:grit-target-tightness}
We first obtain a finite-probability adverse angle from pairs with Gaussian noise, then construct fixed source and target distributions for which source refitting preserves its first-order risk cost. In this subsection, $\sigma=\sigma_\xi>0$.

\begin{lemma}[An adverse angle under Gaussian pair noise]
\label{lem:grit-adverse-angle}
In dimension two, let $\mathcal U=\operatorname{span}(\ve_2)$,
$\vu_j=\sqrt g\,\ve_2$ for $g>0$, and
$\vxi_j=\sigma(X_j,W_j)^\top$, where all $X_j,W_j$ are independent standard normal variables. Set
\begin{equation}
\delta_k=\sqrt{\frac1k\left(\frac{\sigma^2}{g}+\frac{\sigma^4}{g^2}\right)}.
\label{eq:tightness-pair-scale}
\end{equation}
If $\delta_k^2\le1/2048$, then, with probability at least $1/64$, the leading rank-one projector of the observed pair covariance equals
$Q_t=\vu_t\vu_t^\top$, where $\vu_t=(\sin t,\cos t)^\top$ and
\begin{equation}
\delta_k/6\le t\le16\delta_k,\qquad
\delta_k/12\le\|Q_t-P_{\mathcal U}\|_{\rm op}\le16\delta_k.
\label{eq:tightness-angle-event}
\end{equation}
\end{lemma}
\begin{proof}
A fourth-moment bound gives a cross-covariance fluctuation of its natural size with positive probability. A variance bound keeps the diagonal gap near $g$. The exact two-dimensional eigenvector formula then gives the angle.

Define $B_k=(\widetilde{\mSigma}_k)_{12}$,
$D_k=(\widetilde{\mSigma}_k)_{22}-(\widetilde{\mSigma}_k)_{11}$, and
$h^2=\sigma^2g+\sigma^4$. Independence and Gaussian moments give
\begin{align}
B_k&=k^{-1}\sum_j\sigma X_j(\sqrt g+\sigma W_j),
&\mathbb EB_k^2&=h^2/k,\label{eq:tightness-cross-moment}\\
D_k-g&=k^{-1}\sum_j[2\sqrt g\sigma W_j+\sigma^2(W_j^2-X_j^2)],
&\operatorname{Var}(D_k)&=4h^2/k.
\label{eq:tightness-gap-moment}
\end{align}
For one centered cross-term summand $A=\sigma X(\sqrt g+\sigma W)$,
\begin{equation}
\mathbb EA^4=3\sigma^4(g^2+6g\sigma^2+3\sigma^4)\le9h^4,
\qquad
\mathbb EB_k^4=\frac{k\mathbb EA^4+3k(k-1)h^4}{k^4}
\le\frac{9h^4}{k^2}.
\label{eq:tightness-fourth-moment}
\end{equation}
For any nonnegative $H$ with finite second moment,
\begin{equation}
\Pr(H\ge\theta\mathbb EH)\ge
\frac{(1-\theta)^2(\mathbb EH)^2}{\mathbb EH^2},\qquad 0<\theta<1.
\label{eq:tightness-second-fourth}
\end{equation}
Indeed, split $\mathbb EH$ below and above $\theta\mathbb EH$, then use Cauchy--Schwarz on the upper part. Apply~\eqref{eq:tightness-second-fourth} to $H=B_k^2$ with $\theta=1/4$. The distribution of $B_k$ is symmetric under $X_{1:k}\mapsto-X_{1:k}$. This symmetry, a second-moment tail bound, and~\eqref{eq:tightness-gap-moment} yield
\begin{align}
\Pr(B_k\ge h/(2\sqrt k))&\ge1/32,\nonumber\\
\Pr(B_k>8h/\sqrt k)&\le1/128,\nonumber\\
\Pr(|D_k-g|>g/2)&\le16\delta_k^2\le1/128.
\label{eq:tightness-event-probability}
\end{align}
For the second line, symmetry divides the two-sided second-moment tail bound by two. Therefore the event
\begin{equation}
E_k=\{h/(2\sqrt k)\le B_k\le8h/\sqrt k,\quad
                         g/2\le D_k\le3g/2\}
\label{eq:tightness-good-event}
\end{equation}
has probability at least $1/32-1/128-1/128=1/64$. No independence between $B_k$ and $D_k$ is required.

On $E_k$, the leading eigenspace is unique and its angle satisfies
$t=\tfrac12\arctan(2B_k/D_k)$. The arctangent argument is positive and at most
$32\delta_k\le1/\sqrt2$. For $0\le z\le1$, $z/2\le\arctan z\le z$, so
\begin{equation}
\delta_k/6\le B_k/(2D_k)\le t\le B_k/D_k\le16\delta_k.
\end{equation}
The projector error is $\sin t$. Since $0<t<\pi/4$, $t/2\le\sin t\le t$, proving~\eqref{eq:tightness-angle-event}. The Gaussian noises satisfy Assumption~\ref{ass:pairs} with a fixed directional sub-Gaussian constant. Their $k$ original draws remain the independent observations under the symmetrized pair-covariance convention.
\end{proof}

\paragraph{Proof of Proposition~\ref{prop:grit-target-tightness}.}
\begin{proof}
The source has a noisy unstable label signal. A small rotation in the adverse direction reduces the usable source information, and a fixed target shift reduces the same signed margins further. Giving both labels positive probability at each source feature point ensures an attained fitted optimum for every possible pair realization.

Set $\mathcal S=\operatorname{span}(\ve_1)$ and $\mathcal U=\operatorname{span}(\ve_2)$. Let $\Pr(\ry=1)=3/4$ and let $V\in\{-1,1\}$ satisfy
$\Pr(V=\ry\mid\ry)=3/4$. Fix the source and target by the coupling
\begin{equation}
\rvz=(1,V)^\top,\qquad \rvz'=(1,V+\ry)^\top.
\label{eq:tightness-fixed-source-target}
\end{equation}
They satisfy Assumption~\ref{ass:stability} with constant stable feature and conditional label probability $3/4$. Their stable-label marginals agree, so $\breve q=q$ and $C_{\rm br}=0$. Also
$\Delta=\ry\ve_2$, giving $\mathbb E\Delta\Delta^\top=P_{\mathcal U}$ and $\rho_\Delta=1$. Both distributions are fixed before the pair noise and independent of $g,\sigma,k$. Their features are bounded. The stable optimum and source second moment are
\begin{equation}
\vv=a\ve_1,\quad a=\log3,\qquad
\boldsymbol C_p=\begin{pmatrix}1&1/4\\1/4&1\end{pmatrix},\qquad
L_z=\sqrt{5/4}.
\label{eq:tightness-fixed-oracle}
\end{equation}

For a removed direction $\vu_t=(\sin t,\cos t)^\top$, the retained unit direction is
$\vv_t=(\cos t,-\sin t)^\top$. Let $\beta_t$ be the scalar coefficient of the fitted predictor
$\vw_t=\beta_t\vv_t$. Its objective is
\begin{equation}
F(\beta,t)=\mathbb E\varphi\{\beta(\ry\cos t-\ry V\sin t)\}.
\label{eq:tightness-fitted-objective}
\end{equation}
Each of the four pairs $(\ry,V)$ has probability at least $1/16$. Using
$\varphi(s)+\varphi(-s)\ge|s|$ gives, for every $t$,
\begin{equation}
F(\beta,t)\ge\frac{|\beta|}{16}
 \big(|\cos t-\sin t|+|\cos t+\sin t|\big)
\ge\frac{|\beta|}{8\sqrt2}.
\label{eq:tightness-global-existence}
\end{equation}
Hence $F$ is coercive uniformly in $t$. It is strictly convex because the two feature values $\cos t\pm\sin t$ cannot both vanish and logistic curvature is positive at every finite margin. Thus $\beta_t$ exists uniquely for every selected line, with
$|\beta_t|\le8\sqrt2\log2$ by comparison with $F(0,t)=\log2$.

At zero angle, $\beta_0=a$, $F_\beta(a,0)=0$, and $F_{\beta\beta}(a,0)=3/16$. The implicit-function theorem gives a smooth branch $\beta_t=a+O(t)$, which is the actual unique optimizer. Differentiating its source risk removes the optimizing-coefficient derivative by stationarity. Since $\mathbb E[V\mid\ry]=\ry/2$,
\begin{align}
\left.\frac d{dt}R_p(\vw_t)\right|_{t=0}
 &=a\mathbb E\frac{\ry V}{1+e^{a\ry}}=\frac{3a}{16},\label{eq:tightness-source-slope}\\
\left.\frac d{dt}\{R_q(\vw_t)-R_p(\vw_t)\}\right|_{t=0}
 &=a\mathbb E\frac1{1+e^{a\ry}}=\frac{3a}{8}.
\label{eq:tightness-direct-slope}
\end{align}
The second line follows because the target margin is the source margin minus $\beta_t\sin t$. Finite support and smoothness justify Taylor expansion, giving
\begin{equation}
R_p(\vw_t)-R_p(\vv)=\frac{3a}{16}t+O(t^2),\qquad
R_q(\vw_t)-R_p(\vw_t)=\frac{3a}{8}t+O(t^2).
\label{eq:tightness-joint-expansions}
\end{equation}

Define $S_t=L_za\sin t$ and $T_t=\|\vw_t\|_2\sin t$. Since $\beta_t\to a>0$, the ratios of the two expressions in~\eqref{eq:tightness-joint-expansions} to $S_t$ and $T_t$ tend respectively to $3/(16L_z)$ and $3/8$. Thus constants $b_0>0$ and $t_0\in(0,1/4)$ exist such that, for $0<t\le t_0$,
\begin{equation}
R_p(\vw_t)-R_p(\vv)\ge b_0S_t,\qquad
R_q(\vw_t)-R_p(\vw_t)\ge b_0T_t.
\label{eq:tightness-componentwise-lower}
\end{equation}
Both increments are positive on the same angle interval.

Use the Gaussian pair design in Lemma~\ref{lem:grit-adverse-angle}. Let
$\delta_0=\min\{1/\sqrt{2048},t_0/16\}$ and suppose $0<\delta_k\le\delta_0$. On its event of probability at least $1/64$, the fitted output is $\vw_t$ with $0<t\le t_0$ and $\sin t\ge\delta_k/12$. Adding~\eqref{eq:tightness-componentwise-lower} proves the proposition with $c=b_0/12$. It also proves its componentwise interpretation. The deterministic Lipschitz comparison gives the matching upper inequality on this event:
\begin{equation}
R_q(\vw_k)-R_{p,\mathcal S}^*
\le16\delta_k\big(L_z\|\vv\|_2+\rho_\Delta\|\vw_k\|_2\big).
\label{eq:tightness-joint-upper}
\end{equation}
Indeed, projecting $\vv$ onto the retained line costs at most $L_z\|\vv\|_2\sin t$ in source risk, while residual unstable sensitivity is $\|\vw_k\|_2\sin t$. The source and target feature moments and fitted coefficient norms are bounded independently of $g,\sigma,k$.
\end{proof}

\paragraph{Rate and probability scope.}
For fixed confidence $\eta\in(0,1)$, put $x=\sigma^2/g$ and $H=2+\log(1/\eta)$. The expanded pair factor satisfies
\begin{equation}
\sqrt H\,\delta_k
\le(\sqrt x+x)\sqrt{H/k}+xH/k
\le(\sqrt{2H}+H)\delta_k.
\label{eq:tightness-rate-comparison}
\end{equation}
This follows from $\sqrt{x+x^2}\le\sqrt x+x\le\sqrt2\sqrt{x+x^2}$ and $k\ge1$. The construction therefore covers both $\sigma/\sqrt{gk}$ and $\sigma^2/(g\sqrt k)$ regimes when $k$ makes $\delta_k$ small. It does not separately establish necessity of the lower-order $1/k$ concentration term. A uniform $o(\delta_k)$ target-excess bound at failure probability below $1/64$ would contradict the proposition. This is a lower bound for GRIT's fitted procedure at fixed dimension and confidence, rather than an infimum over all estimators. Negative-angle realizations can have negative first-order excess relative to the restricted stable comparator, so the statement does not give the same rate for expected signed excess.

\subsection{Positive bridge shift and an asymptotically attained comparison}
\label{app:grit-bridge-tightness}
The next construction adds a positive bridge covariate-shift contribution of the same order. Its target sequence is specified before the pair noise. It also attains the refined risk comparison to first order, including its source and direct-shift terms.

Fix $a,c,\rho>0$ and $\nu\in(1/2,1)$, and put $\mu=2\nu-1$. Under the source, let $X$ be uniform on $\{1,2\}$, let
$\Pr(\ry=1\mid X=x)=\pi_x=s(ax)$ with $s(t)=(1+e^{-t})^{-1}$, and let $J\in\{-1,1\}$ be independent with $\Pr(J=1)=\nu$. Set
\begin{equation}
\rvz=(X,c\ry J)^\top,\quad \vv=a\ve_1,\quad
\beta_*(X,\ry)=(1+e^{\ry aX})^{-1},\quad
\kappa=\mathbb E_p\beta_*,\quad B_*=(\mathbb E_p\beta_*^2)^{1/2}.
\label{eq:bridge-tightness-source}
\end{equation}
Here $\vv$ is the stable optimum and direct conditioning gives
\begin{equation}
\kappa=\pi_1(1-\pi_1)+\pi_2(1-\pi_2)>0,\qquad
B_*^2=\kappa/2>0.
\label{eq:bridge-tightness-constants}
\end{equation}
For $\varepsilon\in[0,1/2)$ define a target $q_\varepsilon$ by assigning masses $1/2+\varepsilon$ and $1/2-\varepsilon$ to $X=1,2$, retaining the conditional distribution of $(\ry,J)$ given $X$, and setting
\begin{equation}
\rvz'=\left(X,c\ry J+\rho\ry\frac{\beta_*(X,\ry)}{B_*}\right)^\top.
\label{eq:bridge-tightness-target}
\end{equation}
The bridge $\breve q$ has the source's equal stable masses and this target unstable conditional. Its coupling with the source fixes $(X,\ry,J)$ and satisfies
\begin{equation}
\Delta=\rho\ry\frac{\beta_*}{B_*}\ve_2,\qquad
\mathbb E_p\Delta\Delta^\top=\rho^2P_{\mathcal U}.
\label{eq:bridge-tightness-radius}
\end{equation}
Assumption~\ref{ass:stability} holds, all stable supports coincide, and source, bridge, and target features are uniformly bounded. The shift uses the fixed stable oracle, not a fitted coefficient.

For $Q_\theta=\vu_\theta\vu_\theta^\top$, where
$\vu_\theta=(\sin\theta,\cos\theta)^\top$, let $b_\theta$ be the scalar coefficient of the fitted source predictor
$\vw_\theta=b_\theta(\cos\theta,-\sin\theta)^\top$. Let
$\boldsymbol g_* =\nabla R_p(\vv)$,
$\vh_\theta=Q_\theta\vv$, and
$\ell_\theta=\rho\|P_{\mathcal U}\vw_\theta\|_2$. Define
\begin{equation}
\begin{split}
\mathcal C_\varepsilon(\theta)=R_p(\vv)+C_{{\rm br},\varepsilon}(\vw_\theta)
&-\boldsymbol g_*^\top\vh_\theta+\tfrac18\vh_\theta^\top\boldsymbol C_p\vh_\theta\\
&+\ell_\theta B_p(\vw_\theta)+\tfrac18\ell_\theta^2.
\end{split}
\label{eq:bridge-tightness-certificate}
\end{equation}
Lemma~\ref{lem:grit-local-certificate} makes this an upper bound on $R_{q_\varepsilon}(\vw_\theta)$.

\begin{proposition}[Joint attainment of a refined risk comparison]
\label{prop:grit-bridge-tightness}
For the fixed source~\eqref{eq:bridge-tightness-source}, use the Gaussian pair design of Lemma~\ref{lem:grit-adverse-angle} and the target $q_k=q_{\delta_k}$, with $\delta_k$ from~\eqref{eq:tightness-pair-scale}. For all sufficiently small positive $\delta_k$, an event of probability at least $1/64$ has the following properties. The fitted optimizer exists for every pair realization and has uniformly bounded norm. On the event, each of
\begin{equation}
R_p(\vw_k)-R_{p,\mathcal S}^*,\qquad
R_{\breve q}(\vw_k)-R_p(\vw_k),\qquad C_{{\rm br},k}(\vw_k)
\label{eq:bridge-three-contributions}
\end{equation}
is bounded above and below by positive constants times $\delta_k$, and
\begin{equation}
0\le\mathcal C_{\delta_k}(\theta_k)-R_{q_k}(\vw_k)=O(\delta_k^2),
\qquad
\frac{\mathcal C_{\delta_k}(\theta_k)-R_{p,\mathcal S}^*}
{R_{q_k}(\vw_k)-R_{p,\mathcal S}^*}\longrightarrow1.
\label{eq:bridge-tightness-ratio}
\end{equation}
The remainder constants depend only on the fixed source and shift parameters. Convergence is uniform over all angles in~\eqref{eq:tightness-angle-event} as $\delta_k\to0$. Source and target are specified before observing the pairs; the target depends on $k,g,\sigma$ only through its stable marginal.
\end{proposition}
\begin{proof}
Source stationarity makes refitting preserve the first-order rotation cost. The fixed shift in~\eqref{eq:bridge-tightness-target} attains the source-dependent transport coefficient at the stable oracle. Moving stable mass toward the higher-entropy label kernel adds a positive exact bridge term. The remaining slack in the bound is quadratic in the angle.

\paragraph{Existence and smooth fitting.}
The four source feature atoms $(x,cv)$, $x\in\{1,2\}$, $v\in\{-1,1\}$ each occur jointly with each label with probability at least
$\zeta=\tfrac12(1-\pi_2)(1-\nu)>0$. Pairing the two label losses at each atom gives
\begin{equation}
R_p(\vw)\ge\zeta\sum_{x\in\{1,2\},\ v\in\{-1,1\}}|xw_1+cvw_2|
\ge\zeta\sqrt2\min(1,c)\|\vw\|_2.
\label{eq:bridge-tightness-coercivity}
\end{equation}
The last inequality already follows from $(1,\pm c)$, whose absolute scores sum to $2\max(|w_1|,c|w_2|)$. The source risk is coercive and has positive-definite Hessian at every finite coefficient because its atoms span the plane and logistic curvature is positive. Every line-constrained optimum is therefore unique and finite, with norm at most
$\log2/[\zeta\sqrt2\min(1,c)]$.

Its scalar objective is
\begin{equation}
F(b,\theta)=\mathbb E_p\varphi\{b(\ry X\cos\theta-cJ\sin\theta)\}.
\label{eq:bridge-tightness-objective}
\end{equation}
At zero angle, $b_0=a$, $F_b(a,0)=0$, and
$F_{bb}(a,0)=\mathbb E[X^2\pi_X(1-\pi_X)]>0$. The implicit-function theorem gives $b_\theta=a+O(\theta)$, a smooth branch equal to the global line-constrained optimum by uniqueness.

\paragraph{Three positive first-order contributions.}
Differentiating the optimized source risk and the direct-shift difference at zero yields
\begin{align}
\left.\frac d{d\theta}R_p(\vw_\theta)\right|_0
 &=ac\mathbb E[J\beta_*]=ac\mu\kappa,\label{eq:bridge-tightness-source-slope}\\
\left.\frac d{d\theta}\{R_{\breve q}(\vw_\theta)-R_p(\vw_\theta)\}\right|_0
 &=a\rho\frac{\mathbb E\beta_*^2}{B_*}=a\rho B_*.
\label{eq:bridge-tightness-direct-slope}
\end{align}
The first line uses $F_b(a,0)=0$ and independence of $J$. For the second, the target margin is the source margin minus
$\rho b_\theta\sin\theta\,\beta_*/B_*$. Smooth finite sums give second-order remainders for both expressions.

Let $h_2(p)=-p\log p-(1-p)\log(1-p)$ and
$H=h_2(\pi_1)-h_2(\pi_2)>0$. The sign follows because
$1/2<\pi_1<\pi_2<1$ and entropy is strictly decreasing on that interval. If $L_x(\theta)$ is the target conditional risk at $X=x$, then
$L_x(0)=h_2(\pi_x)$ and $L_x$ is smooth. Thus
\begin{equation}
R_{q_\varepsilon}(\vw_\theta)-R_{\breve q}(\vw_\theta)
=\varepsilon[L_1(\theta)-L_2(\theta)]
=\varepsilon[H+O(\theta)].
\label{eq:bridge-tightness-bridge-slope}
\end{equation}
It is positive for small positive $\theta,\varepsilon$, so it equals the absolute bridge term. Combining gives
\begin{equation}
\begin{split}
R_{q_\varepsilon}(\vw_\theta)-R_p(\vv)
&=R_p(\vw_\theta)-R_p(\vv)
 +R_{\breve q}(\vw_\theta)-R_p(\vw_\theta)
 +C_{{\rm br},\varepsilon}(\vw_\theta)\\
&=a(c\mu\kappa+\rho B_*)\theta+H\varepsilon
  +O(\theta^2+\varepsilon\theta).
\end{split}
\label{eq:bridge-tightness-actual-expansion}
\end{equation}
All three leading coefficients are strictly positive.

\paragraph{The refined bound has the same first derivatives.}
Direct differentiation gives $\boldsymbol g_*=(0,-c\mu\kappa)^\top$. Also
\begin{equation}
-\boldsymbol g_*^\top\vh_\theta=ac\mu\kappa\sin\theta\cos\theta,
\quad \|\vh_\theta\|_2=a\sin\theta,
\quad \ell_\theta=\rho b_\theta\sin\theta
\label{eq:bridge-tightness-first-order-certificate}
\end{equation}
for small positive $\theta$. Since $B_p(\vw_\theta)=B_*+O(\theta)$, the source and direct terms in~\eqref{eq:bridge-tightness-certificate} equal
$ac\mu\kappa\theta+O(\theta^2)$ and
$a\rho B_*\theta+O(\theta^2)$, respectively. The exact bridge term appears unchanged in the actual risk and the bound. Their difference is therefore $O(\theta^2)$ and is nonnegative by Lemma~\ref{lem:grit-local-certificate}.

Finally apply Lemma~\ref{lem:grit-adverse-angle} and set $\varepsilon=\delta_k$. On its event, $\theta_k$ is between $\delta_k/6$ and $16\delta_k$, so all local expansions apply for small enough $\delta_k$. Each contribution is bounded above and below by a positive constant times $\delta_k$, the slack is $O(\delta_k^2)$, and the actual excess in the ratio denominator is at least a positive constant times $\delta_k$. This proves~\eqref{eq:bridge-tightness-ratio} uniformly on the stated angle interval.
\end{proof}

Setting $\varepsilon=0$ instead gives one target fixed for every pair count, with zero bridge term, while preserving the ratio-one conclusion for the refined source and direct-shift comparison. With $\varepsilon=\delta_k$, the bridge, source, and direct effects are simultaneously positive. The simpler source and direct terms remain jointly sharp up to constants because
\begin{equation}
\frac{R_p(\vw_\theta)-R_p(\vv)}{L_z\|\vv\|_2\sin\theta}
\longrightarrow\frac{c\mu\kappa}{L_z}>0,
\qquad
\frac{R_{\breve q}(\vw_\theta)-R_p(\vw_\theta)}
{\rho\|\vw_\theta\|_2\sin\theta}\longrightarrow B_*>0.
\label{eq:bridge-simple-certificate-ratios}
\end{equation}
The ratio-one result concerns the complete refined comparison at the realized fitted projector, with $r=r_*=1$. Replacing its angle by a concentration radius can add slack.

\subsection{Quadratic cost of removing additional directions}
\label{app:quadratic-extra-removal}
Source stationarity also permits a sharper treatment of extra-direction removal without a Hessian lower bound. Let $P=P_{\mathcal U}$, let $Q$ be a rank-$r_*$ orthogonal projector, and let $Q_r$ be an orthogonal projector containing its range. Define
\begin{equation}
\alpha=\|Q-P\|_{\rm op},\quad H=Q_r-Q,\quad
\vv=\vw_{\mathcal S}^*,\quad a=\|\vv\|_2,\quad
\delta=\|H\vv\|_2,\quad G_U=\|P\nabla R_p(\vv)\|_2.
\label{eq:extra-removal-quantities}
\end{equation}

\begin{proposition}[Quadratic extra-direction comparison]
\label{prop:quadratic-extra-removal}
Suppose the source has finite second feature moments, $\vv$ is an attained stable logistic optimum, and $\vw_r$ is an attained source optimum on $\ker Q_r$. Then
\begin{equation}
R_p(\vw_r)-R_p(\vv)
\le G_U\alpha a+\frac{L_z^2}{8}(\alpha^2a^2+\delta^2).
\label{eq:quadratic-extra-source}
\end{equation}
Under the bridge and conditional-displacement assumptions, this implies
\begin{equation}
R_q(\vw_r)\le R_p(\vv)+C_{\rm br}(\vw_r)
 +G_U\alpha a+\frac{L_z^2}{8}(\alpha^2a^2+\delta^2)
 +\rho_\Delta\alpha\|\vw_r\|_2.
\label{eq:quadratic-extra-target}
\end{equation}
\end{proposition}
\begin{proof}
Compare with the projected stable oracle using the logistic Taylor bound. Stable stationarity confines the linear term to unstable directions; nesting controls it by the estimation angle. Orthogonality separates the two squared removal costs.

Set $\vh=Q_r\vv$. By feasibility,~\eqref{eq:tightness-source-stationarity}, and~\eqref{eq:tightness-source-taylor},
\begin{align}
R_p(\vw_r)-R_p(\vv)
&\le-\nabla R_p(\vv)^\top\vh+\tfrac18\vh^\top\boldsymbol C_p\vh\nonumber\\
&\le G_U\|P Q_r\vv\|_2+\tfrac18L_z^2\|Q_r\vv\|_2^2.
\label{eq:quadratic-extra-comparison}
\end{align}
Because $P\vv=0$, $PQ_r\vv=-P(I-Q_r)\vv$. Nesting yields
$(I-Q_r)=(I-Q)(I-Q_r)$, so
$\|P(I-Q_r)\|_{\rm op}\le\|P(I-Q)\|_{\rm op}\le\alpha$.
Moreover, $Q$ and $H$ are orthogonal projectors with orthogonal ranges. Hence
\begin{equation}
\|PQ_r\vv\|_2\le\alpha a,\qquad
\|Q_r\vv\|_2^2=\|Q\vv\|_2^2+\|H\vv\|_2^2
\le\alpha^2a^2+\delta^2.
\end{equation}
Insert these into~\eqref{eq:quadratic-extra-comparison}. Finally
$\|P\vw_r\|_2\le\alpha\|\vw_r\|_2$ because $\vw_r\in\ker Q_r$; Theorem~\ref{thm:risk-decomposition} proves~\eqref{eq:quadratic-extra-target}.
\end{proof}

When $Q=P$, the source bound reduces to $L_z^2\delta^2/8$. Thus the linear cost $L_z\delta$ can be locally loose for pure extra-direction removal with bounded moments. The quadratic coefficient is attained to leading order by an actual noiseless GRIT example.

\paragraph{Attaining example.}
Take $d=3$, $\mathcal S=\operatorname{span}(\ve_1,\ve_2)$,
$\mathcal U=\operatorname{span}(\ve_3)$, and $q=p$ with
\begin{equation}
\rvz=(1,X,0)^\top,\quad\Pr(X=1)=\Pr(X=-1)=1/2,\quad
\Pr(\ry=1\mid X)=s(bX),\qquad b>0.
\label{eq:extra-removal-example}
\end{equation}
The stable optimum is uniquely $\vv=(0,b,0)^\top$, since its score is the conditional log odds and the stable feature vectors span the stable plane. Here
$\boldsymbol C_p=\operatorname{diag}(1,1,0)$ and $L_z=1$.
For every $k$, take noiseless clean differences $\vu_j=\sqrt g\,\ve_3$.
The observed pair covariance is $gP_{\ve_3}$, satisfying Assumption~\ref{ass:pairs} with $\sigma_\xi=0$.
Fix the zero-eigenvalue tie rule to select $\ve_2$ before $\ve_1$; it is specified before labels. Then $Q=P_{\ve_3}$ and $Q_2=Q+P_{\ve_2}$. Only the intercept remains, and its fitted optimum is $\vw_2=0$ because the marginal label probability is $1/2$. With $\delta=\|(Q_2-Q)\vv\|_2=b$,
\begin{align}
R_q(\vw_2)-R_p(\vv)
&=\log2-h_2(s(b))\nonumber\\
&=\int_0^b t\,s(t)(1-s(t))\,dt
=\frac{b^2}{8}+O(b^4)
=\frac{\delta^2}{8}+O(\delta^4).
\label{eq:extra-removal-sharp-constant}
\end{align}
The integral follows by differentiating the entropy expression; $s(t)(1-s(t))=1/4+O(t^2)$ gives the expansion. Thus the quadratic coefficient $1/8$ cannot be reduced uniformly. This example uses an explicit tie rule, since the leading rank-two eigenspace is not unique without one.
 \section{Predictive Stability and Sources of Pair Information}
\label{app:connections}
This appendix relates Assumption~\ref{ass:stability} to established invariance criteria, then gives sufficient feature-space conditions under which latent models, repeated observations, and augmentations supply pair information.

\subsection{Conditional invariance and common optimal prediction}
\label{app:a1-comparison}
Assumption~\ref{ass:stability} posits a single conditional label distribution $P_{\mathrm{inv}}$ on projected features, valid almost surely in each environment. The notation $P^{(e)}(y\mid\vz_{\mathcal S})=P_{\mathrm{inv}}(y\mid\vz_{\mathcal S})$ denotes equality of conditional distributions, without assuming a density or discrete labels.
This is the subspace analogue of the invariant feature-subset condition in \citet[Section~2, assumptions A1 and A1$'$]{rojascarulla2018invariant}.
It does not assert that $\mathcal S$ is unique or that it preserves all label information available in the original features.
The zero subspace, for example, satisfies Assumption~\ref{ass:stability} whenever label marginals agree. Predictions based on this projection use only the common label marginal.

Common conditional distributions and common optimal predictors are related through the loss and predictor class.
Ideal IRM requires one predictor to minimize each environment's risk on the chosen representation \citep[Section~3, Definition~3]{arjovsky2019invariant}.
Under squared loss, unrestricted optimal predictions determine conditional means but leave higher moments free.
For binary logistic loss with attained finite optima, unrestricted optimal logits determine conditional label probabilities.
Restrictions on the head change this relationship \citep[Section~2.1]{kamath2021does}.
In particular, even a common conditional distribution need not yield the same best misspecified linear predictor when the feature marginal changes.

\begin{samepage}
\begin{example}[A common conditional distribution with different optimal linear heads]
\label{ex:misspecified-head}
Let $Z\in\{-2,-1,1,2\}$ have a symmetric distribution with $\Pr_e(|Z|=2)=t_e\in(0,1)$, and set $Y=Z^3$ in every environment.
The conditional distribution of $Y$ given $Z$ is common, but the squared-loss optimal slope is
\begin{equation}
w_e=\frac{1+15t_e}{1+3t_e}.
\end{equation}
It differs for distinct values of $t_e$, with or without a fitted intercept.
\end{example}
\end{samepage}
\begin{proof}
\emph{Sketch.} Compute the population normal equation; symmetry removes the intercept.
Symmetry gives $\mathbb E_e Z=\mathbb E_e Y=0$, so the optimal intercept is zero.
Moreover, $\mathbb E_e Z^2=1+3t_e>0$ and $\mathbb E_e ZY=\mathbb E_e Z^4=1+15t_e$.
Differentiating the quadratic risk gives $w_e\mathbb E_e Z^2=\mathbb E_e ZY$.
The displayed ratio has derivative $12/(1+3t_e)^2>0$, proving the claim.
\end{proof}

The same distinction limits causal interpretations of Assumption~\ref{ass:stability}.
Invariant causal prediction uses invariance to identify candidate causal predictors under suitable intervention assumptions \citep{peters2016causal}.
Here Assumption~\ref{ass:stability} is a statistical condition on a selected projection, and pairs supply geometric information about its complement.
Neither conditional invariance nor observed pair agreement alone identifies a causal graph.

\subsection{A sufficient latent construction}
\label{app:latent-construction}
An orthogonal latent decomposition provides one transparent way to satisfy the pair model.
Let $A$ and $B$ have orthonormal columns with $A^\top B=0$ and $[A\ B]$ an orthogonal basis of $\mathbb R^d$.
Suppose features have the form $Z=AC+BV$, with a common conditional distribution of $Y$ given $C$ across environments.
The joint distribution of $(C,V)$ and the conditional distribution of $V$ given $(C,Y)$ may change.
Choose $\mathcal S=\operatorname{range}(A)$ and $\mathcal U=\operatorname{range}(B)$.

\begin{proposition}[Task-preserving repeated observations]
\label{prop:latent-pairs}
Under this construction Assumption~\ref{ass:stability} holds.
Two observations sharing $C$ have feature difference in $\mathcal U$.
For a finite collection of such pairs, its clean pair covariance has range $\mathcal U$ if and only if the latent differences span the column-coordinate space of $B$.
\end{proposition}
\begin{proof}
\emph{Sketch.} The stable projection recovers $C$, and differencing cancels it; the remaining covariance range is the span of the differences.
Since $P_{\mathcal S}Z=AC$ and $A^\top AC=C$, conditioning on $P_{\mathcal S}Z$ is equivalent to conditioning on $C$.
The common conditional distribution therefore gives Assumption~\ref{ass:stability}.
For shared-$C$ observations, $Z-Z'=B(V-V')\in\mathcal U$.
Let $D$ be the matrix with columns $V_j-V_j'$. The pair covariance is $BDD^\top B^\top/k$.
Its range is $\operatorname{range}(BD)$ because a Gram matrix and its factor have the same range.
Since $B$ is injective on its column-coordinate space, this range equals $\mathcal U$ precisely when $D$ has full row rank.
\end{proof}

Interventions that change $V$ while preserving $C$ supply one causal realization of this sufficient construction.
For a nonorthogonal mixing model, a stable linear readout may recover $C$, but its orthogonal complement must be checked against the actual pair-difference span.

\subsection{Multiple views, augmentations, and pooled collections}
Two camera views or repeated measurements can instantiate the shared-$C$ construction when the measured state determining the task remains fixed and only the viewing or measurement variable changes.
Even then, coverage depends on which changes are observed.
Repeated observations along one direction cannot identify a higher-dimensional unstable space, irrespective of the number of repetitions.
Changes to the task-relevant state instead challenge pair validity.

An augmentation $T$ yields an invariant pair with exact geometric compatibility if
$P_{\mathcal S}\phi(Tx)=P_{\mathcal S}\phi(x)$.
The difference then lies in $\ker P_{\mathcal S}=\mathcal U$.
This identity is a sufficient feature-space condition, stronger than simply assigning both endpoints the same class label.
Augmented views in representation learning \citep{chen2020simple,grill2020bootstrap} motivate the use of related observations but need not satisfy this condition for a downstream prediction task.
Residual geometric differences outside $\mathcal U$ require a noise or misspecification model, not merely the label ``augmentation.''

For several compatible sources, let $\widetilde{\boldsymbol\Sigma}^{(a)}$ be the pair covariance of source $a$, based on $k_a$ pairs, and let $k=\sum_a k_a$.
Pooling gives the exact identity
\begin{equation}
\widetilde{\boldsymbol\Sigma}_k=\sum_a\frac{k_a}{k}\widetilde{\boldsymbol\Sigma}^{(a)}.
\end{equation}
Complementary clean spans can enlarge coverage, but source frequency and difference magnitude also change spectral ordering.
Biased residuals or incompatible invariance relations can remove task information.
The theoretical noise assumptions must therefore be checked for the pooled collection rather than inferred from its sources' names.
 \section{Experimental Protocols and Additional Results}
\label{app:experiments}
This appendix describes the core comparison on ColoredMNIST-INV and Waterbirds-INV, our variants augmented with invariant-pair information. Appendix~\ref{app:cg-sweeps}
documents the controlled CueMNIST and Gaussian studies.

\subsection{Data, features, and information access}
\label{app:experimental-data}
\paragraph{ColoredMNIST-INV.}
Following the task construction of \citet{arjovsky2019invariant}, the clean binary
target indicates whether an MNIST digit is at least five. Independently flipping
this target with probability $0.25$ defines the supervised target.
Color is sampled relative to the supervised target, with disagreement probabilities
$0.1$ and $0.2$ in the two training environments and $0.9$ at test.
The official training sources are partitioned into two disjoint sets of 25,000
digits and a validation set of 10,000 digits. The 10,000 official test digits are
reserved for final evaluation. Partitioning precedes rendering.
Validation uses three renderings of its sources, at color disagreement rates
$0.1$, $0.2$, and $0.3$. These are repeated measurements of 10,000 sources.

The 256 pairs are sampled without replacement from training sources and show the
same digit in both colors. Both endpoints preserve its noisy target as well as
its underlying digit. Their canonical orientation is red minus green.
The auxiliary endpoints enter pair estimation or penalties, but are not added as
supervised records. An anchor may also occur in the labeled training set.
Source IDs are available to V-REx, IRMv1, Fishr, and RDM; training colors are not.
GroupDRO and LISA use four groups defined by noisy target label crossed with
training source. SD uses no source IDs.
Consequently, comparing Vanilla, +PC, and +GRIT within a base objective holds its annotation
access fixed.

\paragraph{Waterbirds-INV.}
The released Waterbirds construction combines bird foregrounds with land or water
backgrounds \citep{sagawa2020distributionally}. Waterbirds-INV follows this compositing
methodology to construct invariant pairs while retaining the training-set size.
Our controlled-pair training set preserves its 4,795-record size and official
validation and test splits. We select 184 landbirds on land and 56 waterbirds on
water, create an opposite-background version of each, and replace the 240
unrelated minority training records with these generated endpoints.
The training set therefore contains 4,315 unpaired records and both endpoints of
240 pairs. Every method uses the same labeled dataset; pair-aware variants
add only the relations between endpoints.
The group counts are 3,498 landbirds on land, 184 landbirds on water,
56 waterbirds on land, and 1,057 waterbirds on water.
Validation and test contain 1,199 and 5,794 records, respectively.

Each pair reuses the source bird and compositing geometry while changing the
background category. The retained endpoint is a released JPEG and the generated
endpoint is a PNG composite, so the construction does not assert byte-identical
foregrounds after image encoding. Pair orientation is land minus water.
The pair construction is available in the public code: \url{https://github.com/inouye-lab/GRIT}.
V-REx, IRMv1, Fishr, and RDM use each training example's background as its environment
label. GroupDRO and LISA use four bird-label--background groups.
ERM and SD use bird labels; their pair-based variants additionally receive the
240 relations. These are different forms of information, so comparisons across
base objectives are not comparisons under identical supervision.

\paragraph{Frozen features and predictors.}
All matrix variants use deterministic evaluation preprocessing and the official
CLIP ViT-B/32 image encoder \citep{radford2021learning}.
Its unnormalized 512-dimensional output defines $\vz=\phi(\vx)$.
Vanilla, GRIT, and prediction consistency fit a two-class affine classifier.

\subsection{Pair-based variants, optimization, and selection}
\label{app:experimental-selection}
\paragraph{Computing the GRIT projection.}
\label{app:grit-computation}
Assign mass $1/(2k)$ to each difference $\boldsymbol\delta_j$ and its opposite
$-\boldsymbol\delta_j$. This symmetric distribution has mean zero and covariance
\[
\frac1{2k}\sum_j\left[
\boldsymbol\delta_j\boldsymbol\delta_j^\top
+(-\boldsymbol\delta_j)(-\boldsymbol\delta_j)^\top\right]
=\frac1k\sum_j\boldsymbol\delta_j\boldsymbol\delta_j^\top.
\]
Symmetrization preserves the pair covariance and the number of pairs, $k$.

For computation, collect the differences in $\boldsymbol D\in\mathbb R^{d\times k}$. Since $\widetilde{\boldsymbol\Sigma}_k=\boldsymbol D\boldsymbol D^\top/k$, a singular-value decomposition supplies the same directions without forming the covariance. With $\widehat{\boldsymbol U}_r=[\widehat{\vu}_1,\ldots,\widehat{\vu}_r]$, project as $\vz-\widehat{\boldsymbol U}_r(\widehat{\boldsymbol U}_r^\top\vz)$. The projection stays fixed during fitting and at test time. The experimental protocol caps removal at numerical rank; the mathematical definition allows orthonormal completion.

\paragraph{The three variants.}
Let $\boldsymbol W\vz+\boldsymbol b$ denote the two class logits and define
$\boldsymbol\delta_j=\phi(\vx_j)-\phi(\vx'_j)$ for pair $j$.
Vanilla fits the base objective on original features.
GRIT uses $(\boldsymbol I-P_{k,r})\vz$ throughout training and inference.
The implementation computes a deterministic full SVD of the original oriented
difference matrix without subtracting its mean. This estimates the pair covariance
defined in the method section. The SVD uses float64 and a relative singular-value
tolerance of $10^{-12}$; the effective removal rank is capped by numerical rank.

Prediction consistency adds the penalty
\begin{equation}
\label{eq:experimental-pair-penalties}
 \mathcal L_{\mathrm{pred}}
 =\frac{1}{k}\sum_{j=1}^k\|\boldsymbol W\boldsymbol\delta_j\|_2^2.
\end{equation}
The sum over coordinates is not divided by the number of logits.
Biases cancel. The whole pair bank is evaluated at each update, including
base-objective warm-up. The two-logit experimental penalty is related to, but
not identical to, the scalar-score penalty analyzed in the theory.
For logits with rows $\boldsymbol w_0,\boldsymbol w_1$, it penalizes both their
contrast and their common component. Accuracy and cross-entropy depend only on
the contrast. The scalar guarantee therefore does not directly certify the full
experimental training procedure.

\paragraph{Connection to paired consistency methods.}
CoRe penalizes within-identity prediction variability and applies this penalty
to logits for classification \citep[Sec.~3.2]{heinzedeml2021conditional}.
For a two-member group, its variance penalty with exponent one is proportional
to squared paired logit differences. MatchDG first learns cross-domain matches;
its reported second-phase implementation penalizes squared distances between
their logits \citep[Sec.~4 and App.~C.1]{mahajan2021domain}.
SelfReg includes logit penalties for randomly paired same-class examples, with
additional learned projection and interpolation components
\citep[Sec.~3]{kim2021selfreg}.
Those pairs need not share an underlying object.
Our PC variant is a representative output-consistency penalty applied to
supplied invariant pairs and frozen CLIP features; it does not reproduce the
complete acquisition and training procedures of these methods.

\paragraph{Base objectives and optimization.}
V-REx penalizes variation in source risks \citep{krueger2021out}; IRMv1 uses the
IRM gradient penalty \citep{arjovsky2019invariant}; Fishr matches gradient
variances across environments \citep{rame2022fishr}.
SD adds the mean squared raw logits, averaged over examples and classes
\citep{pezeshki2021gradient}. RDM matches the worst environment's distribution
of individual cross-entropy losses to their pooled distribution using a
Gaussian-kernel maximum mean discrepancy \citep{nguyen2024domain}.
Its implementation also regularizes the two loss variances, with coefficient
$0.004$. GroupDRO reweights supervised risks over four label--environment groups;
its adversarial weights use supervised cross-entropy only
\citep{sagawa2020distributionally}. Feature-space LISA selectively interpolates
same-label examples across environments or different labels within an environment,
using Beta$(2,2)$ mixing weights \citep{yao2022improving}.
GRIT projects the resulting mixed features; prediction consistency uses the
original invariant pair bank independently of the mixing partners.
Their supervised statistics are computed separately from the pair penalties.
For V-REx and IRMv1, a base penalty coefficient $\alpha$ scales the combined
objective as
$(\mathcal L_{\mathrm{risk}}+\alpha\mathcal L_{\mathrm{base}}
+\lambda\mathcal L_{\mathrm{pair}})/\max(1,\alpha)$.
Fishr computes per-example gradients of the classifier.
Its moving-average coefficient is $0.95$, with 1,500 warm-up updates and an Adam
reset when the base penalty activates.

Every variant uses Adam with batch size 256 and at most 40 epochs on ColoredMNIST-INV
or 100 on Waterbirds-INV. The shared learning-rate grid is
$\{10^{-4},3\cdot10^{-4},10^{-3},3\cdot10^{-3}\}$ and weight-decay grid is
$\{0,10^{-5},10^{-4},10^{-3}\}$.
The matrix searches removal ranks $\{2,8,16,24\}$ and prediction-consistency
strengths $\{0.01,0.1,1,10\}$.
V-REx and Fishr use base coefficients $\{10,100,1000,10000\}$;
IRMv1 uses $\{100,1000,10000,100000\}$.
V-REx and IRMv1 activate their penalties at updates 100 and 190, respectively.
GroupDRO searches adversarial step sizes $\{0.001,0.01,0.1\}$; SD searches
$\{10^{-5},10^{-4},10^{-3},10^{-2},0.1\}$; RDM searches $\{0.1,1,5,10\}$;
LISA searches mixing probabilities $\{0,0.1,0.3,0.5,0.7,0.9,1\}$.
SD acts from the first update. RDM uses the same 1,500-update warm-up and Adam
reset convention as Fishr, with its pair penalty active throughout.

\paragraph{Validation and uncertainty.}
ColoredMNIST-INV's primary selector maximizes the minimum accuracy across its three
validation environments. This uses a held-out color environment in addition to
source-like validation; a separate source-only selector uses the first two.
Waterbirds-INV maximizes validation worst-group accuracy over the four combinations
of bird label and background. Thus validation group metadata is available to all
methods through selection. Test metrics do not select ordinary runs.
Each candidate uses tuning seeds 1101--1103. The top three candidates receive
confirmation seeds 2101--2102, and their combined five-seed validation means select
one configuration. Final evaluation uses fresh seeds 3101--3110 shared across
methods. Each final run selects its checkpoint with validation under that frozen
configuration. Restoring this checkpoint precedes final test evaluation.

We report means and sample standard deviations over final seeds, and 95\%
Student-$t$ intervals for selected paired differences matched by seed. For the ten
differences $d_s$, these intervals are $\bar d\pm t_{0.975,9}s_d/\sqrt{10}$.
Such intervals summarize training variability conditional on the dataset and pair
bank; they do not quantify uncertainty over new pair collections or populations.
Table~\ref{tab:pair-matrix-expanded} reports sample standard deviations over
ten final seeds. These are not standard errors; the standard error is
$s/\sqrt{10}$.

\subsection{Expanded core results}
\label{app:result-verification}
Table~\ref{tab:pair-matrix-expanded} expands the main comparison with sample
standard deviations over the ten final training seeds.

\begin{table}[htbp]
\caption{Expanded core comparison: means and sample standard deviations (s.d.)
over ten final training seeds. +PC and +GRIT add prediction consistency
and GRIT to each objective.}
\label{tab:pair-matrix-expanded}
\centering\small
\begin{tabular}{lrrrrrr}
\toprule
& \multicolumn{2}{c}{Vanilla} & \multicolumn{2}{c}{+PC}
& \multicolumn{2}{c}{+GRIT}\\
\cmidrule(lr){2-3}\cmidrule(lr){4-5}\cmidrule(lr){6-7}
Objective & Mean & s.d. & Mean & s.d. & Mean & s.d.\\
\midrule
\multicolumn{7}{l}{\textit{ColoredMNIST-INV (CLIP feat.)}}\\
ERM & 17.02 & 2.68 & 57.16 & 0.47 & 67.89 & 0.55\\
SD & 17.02 & 2.68 & 56.68 & 0.31 & 67.77 & 0.62\\
RDM & 39.95 & 3.47 & 59.55 & 0.89 & 67.60 & 0.58\\
IRMv1 & 59.89 & 2.71 & 61.28 & 1.19 & 70.27 & 0.41\\
V-REx & 54.83 & 3.45 & 59.98 & 1.52 & 68.20 & 0.13\\
Fishr & 65.28 & 1.74 & 65.33 & 1.61 & 68.91 & 0.18\\
GroupDRO & 30.90 & 2.46 & 61.51 & 0.51 & 68.01 & 0.43\\
LISA & 36.68 & 3.92 & 61.62 & 0.68 & 67.80 & 0.22\\
\midrule
\multicolumn{7}{l}{\textit{Waterbirds-INV (CLIP feat.)}}\\
ERM & 52.69 & 0.49 & 66.84 & 1.84 & 70.09 & 1.19\\
SD & 52.06 & 1.12 & 66.82 & 1.82 & 70.09 & 1.19\\
RDM & 52.80 & 1.65 & 78.59 & 0.91 & 76.29 & 1.56\\
IRMv1 & 47.78 & 1.26 & 77.86 & 1.04 & 78.78 & 0.82\\
V-REx & 48.27 & 0.67 & 78.09 & 1.39 & 80.99 & 0.94\\
Fishr & 47.16 & 0.80 & 76.72 & 0.90 & 80.84 & 0.61\\
GroupDRO & 79.83 & 1.75 & 81.66 & 0.80 & 81.91 & 0.54\\
LISA & 79.65 & 1.28 & 80.77 & 0.89 & 82.30 & 0.73\\
\bottomrule
\end{tabular}
\end{table}
 
\paragraph{Paired contrasts.}
For four-group GroupDRO on Waterbirds-INV, the reported GRIT gains and paired
95\% Student-$t$ intervals are $2.08\,[1.04,3.12]$ against vanilla,
$0.25\,[-0.44,0.95]$ against prediction consistency.
The latter interval includes zero; the highest mean does not establish a
separation from prediction consistency.
For ColoredMNIST-INV four-group GroupDRO, the corresponding intervals are
$37.11\,[35.35,38.86]$ and $6.50\,[5.99,7.02]$.
These intervals measure training-seed variability conditional on the fixed inputs.

The main table ranks displayed means: bold and italic identify the top two
within each objective and dataset; solid and dotted boxes identify the top two
among all entries within each dataset.
Unrounded paired contrasts can differ from differences of the displayed means.
No statistical significance is inferred from these rankings.
\section{UrbanCars: combining background and object invariances}
\label{app:urbancars}

\subsection{Data, pair types, and allocation}
UrbanCars composes car, background, and co-occurring-object images to study
multiple spurious correlations \citep{li2023whac}.
We classify urban versus country car body types with frozen, unnormalized,
512-dimensional CLIP ViT-B/32 features.
Background and object categories each agree with the label with probability 0.95
in training. There are 8,000 training, 1,000 validation, and 1,000 test images.
The eight groups cross label, background alignment, and object alignment.
Each label has 3,610, 190, 190, and 10 training images in the both-aligned,
background-only-misaligned, object-only-misaligned, and both-misaligned groups.
Validation and test each contain 125 examples per group.
We use the standard compositing construction.

Each invariant pair of the first type preserves the car and co-object and replaces the background
with one drawn from the opposite category's training pool.
Each invariant pair of the second type preserves the car and background and replaces the co-object
analogously. Neither pair endpoint supplies an additional supervised training
record. Labeled examples and features are shared across methods and allocations.
Construction and partner sampling use seed 1729.
We reserve 1,024 training anchors for energy calibration, 2,048 for reference
geometry, and three disjoint banks of 1,024 anchors. Every anchor has both partner
types. An allocation uses the first 256 anchors of a bank; a fixed random ordering
chooses which $n_O$ receive object partners and which $n_B=256-n_O$ receive
background partners. Thus changing allocation preserves the anchors and couples
the collections within each bank. Both partner types use the same rendering and
feature pipeline.

\paragraph{Equal count and equal energy.}
The main comparison uses the two single-type endpoints and the equal-count
allocation $(n_B,n_O)=(128,128)$.
The full experiment also includes $(114,142)$ and $(54,202)$.
For calibration, let $\bar E_B$ and $\bar E_O$ be the average squared Euclidean
lengths of background and object differences in CLIP feature space.
Their values are 61.0 and 16.2, with ratio 3.76.
A target object-energy share $s$ determines
\[
n_O=\operatorname{round}\!\left(
\frac{256s\bar E_B}{s\bar E_B+(1-s)\bar E_O}\right),
\qquad n_B=256-n_O.
\]
The frozen counts are 142 for $s=0.25$ and 202 for $s=0.5$.
The latter approximately equalizes total background and object difference energy,
so it contains more pairs that vary the object because each contributes less energy on average.
Equal count instead gives approximately 21\% of its energy to object differences.
These allocations change counts, not image intensity or pair strength; equal
energy is a diagnostic choice, not an optimality condition for GRIT.

\subsection{Training, validation, and metrics}
All methods fit a two-class linear probe using Adam, batch size 256, for
100 epochs. GRIT searches removal ranks $\{1,2,4,8,16,32,64,128\}$, capped at the
numerical rank of the pair matrix. Prediction consistency uses
$\alpha\,\operatorname{tr}(W\widetilde{\boldsymbol\Sigma}_kW^\top)/
\operatorname{tr}(\widetilde{\boldsymbol\Sigma}_k)$,
where $W$ contains the two logit coefficient vectors, and searches
$\alpha\in\{10^{-3},10^{-2},\ldots,10^6\}$.
Normalization gives a common strength grid across collections with different
average difference energies.

Before the collection comparison, a validation optimizer search freezes one
learning rate and weight decay per method.
For GRIT and PC, the criterion averages each collection's best validation score
equally over the background-only, 142-object, 202-object, and object-only
collections on bank 0. This chooses settings that perform well across those
collections; each allocation does not receive a separate optimizer search.
The learning-rate grid was extended on validation to
$\{10^{-5},3\cdot10^{-5},10^{-4},3\cdot10^{-4},10^{-3},3\cdot10^{-3},10^{-2}\}$,
with weight decay in $\{0,10^{-4}\}$.
Frozen learning rates are $10^{-5}$ for ERM, $10^{-2}$ for PC, and $10^{-3}$ for
GRIT; weight decay is zero for ERM and $10^{-4}$ for PC and GRIT.
These optimizer choices are shared across both validation selectors.
The ERM learning rate is the grid's lower edge; PC's is the upper edge, on a
validation plateau after grid extension.

\paragraph{Validation selects hyperparameters and checkpoints.}
The primary selector maximizes the minimum validation accuracy over all eight
groups. It therefore uses group metadata and examples of every shifted group.
The secondary selector maximizes the same validation groups' accuracy weighted
by their training frequencies: within each label the weights are
$0.9025,0.0475,0.0475,0.0025$ in the order above, and labels receive equal weight.
Both selectors see the same balanced validation examples; only their objectives
differ.
For each method, allocation, bank, and selector, candidates use tuning seeds
1101--1103. The top three receive confirmation seeds 2101--2102, and their
five-seed validation means choose the configuration. Lower rank or penalty
strength breaks ties before candidate identity.
Final fits use seeds 3101--3110 and select their own checkpoints by the specified
validation criterion under the frozen configuration. Reported test metrics use
the validation-selected checkpoints. All selections are fixed before test data are used.

\paragraph{Test metrics.}
B-shift accuracy averages the two labels' groups where background is misaligned
and object aligned; O-shift reverses these roles. Both-shift averages the two
labels' groups where both are misaligned. Worst-group accuracy takes the minimum
of all eight group accuracies \emph{for each fitted model}, before averaging
across runs. Training-weighted test accuracy reweights the same eight test groups
by the training probabilities above. It is not accuracy only on the both-aligned
groups, and is distinct from choosing a checkpoint with the training-weighted
validation selector.

\subsection{Complete allocation results and selection sensitivity}
Tables~\ref{tab:urbancars-full-primary} and~\ref{tab:urbancars-full-secondary}
retain all five allocations and five metrics under each validation selector.
The primary selector favors collections covering both factors: GRIT's equal-count
worst-group mean is 71.9\%, compared with 56.7\% and 43.8\% for the single-type
collections. The 142-object collection is similar at 71.7\%.
The equal-energy collection allocates more pairs to objects and raises O-shift
accuracy to 87.5\%, but its worst-group accuracy is lower at 67.9\%.
GRIT has higher worst-group means than PC throughout, while PC or ERM wins some
individual-shift metrics.

Weighting validation toward the training distribution changes both the magnitude
and ordering of these results. Equal-count GRIT reaches 51.7\% worst-group
accuracy, above background-only 49.3\% and object-only 44.0\%, but by smaller
margins than under worst-group validation.
The equal-energy collection drops to 42.8\%, below both endpoints.
Under training-weighted validation, PC selects its smallest penalty
\((\alpha = 10^{-3})\) in 12 of 15 bank-level fits
(balanced bank 1 selects \(\alpha = 0.1\); object-only banks 0 and 1
select \(\alpha = 0.01\)), with worst-group means between
29.9\% and 32.5\%.
Under primary validation, GRIT usually selects checkpoints at epochs 1--7,
except when only the object varies; secondary validation typically selects epochs
35--100. Each fit runs for 100 epochs; validation selects which saved weights
to retain.
Per-epoch test curves are diagnostics and do not select checkpoints.
Thus the observed benefit combines pair information with a validation objective
that rewards performance on the shifted groups.

\begin{table}[!htbp]
\centering
\caption{Mixed pair types improve worst-group accuracy under worst-group validation. Test accuracy (\%), mean $\pm$ sample SD, for all five allocations, sorted by $n_O$.
PC denotes prediction consistency. ``Train-weighted'' reweights test groups by
training frequencies; it is a test metric, distinct from the validation selector.
ERM uses no pairs and repeats for comparison. Bold marks the highest mean for
each allocation and metric.}
\label{tab:urbancars-full-primary}
\small
\setlength{\tabcolsep}{4pt}
\begin{tabular}{rrlrrrrr}
\toprule
$n_B$ & $n_O$ & Method & B shift & O shift & Both shift & Worst group & Train-weighted\\
\midrule
256 & 0 & ERM & 45.4 $\pm$ 5.4 & \textbf{76.0} $\pm$ 11.7 & 28.5 $\pm$ 12.4 & 20.4 $\pm$ 9.5 & 83.1 $\pm$ 14.9\\
 &  & PC & 85.7 $\pm$ 1.4 & 65.1 $\pm$ 0.9 & 55.1 $\pm$ 1.0 & 52.9 $\pm$ 1.6 & 90.2 $\pm$ 0.5\\
 &  & GRIT & \textbf{88.5} $\pm$ 0.6 & 68.4 $\pm$ 0.8 & \textbf{60.3} $\pm$ 1.9 & \textbf{56.7} $\pm$ 1.5 & \textbf{92.3} $\pm$ 0.8\\
\addlinespace[2pt]
128 & 128 & ERM & 45.4 $\pm$ 5.4 & 76.0 $\pm$ 11.7 & 28.5 $\pm$ 12.4 & 20.4 $\pm$ 9.5 & 83.1 $\pm$ 14.9\\
 &  & PC & 76.0 $\pm$ 0.8 & 84.4 $\pm$ 0.5 & 71.7 $\pm$ 1.3 & 68.3 $\pm$ 1.3 & 87.5 $\pm$ 0.6\\
 &  & GRIT & \textbf{79.1} $\pm$ 0.4 & \textbf{85.1} $\pm$ 0.8 & \textbf{74.8} $\pm$ 0.9 & \textbf{71.9} $\pm$ 0.9 & \textbf{89.3} $\pm$ 0.9\\
\addlinespace[2pt]
114 & 142 & ERM & 45.4 $\pm$ 5.4 & 76.0 $\pm$ 11.7 & 28.5 $\pm$ 12.4 & 20.4 $\pm$ 9.5 & 83.1 $\pm$ 14.9\\
 &  & PC & 75.9 $\pm$ 0.9 & \textbf{86.0} $\pm$ 0.5 & 70.7 $\pm$ 1.0 & 67.9 $\pm$ 1.4 & \textbf{89.9} $\pm$ 0.7\\
 &  & GRIT & \textbf{78.2} $\pm$ 0.4 & 85.2 $\pm$ 0.3 & \textbf{74.6} $\pm$ 0.8 & \textbf{71.7} $\pm$ 1.0 & 89.1 $\pm$ 0.8\\
\addlinespace[2pt]
54 & 202 & ERM & 45.4 $\pm$ 5.4 & 76.0 $\pm$ 11.7 & 28.5 $\pm$ 12.4 & 20.4 $\pm$ 9.5 & 83.1 $\pm$ 14.9\\
 &  & PC & 68.8 $\pm$ 1.1 & 85.3 $\pm$ 1.3 & 65.7 $\pm$ 0.7 & 63.3 $\pm$ 0.6 & 87.7 $\pm$ 0.8\\
 &  & GRIT & \textbf{73.7} $\pm$ 1.1 & \textbf{87.5} $\pm$ 0.7 & \textbf{70.4} $\pm$ 1.1 & \textbf{67.9} $\pm$ 1.6 & \textbf{89.5} $\pm$ 0.7\\
\addlinespace[2pt]
0 & 256 & ERM & 45.4 $\pm$ 5.4 & 76.0 $\pm$ 11.7 & 28.5 $\pm$ 12.4 & 20.4 $\pm$ 9.5 & 83.1 $\pm$ 14.9\\
 &  & PC & \textbf{54.7} $\pm$ 3.2 & 80.2 $\pm$ 3.1 & \textbf{46.9} $\pm$ 2.8 & 42.6 $\pm$ 2.4 & 82.4 $\pm$ 2.4\\
 &  & GRIT & 53.9 $\pm$ 0.4 & \textbf{94.1} $\pm$ 0.4 & 45.5 $\pm$ 0.7 & \textbf{43.8} $\pm$ 0.8 & \textbf{95.5} $\pm$ 0.1\\
\bottomrule
\end{tabular}
\end{table}

\begin{table}[!htbp]
\centering
\caption{Training-weighted validation reduces the benefit of mixing pair types and changes the collection ordering. Test accuracy (\%), mean $\pm$ sample SD, for all five allocations, sorted by $n_O$.
PC denotes prediction consistency. ``Train-weighted'' reweights test groups by
training frequencies; it is a test metric, distinct from the validation selector.
ERM uses no pairs and repeats for comparison. Bold marks the highest mean for
each allocation and metric.}
\label{tab:urbancars-full-secondary}
\small
\setlength{\tabcolsep}{4pt}
\begin{tabular}{rrlrrrrr}
\toprule
$n_B$ & $n_O$ & Method & B shift & O shift & Both shift & Worst group & Train-weighted\\
\midrule
256 & 0 & ERM & 46.4 $\pm$ 2.9 & \textbf{86.5} $\pm$ 1.7 & 16.4 $\pm$ 1.8 & 15.3 $\pm$ 1.5 & 94.6 $\pm$ 0.7\\
 &  & PC & 70.8 $\pm$ 1.4 & 81.7 $\pm$ 0.6 & 32.0 $\pm$ 1.8 & 30.0 $\pm$ 2.9 & \textbf{97.2} $\pm$ 0.1\\
 &  & GRIT & \textbf{87.3} $\pm$ 0.6 & 69.0 $\pm$ 0.7 & \textbf{51.8} $\pm$ 0.4 & \textbf{49.3} $\pm$ 1.3 & 95.6 $\pm$ 0.2\\
\addlinespace[2pt]
128 & 128 & ERM & 46.4 $\pm$ 2.9 & \textbf{86.5} $\pm$ 1.7 & 16.4 $\pm$ 1.8 & 15.3 $\pm$ 1.5 & 94.6 $\pm$ 0.7\\
 &  & PC & 70.7 $\pm$ 1.4 & 81.8 $\pm$ 0.7 & 32.0 $\pm$ 1.9 & 29.9 $\pm$ 2.9 & \textbf{97.2} $\pm$ 0.1\\
 &  & GRIT & \textbf{80.8} $\pm$ 0.6 & 78.8 $\pm$ 0.2 & \textbf{53.2} $\pm$ 0.5 & \textbf{51.7} $\pm$ 0.8 & 95.5 $\pm$ 0.2\\
\addlinespace[2pt]
114 & 142 & ERM & 46.4 $\pm$ 2.9 & \textbf{86.5} $\pm$ 1.7 & 16.4 $\pm$ 1.8 & 15.3 $\pm$ 1.5 & 94.6 $\pm$ 0.7\\
 &  & PC & 70.7 $\pm$ 1.3 & 81.7 $\pm$ 0.7 & 32.0 $\pm$ 2.0 & 30.1 $\pm$ 2.9 & \textbf{97.2} $\pm$ 0.1\\
 &  & GRIT & \textbf{78.0} $\pm$ 0.2 & 80.9 $\pm$ 0.2 & \textbf{54.2} $\pm$ 0.4 & \textbf{52.0} $\pm$ 0.3 & 95.3 $\pm$ 0.2\\
\addlinespace[2pt]
54 & 202 & ERM & 46.4 $\pm$ 2.9 & 86.5 $\pm$ 1.7 & 16.4 $\pm$ 1.8 & 15.3 $\pm$ 1.5 & 94.6 $\pm$ 0.7\\
 &  & PC & \textbf{70.3} $\pm$ 1.1 & 82.9 $\pm$ 0.5 & 34.7 $\pm$ 1.3 & 32.5 $\pm$ 2.0 & \textbf{97.2} $\pm$ 0.1\\
 &  & GRIT & 54.9 $\pm$ 0.7 & \textbf{93.5} $\pm$ 0.6 & \textbf{45.6} $\pm$ 0.7 & \textbf{42.8} $\pm$ 1.8 & 95.2 $\pm$ 0.2\\
\addlinespace[2pt]
0 & 256 & ERM & 46.4 $\pm$ 2.9 & 86.5 $\pm$ 1.7 & 16.4 $\pm$ 1.8 & 15.3 $\pm$ 1.5 & 94.6 $\pm$ 0.7\\
 &  & PC & \textbf{70.5} $\pm$ 0.6 & 82.7 $\pm$ 0.5 & 34.0 $\pm$ 1.3 & 32.1 $\pm$ 2.3 & \textbf{97.1} $\pm$ 0.1\\
 &  & GRIT & 53.9 $\pm$ 0.7 & \textbf{93.8} $\pm$ 0.6 & \textbf{45.5} $\pm$ 0.3 & \textbf{44.0} $\pm$ 0.7 & 95.5 $\pm$ 0.1\\
\bottomrule
\end{tabular}
\end{table}

\paragraph{ERM optimizer control and large-pair reference.}
Table~\ref{tab:urbancars-controls} separates the primary ERM baseline from
plain ERM fitted with GRIT's optimizer ($10^{-3}$ learning rate, $10^{-4}$
weight decay). This second row is an optimizer control, not GRIT: it applies no
projection and uses no pairs. With primary selection, the two ERM worst-group
means are similar (20.4\% and 20.6\%), although their background-shift and
training-weighted accuracies differ substantially. GRIT remains above this
control in worst-group accuracy for every allocation under both selectors.
The reference projection uses 2,048 pairs that vary the background and 2,048 that vary the object, with rank
selected on validation. Its learning rate is $3\cdot10^{-4}$ and weight decay
zero. The larger pair budget makes it a diagnostic rather than a matched
comparison or guaranteed accuracy ceiling.

\begin{table}[!htbp]
\centering
\caption{Optimizer and pair-budget controls on UrbanCars. Accuracy (\%), mean $\pm$ sample SD
over ten final seeds. WG and TW identify worst-group and training-weighted
\emph{validation} selection; the final column is training-weighted \emph{test}
accuracy. Controls use no allocation-specific pair banks.}
\label{tab:urbancars-controls}
\small
\setlength{\tabcolsep}{4pt}
\begin{tabular}{llrrrrr}
\toprule
Validation & Method & B shift & O shift & Both shift & Worst group & Train-weighted\\
\midrule
WG & ERM & 45.4 $\pm$ 5.4 & 76.0 $\pm$ 11.7 & 28.5 $\pm$ 12.4 & 20.4 $\pm$ 9.5 & 83.1 $\pm$ 14.9\\
 & ERM, matched & 66.0 $\pm$ 4.0 & 82.5 $\pm$ 0.8 & 25.8 $\pm$ 4.3 & 20.6 $\pm$ 6.6 & 96.8 $\pm$ 0.5\\
 & Reference & 84.0 $\pm$ 1.2 & 83.8 $\pm$ 0.7 & 82.0 $\pm$ 0.4 & 77.4 $\pm$ 0.9 & 88.9 $\pm$ 0.9\\
\addlinespace[2pt]
TW & ERM & 46.4 $\pm$ 2.9 & 86.5 $\pm$ 1.7 & 16.4 $\pm$ 1.8 & 15.3 $\pm$ 1.5 & 94.6 $\pm$ 0.7\\
 & ERM, matched & 69.4 $\pm$ 1.1 & 81.5 $\pm$ 0.3 & 31.3 $\pm$ 0.8 & 27.4 $\pm$ 1.9 & 97.1 $\pm$ 0.0\\
 & Reference & 93.7 $\pm$ 0.3 & 66.4 $\pm$ 0.6 & 62.5 $\pm$ 0.6 & 61.4 $\pm$ 1.1 & 95.0 $\pm$ 0.2\\
\bottomrule
\end{tabular}
\end{table}

\clearpage
\subsection{Uncertainty and interpretation}
PC and GRIT means first average the ten final training seeds within each of the
three fixed banks, then average the banks equally. ERM and reference rows average
ten seeds without bank-specific refits.
For a paired contrast, the protocol pairs methods by final seed, averages the
bank-specific differences within each seed, and constructs a 95\% Student-$t$
interval over those ten differences.
This measures training variability conditional on the three banks and fixed test
examples; it does not establish variation across arbitrary new acquisitions.

Paired intervals clarify these comparisons.
Under primary selection, GRIT minus PC on pairs that vary only the object has worst-group
difference $1.3\,[-0.5,3.0]$ percentage points (mean and 95\% seed interval).
Its interval includes zero, so the higher mean does not establish a separation.
Comparing GRIT's 202-object and 142-object collections gives worst-group
difference $-3.8\,[-4.8,-2.9]$ despite an O-shift gain of
$2.3\,[1.8,2.9]$. The 202-object collection exceeds matched-optimizer ERM
in worst-group accuracy by $47.2\,[42.6,51.9]$.
Under training-weighted validation, that same collection falls below
background-only GRIT by $6.5$ points, with signed difference
$-6.5\,[-8.3,-4.8]$.
Contrasts of unrounded values can differ from subtraction of the displayed means.

Tables~\ref{tab:urbancars-full-primary}--\ref{tab:urbancars-controls} report means and sample standard deviations.
For PC and GRIT, the sample standard deviation is across
the ten bank-averaged seed scores, with denominator nine; for ERM it is across
the ten ordinary seed scores, as for the reference projection. The corresponding standard error is SD$/\sqrt{10}$.
Treating thirty bank-by-seed fits as independent, or recovering row SDs from
paired contrast intervals, would give a different or unidentified quantity.

The coverage interpretation is qualitative in this feature space. Reference
background and object covariances require 151 and 166 directions, respectively,
to explain 90\% of their energy, and their leading subspaces overlap strongly.
These empirical covariances are not ground-truth unstable subspaces.
Consequently, two semantic factors do not mean two orthogonal feature directions,
and these results do not measure the theorem's clean weakest eigenvalue or
establish its pair-noise assumptions. CueMNIST supplies the complementary
known-geometry experiment. The claim supported here is that combining the two
supplied relations improves worst-group means under the specified selection
procedure.

\subsection{Construction fidelity and component overlap}
The construction follows the UrbanCars generator's component pools, split rules,
image placement, and JPEG output: a $256\times256$ canvas with a centered car
and the co-object at the right edge. The upstream generator revision is
\texttt{c91b0970}. CLIP preprocessing retains the full canvas.
Sixteen car images required fallback masks: fourteen used reduced MaskFormer
confidence thresholds and two used GrabCut. The remaining masks and pool order
were unchanged. A reported comparison of 200 decoded images per split found
597 of 600 pixel-identical to the original generator with the same masks; the
remaining three differed by at most three intensity levels on a few pixels.

Component IDs are split-disjoint, but the construction audit records six pairs
of byte-identical files under different names in different upstream splits:
three Places backgrounds and three Stanford Cars photographs.
It reports that at most six of the 2,000 validation and test images share one
component with a training image, with no shared full composite.
The audit also records 908 COCO photos supplying distinct object instances to
different splits. These inherited overlaps were retained to preserve the
upstream construction. Thus component-ID separation should not be interpreted
as complete separation of source photographs or bytes.
 \section{CueMNIST and controlled pair-information sweeps}
\label{app:cg-sweeps}
\label{app:geometry-protocol}
\subsection{Independent sweeps of weakest signal, pair noise, and pair count}
\label{app:pair-information-sweeps}

These experiments vary the smallest clean pair-covariance eigenvalue, pair-noise level, and number of pairs separately.
CueMNIST, MNIST \citep{lecun1998gradient} with two independently controlled color cues, supplies the running example; a synthetic task provides a second controlled setting.
Both use $d=96$ fixed features and two orthonormal cue directions, collected in $\boldsymbol U=[\boldsymbol u_B,\boldsymbol u_O]$.
The known unstable subspace is $\mathcal U=\operatorname{span}(\boldsymbol U)$, with $r_*=2$, and $\mathcal S=\mathcal U^\perp$.
Task examples, target shifts, feature geometry, and labeled sample size remain fixed across all sweep conditions within a replicate.

\paragraph{CueMNIST task and fixed features.}
The binary label is one for an MNIST digit less than five and zero otherwise.
Each replicate uses 3,000 training identities and 1,000 disjoint validation identities from the official training pool, and 1,000 identities from the official test pool, with equal numbers of the two labels in every split.
A $40\times48$ float RGB canvas contains the unchanged $28\times28$ grayscale digit, a colored frame $B$, a separate colored header $O$, and a strip of 30 color patches.
With $t=2y-1$, the cue values are
\[
 b=2.5q_B+0.3\eta_B,\qquad o=2.5q_O+0.3\eta_O,
 \qquad \eta_B,\eta_O\stackrel{\rm iid}{\sim}\mathcal N(0,1),
\]
where the signs $q_B,q_O$ independently equal $t$ with source probability $0.95$.
The strip coordinates are independent $\mathcal N(0,0.2^2)$ variables.
Each scalar $v$ is rendered as color $(a(v),1-a(v),1/2)$, where $a(v)=(1+\tanh(v/4))/2$.
Regional extraction recovers $v$ from the mean red channel using $4\operatorname{arctanh}(2a-1)$, avoiding clipping the Gaussian cue values.

The extractor adaptively average-pools the digit to $8\times8$ values, subtracts the training-coordinate means, and applies one scalar normalization so the training digit vectors have root-mean-square norm five.
Concatenating the 64 digit values, two decoded cues, and 30 decoded strip values gives 96 coordinates.
A random orthogonal embedding per replicate hides the cue axes.
It is obtained by QR factorization of a standard-Gaussian $96\times96$ matrix,
with signs chosen so the diagonal of $R$ is positive.
Normalization uses training images only, and the feature map remains fixed during fitting.
Batch construction implements these regional operations analytically; a test verifies equality with extraction from rendered float canvases.
This semi-synthetic task uses the engineered feature map described above.
Gaussian pair noise, defined below, is added after feature extraction and need not correspond to a valid pixel-space perturbation.

Validation and test contain all four cue-alignment combinations for each identity, producing 4,000 views per split.
The four views share the digit, strip, and within-cue Gaussian variations.
Eight groups combine the binary label with the two cue-alignment indicators.
Both-cue-reversal accuracy averages the two label groups in which both cues oppose the label.
Single-cue shifts and source-weighted metrics are also retained.
For each label, the four joint source-group probabilities are $\tfrac12(0.9025,0.0475,0.0475,0.0025)$ for both aligned, only $B$ misaligned, only $O$ misaligned, and both misaligned, respectively.
Ten replicates redraw the split, embedding, nuisance values, and pair noise.
Their uncertainty describes construction variability on the fixed MNIST dataset; the four views of an identity are not treated as independent replicates.

\paragraph{Synthetic task.}
Each of twenty replicates draws mutually orthogonal $\boldsymbol S\in\mathbb R^{96\times4}$, $\boldsymbol u_B$, and $\boldsymbol u_O$, and an independent orthogonal matrix $\boldsymbol V$.
These directions and rotations use the same Gaussian QR construction.
For a uniform binary label $y$ and $t=2y-1$, examples follow
\begin{align*}
 \boldsymbol z
 &=\boldsymbol S\boldsymbol s+\boldsymbol u_Bb+\boldsymbol u_Oo
   +\boldsymbol V\operatorname{diag}(1,2^{-1/2},\ldots,96^{-1/2})\boldsymbol\eta,\\
 \boldsymbol s&=1.28t\,\boldsymbol1_4/2+\boldsymbol\eta_S,
 \qquad b=2.5q_B+\eta_B,
 \qquad o=2q_O+\eta_O.
\end{align*}
All displayed Gaussian variables are independent standard Gaussian variables of the required dimensions.
The cue signs have the same independent $0.95$ source alignment as in CueMNIST.
There are 3,000 labeled training examples, 160 validation examples per group, and 400 test examples per group.
The anisotropic Gaussian feature noise is part of the source generator; it is distinct from the isotropic acquisition noise added to pair differences.
The same group definitions and source weights are used for both tasks.

\paragraph{Pairs with an exact prescribed spectrum.}
Let $g$ denote the smallest clean eigenvalue and fix the largest eigenvalue at one.
For even $k$, the clean differences alternate between
\[
 \boldsymbol u_j=
 \begin{cases}
  \sqrt{2}\,\boldsymbol u_B,&j\text{ odd},\\
  \sqrt{2g}\,\boldsymbol u_O,&j\text{ even}.
 \end{cases}
 \qquad
 \boldsymbol\Sigma_k=\frac1k\sum_{j=1}^k\boldsymbol u_j\boldsymbol u_j^\top
 =\boldsymbol u_B\boldsymbol u_B^\top+g\boldsymbol u_O\boldsymbol u_O^\top.
\]
Consequently, the clean empirical eigenvalues on $\mathcal U$ are exactly $(1,g)$ for every tested budget.
The runner constructs these differences analytically and assigns them to recorded training anchors, cycling through the 3,000 training examples when necessary.
The ideal partner of anchor $\boldsymbol z_{a_j}$ is $\boldsymbol z_{a_j}-\boldsymbol u_j$.
It changes one cue while preserving the task content and all other coordinates.
The CueMNIST renderer and extractor realize these clean scalar cue interventions exactly; noisy image endpoints are not materialized.
No partner is added to the labeled training sample.

Observed differences satisfy $\boldsymbol\delta_j=\boldsymbol u_j+\sigma_\xi\boldsymbol\epsilon_j$, with independent $\boldsymbol\epsilon_j\sim\mathcal N(\boldsymbol0,\boldsymbol I_{96})$ independent of the task data.
A master array supplies the same Gaussian draws across conditions and nested prefixes across budgets.
Anchor indices and clean differences likewise have nested prefixes when $g$ is fixed.
Here $k$ counts independent noisy measurements, including repeated measurements at reused training anchors; it does not count distinct digit identities.
Conditioning on the designed clean signals leaves the $k$ acquisition noises independent, as required by the pair-noise model.
Neither clean nor observed differences are centered, and noisy differences are not renormalized.
The pair covariance is
$\widetilde{\boldsymbol\Sigma}_k=k^{-1}\sum_j\boldsymbol\delta_j\boldsymbol\delta_j^\top$,
equivalently the covariance obtained by assigning half of each pair's weight to each orientation.\footnote{Symmetrization leaves the matrix unchanged and does not increase the number of independent measurements, which remains $k$.}

\paragraph{Three primary sweeps.}
The primary sweeps share the reference $(g,\sigma_\xi,k)=(1,0.1,256)$.
The prespecified grids are
\begin{align*}
 g&\in\{1/64,1/32,1/16,1/8,1/4,1/2,3/4,7/8,1\},\\
 \sigma_\xi&\in\{0.025,0.05,0.075,0.1,0.15,0.2,0.3,0.45,0.6,0.8,1\},\\
 k&\in\{64,96,128,192,256,384,512,1024,2048,4096\}.
\end{align*}
Each sweep holds the other two quantities at their reference values.
Changing $g$ holds the other eigenvalue fixed, so clean trace is $1+g$; this sweep does not fix total clean energy.
Changing $k$ preserves both clean eigenvalues exactly.
The expected noisy trace is $1+g+96\sigma_\xi^2$.
A $g=0$ control at the reference noise and budget omits the second clean direction and lies outside the spanning assumption.
A $\sigma_\xi=0$ control at the reference spectrum and budget isolates exact geometry.
Both controls are reported separately and excluded from logarithmic fits.
Conditions shared across axes reuse the same fitted models.

\paragraph{Secondary geometry settings.}
A follow-up specified after inspecting the primary results adds two nearby
settings around each reference geometry curve.
The eigenvalue sweep uses $\sigma_\xi\in\{0.05,0.1,0.2\}$ at $k=256$;
the noise sweep uses $k\in\{128,256,512\}$ at $g=1$;
and the pair-count sweep uses $\sigma_\xi\in\{0.05,0.1,0.2\}$ at $g=1$.
All original x values, replicates, clean signals, and Gaussian noise streams
are retained, including nested noise prefixes across budgets.
The secondary settings reuse the original subspaces, clean signals, noise draws, and anchors.
This follow-up computes subspace error only and reuses the primary predictor
results.
The three top-row curves show these geometry settings; bottom-row accuracies
and the appendix's fitted-instability mirrors retain the primary defaults.
The secondary settings were fixed before their computation.

\paragraph{Predictor fitting and validation.}
For $f_i=\boldsymbol w^\top\boldsymbol z_i+c$, all methods minimize binary logistic empirical risk plus $0.1\|\boldsymbol w\|_2^2/2$, with an unpenalized intercept.
GRIT additionally imposes $P_{k,2}\boldsymbol w=0$, where $P_{k,2}$ removes the leading two pair-covariance directions.
The shared SVD implementation excludes singular values below $10^{-4}$ of the largest and caps the requested rank at numerical rank.
We fix $r=r_*=2$ to isolate estimation; this experiment does not select rank or assume that the correct rank is known in general applications.
ERM and projection onto the true $\mathcal U^\perp$ are fitted once per replicate as controls.

Quadratic prediction consistency (PC) adds
\[
 \frac{\alpha}{2\operatorname{tr}(\widetilde{\boldsymbol\Sigma}_k)}
 \boldsymbol w^\top\widetilde{\boldsymbol\Sigma}_k\boldsymbol w,
 \qquad
 \alpha\in\{0,10^{-3},10^{-2},10^{-1},1,10,10^2,10^3\}.
\]
The factor $1/2$ matches the sum of squared differences of two symmetric logits $(-f/2,f/2)$.
For comparison with the score-penalty coefficient in the theory, the effective coefficient is $\lambda=\alpha/[2\operatorname{tr}(\widetilde{\boldsymbol\Sigma}_k)]$.
We record both coefficients because changing the pair spectrum or noise changes the trace.
Each condition's PC strength maximizes worst-group validation accuracy, with ties resolved toward smaller $\alpha$.
The balanced validation groups expose cue misalignment to this selector.

Fitting uses deterministic float64 Newton steps, zero initialization, Armijo backtracking, at most 80 iterations, and a required gradient infinity norm no greater than $10^{-7}$ in the optimization coordinates (retained features and intercept).
Backtracking starts at step size one and halves it until the objective decrease
is at least $10^{-4}$ times the step size times the gradient--Newton-step inner
product. A step size below $10^{-12}$ or failure to meet the gradient tolerance
terminates the run as unconverged.
All selections are fixed before the replicate's test views are generated.
The selected PC model, fixed-rank GRIT, ERM, and true-subspace control are then evaluated without further selection.
The master seed is 20260927.

\paragraph{Geometry, instability, and uncertainty.}
For the effective coefficient in the original feature coordinates, define
\[
 e_{\rm sub}=\|(\boldsymbol I-P_{k,2})\boldsymbol U\|_{\rm op},
 \qquad
 s_{\rm rel}=\frac{\|\boldsymbol U^\top\boldsymbol w\|_2}{\|\boldsymbol w\|_2}.
\]
The first quantity is the largest missed unstable-direction component.
The second measures the fitted predictor's unstable sensitivity relative to its coefficient norm.
Subspace error is a GRIT diagnostic; PC does not estimate a removal subspace.
For a nonzero hard-projected predictor, $s_{\rm rel}\le e_{\rm sub}$ holds deterministically; the pipeline checks this inequality.
Normalization is undefined for a zero coefficient, and the report rejects coefficients with norm at most $10^{-12}$.
Raw sensitivity $\|\boldsymbol U^\top\boldsymbol w\|_2$, coefficient norm, logistic loss, and shifted accuracy are retained alongside the normalized quantities.
The risk decomposition uses raw sensitivity, so the normalized plot must be interpreted together with the coefficient norm.

Curves display replicate means and sample standard deviations.
Endpoint contrasts use matched replicate differences with 5,000 percentile-bootstrap resamples and base seed 260925.
The pair-count reference curve is proportional to $k^{-1/2}$ and anchored to mean subspace error at the prespecified $k=256$; its vertical normalization is descriptive.
Log--log slopes are fit separately per replicate using $k\ge512$, giving the four points $512,1024,2048,4096$.
Replicate resampling gives uncertainty for the mean slope.
Any normalized-sensitivity slope describes the fitted predictors and is not asserted to equal the subspace-error slope.

\paragraph{Connection to the finite-pair bound.}
Let $h=d+\log(1/\eta)$. The GRIT theorem's pair-dependent factor is
\[
 T(g,\sigma_\xi,k)=
 \left(\frac{\sigma_\xi}{\sqrt g}+\frac{\sigma_\xi^2}{g}\right)
 \sqrt{\frac hk}
 +\frac{\sigma_\xi^2}{g}\frac hk.
\]
The largest clean eigenvalue does not appear in this factor.
The three sweeps separately control its weakest-signal, noise, and finite-pair inputs.
In the larger-$k$ regime, controlled noise-to-signal ratio yields a $k^{-1/2}$ subspace-estimation scale; the full expression retains its $1/k$ contribution.
Weak-signal and high-noise settings can instead saturate at subspace error one.
The displayed shape reference does not estimate unknown theorem constants or supply a calibrated bound.

The risk theorem also contains source prediction risk, bridge covariate shift, and predictor-dependent factors; its extra-rank term vanishes at the fixed correct rank studied here.
Finite-label ridge-logistic fitting differs from the theorem's unregularized population optimization.
Accordingly, the sweeps examine how pair information affects estimated geometry and fitted instability, alongside target performance.
The complete target-risk bound additionally requires the source-risk, bridge,
and predictor-dependent quantities for each fitted model.
The PC comparison uses the same supplied pair information and the stated validation rule.
 \subsection{Results and interpretation}

Figure~\ref{fig:cg-sweeps} pairs GRIT's subspace recovery with the target
accuracy of GRIT and validation-selected PC; Figure~\ref{fig:cg-sweep-synthetic}
gives the Gaussian counterpart. Figures~\ref{fig:cg-sweep-instability-digits}
and~\ref{fig:cg-sweep-instability-synthetic} compare subspace error with
normalized fitted instability for the same GRIT predictors.
The numerical contrasts below use the primary settings.

\paragraph{Weakest-direction strength.}
At fixed $k=256$, $\sigma_\xi=0.1$, and largest clean eigenvalue one,
increasing $g$ from $1/64$ to one reduces CueMNIST subspace error from 0.535 to
0.064. Normalized fitted instability falls from 0.469 to 0.029.
Target accuracy increases from 3.43\% to 75.86\%, a paired gain of 72.43
percentage points (95\% bootstrap interval: 70.11--74.57).
This isolates weakest-direction strength: the two directions already span
$\mathcal U$ at every positive $g$, and the largest eigenvalue stays fixed.
The separate nonspanning control has subspace error 0.997 and target accuracy
0.02\%. The noiseless full-span control reaches the true-subspace reference
accuracy of 80.31\%, with subspace error below $10^{-14}$.

\paragraph{Noise and pair count.}
With clean spectrum $(1,1)$ and $k=256$, raising noise SD from 0.025 to one
increases CueMNIST subspace error from 0.016 to 0.780 and normalized instability
from 0.002 to 0.673. At the upper noise levels, prediction can fail despite
the clean collection spanning the full unstable subspace.
Conversely, at noise SD 0.1 and the same clean spectrum, increasing $k$ from
64 to 4,096 reduces subspace error from 0.129 to 0.016 and raises target
accuracy from 60.45\% to 79.97\%.
The subspace-error log--log slope over the prespecified $k\ge512$ range is
$-0.500$ (95\% interval: $[-0.519,-0.482]$).
Normalized instability has a steeper empirical slope, $-1.004$
($[-1.209,-0.839]$). The geometric bound is an upper bound for this quantity,
so it need not describe the rate of every fitted predictor.

\paragraph{Synthetic replication.}
The Gaussian task gives the same qualitative dependencies
(Figure~\ref{fig:cg-sweep-synthetic}). Its weakest-eigenvalue sweep reduces
subspace error from 0.568 to 0.064 and improves target accuracy from
55.49\% to 88.43\%. Increasing pair count from 64 to 4,096 reduces
subspace error from 0.128 to 0.016, with a large-$k$ slope of $-0.497$
(95\% interval: $[-0.511,-0.482]$).
The synthetic and CueMNIST geometric curves are similar because both have the
same prescribed signal eigenvalues, dimension and isotropic pair-noise model.
Their fitted sensitivities and target accuracies also depend on their different
labeled-data distributions.

\paragraph{GRIT and consistency regularization.}
At the shared reference setting,
GRIT and PC attain 75.86\% and 71.58\% on CueMNIST, and 88.43\% and 86.77\%
on the synthetic task. On CueMNIST, PC exceeds GRIT at $k=64$ (63.83\% versus
60.45\%) and at noise SD 0.3 (39.14\% versus 35.94\%).
At $k=4096$, GRIT attains 79.97\% versus 71.40\% for PC on CueMNIST, and
89.61\% versus 87.18\% on the synthetic task.
The figures retain all tested noise, signal-strength, and budget values.
The ERM and true-subspace controls use the same labeled data and fitting
procedure, respectively retaining all features or removing the known cue subspace.

With few pairs, PC's soft penalty can yield lower fitted instability than
hard removal using an imperfect subspace: at $k=64$, CueMNIST normalized
instability is 0.056 for PC and 0.091 for GRIT.
As $k$ increases, GRIT estimates the unstable subspace more accurately.
In contrast, the mean noisy pair covariance is
$\boldsymbol\Sigma_k+\sigma_\xi^2\boldsymbol I$, so PC's penalty also acts
on stable directions even as sampling fluctuations decrease.
The large-budget fitted diagnostics are consistent with this tradeoff:
at $k=4096$, normalized instability is 0.0022 for GRIT versus 0.0503 for PC
on CueMNIST, and 0.0021 versus 0.0386 on the synthetic task.
These diagnostics describe the association between geometry, fitted sensitivity,
and accuracy in the tested settings.

The PC strength reaches the largest tested value, $\alpha=1000$, in nine
of ten CueMNIST replicates and four of twenty synthetic replicates at $k=4096$
(and all ten CueMNIST replicates at $k=64$).
The comparison therefore concerns the stated finite grid and validation rule.

\paragraph{Reproducing the comparisons.}
The construction, grids, objective, selector, and estimators in
Appendix~\ref{app:pair-information-sweeps} specify the experiment directly.
For each replicate, construct the fixed labeled features and cue subspace,
then form the clean differences and one master Gaussian-noise array.
For each sweep condition, take the required prefix, add the specified noise,
and compute the leading two singular directions.
Fit GRIT and every PC candidate on the same labeled training examples,
select PC strength by worst-group validation accuracy, and freeze the choices
before evaluating the test groups.
Fit ERM and the true-subspace control once per replicate.
Aggregate test metrics across the ten CueMNIST or twenty synthetic replicates;
use matched replicate differences for contrasts and replicate resampling for
their intervals. The primary fits used CPU float64 and took under a minute
after environment and data setup.

\begin{figure}[htbp]
\centering
\includegraphics[width=\linewidth]{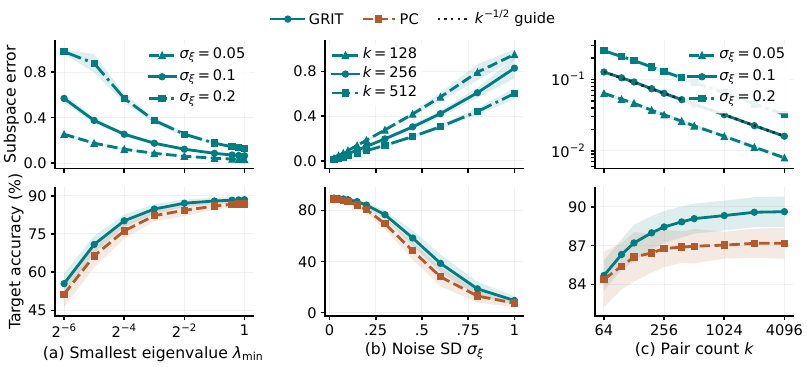}
\caption[The Gaussian task replicates recovery and accuracy dependencies]{\textbf{Stronger signal in the weakest direction and less noise improve recovery and target accuracy; more pairs yield near-$k^{-1/2}$ subspace error and a GRIT advantage over PC at larger budgets.} On the Gaussian synthetic task, the top row shows GRIT's subspace error; the bottom compares target accuracy. Defaults are $k=256$, $\sigma_\xi=0.1$, and clean spectrum $(1,1)$. Top-panel legends vary a second setting; bottom panels use the defaults. Bands show sample SD across 20 replicates; the dotted line is a shape guide. Y-axis ranges differ by panel.
 }
\label{fig:cg-sweep-synthetic}
\end{figure}

\begin{figure}[p]
\centering
\includegraphics[width=\linewidth]{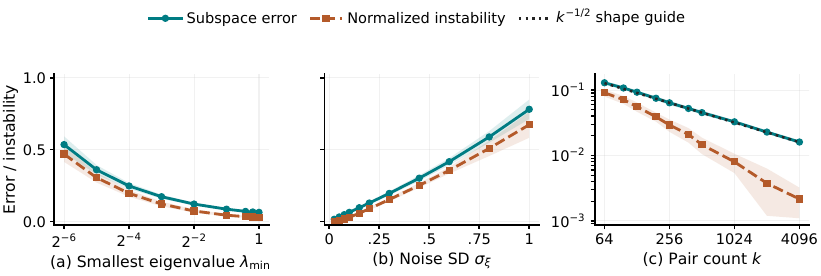}
\caption[CueMNIST fitted instability decreases as recovery improves]{\textbf{GRIT's fitted instability follows the recovery trends and decreases faster than subspace error as pair count grows.} Mirror of the three sweeps on CueMNIST, adding normalized predictor sensitivity to the subspace-error curves. Defaults are $k=256$, $\sigma_\xi=0.1$, and clean spectrum $(1,1)$; bands show sample SD. The dotted $k^{-1/2}$ line describes subspace recovery, not a predictor-rate guarantee.
 }
\label{fig:cg-sweep-instability-digits}
\end{figure}

\begin{figure}[p]
\centering
\includegraphics[width=\linewidth]{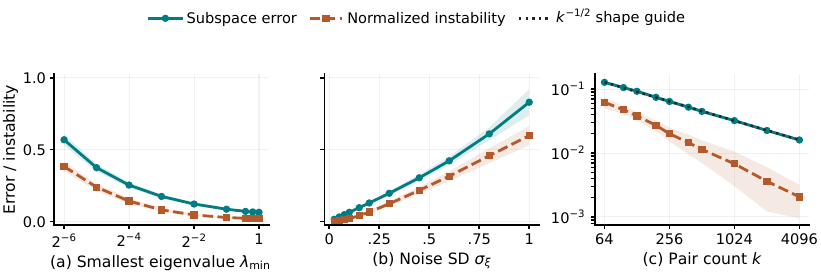}
\caption[Gaussian fitted instability decreases as recovery improves]{\textbf{GRIT's fitted instability follows the recovery trends and decreases faster than subspace error as pair count grows.} Mirror of the three sweeps on the Gaussian synthetic task, adding normalized predictor sensitivity to the subspace-error curves. Defaults are $k=256$, $\sigma_\xi=0.1$, and clean spectrum $(1,1)$; bands show sample SD. The dotted $k^{-1/2}$ line describes subspace recovery, not a predictor-rate guarantee.
 }
\label{fig:cg-sweep-instability-synthetic}
\end{figure}
  \clearpage
\section{Further Discussion and Future Directions}
\label{app:discussion}
\subsection{Extending the predictor and the invariance geometry}
\label{app:discussion-geometry}
Two distinct linearity assumptions shape the present analysis. Assumption~\ref{ass:stability} places the stable information in a linear subspace of the fixed, potentially nonlinear representation; our risk guarantees additionally use a linear probe on that representation. Replacing the probe with a nonlinear head is straightforward empirically: the head can be fitted on the same projected features, and remains exactly invariant to translations in the removed subspace. We expect extensions of the risk analysis to be feasible with suitable control of the head's sensitivity to projection error. Relaxing the probe therefore leaves the underlying geometric model intact.

The deeper challenge is to relax Assumption~\ref{ass:stability} itself, allowing stable information to depend on a nonlinear map of the features. A single subspace and its orthogonal complement may then be insufficient to describe the variation that prediction should ignore. Nonlinear PCA offers a useful analogy: kernel PCA and neural approaches introduce choices about feature geometry and learned mappings beyond a linear eigenspace \citep{scholkopf1998nonlinear,kramer1991nonlinear}. Learning such geometry from invariant pairs would also require identifying which variation can be removed while preserving predictive information. Developing this extension is a promising direction for building on GRIT's foundation of estimating invariance before enforcing it.

\subsection{Pair acquisition, validity, and empirical scope}
\label{app:discussion-pairs}
Progress in either direction also depends on the information supplied by the pairs. Acquiring pairs and checking that they preserve task information remain application-specific tasks. Pooling compatible sources can improve coverage, but their relative counts and scales affect which variations dominate the covariance. Moreover, preserving task content does not by itself establish our geometric or noise assumptions; biased or dependent acquisition may require a different analysis. The current bounds explain dependencies rather than certify individual experimental runs, and the controlled noise study does not establish why GRIT improves performance on pretrained features. Broader studies of pair acquisition and validity can test whether invariant pairs can serve as a reusable interface between application-specific knowledge and robust learning.
 \end{document}